\documentclass{article}

\usepackage{mathpazo}
\usepackage[round,authoryear]{natbib}
\usepackage[a4paper,left=0.83in,right=0.83in,top=0.94in,bottom=0.84in]{geometry}

\usepackage[utf8]{inputenc}
\usepackage[T1]{fontenc}
\usepackage{microtype}
\usepackage{graphicx}
\usepackage{subcaption}
\usepackage{booktabs}
\usepackage{multirow}
\usepackage{array}
\usepackage{tabularx}
\usepackage{arydshln}
\usepackage[table]{xcolor}
\usepackage{nicefrac}
\usepackage{xspace}
\usepackage{textcomp}

\usepackage{amsmath}
\usepackage{amssymb}
\usepackage{mathtools}
\usepackage{amsthm}
\usepackage{bm}

\usepackage[most]{tcolorbox}
\usepackage{xurl}
\usepackage{float}
\usepackage{wrapfig}
\usepackage{needspace}
\newcommand{\keepcompacttable}{\Needspace{22\baselineskip}}
\usepackage{algorithm}
\usepackage{algorithmic}
\usepackage{kdmlab-manuscript}

\renewcommand{\KDMlabmark}{%
  \includegraphics[height=0.75in]{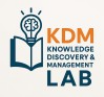}%
}
\renewcommand{\LCSmark}{%
  \includegraphics[width=1.39in]{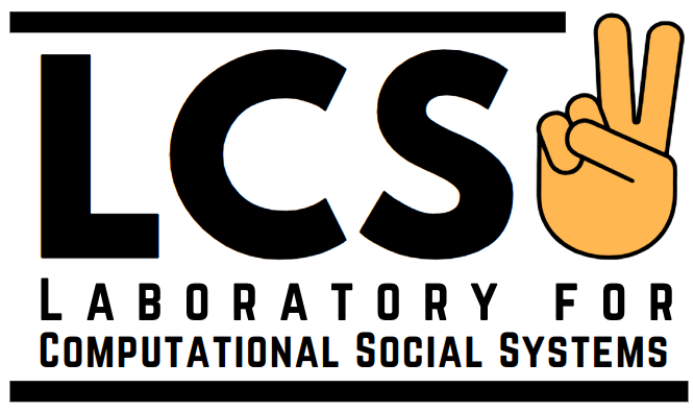}%
}
\renewcommand{\KDMidentitycard}[3]{%
  \begin{minipage}[c][1.18in][c]{\linewidth}
    \centering
    #2\par\vspace{0.10in}%
    {\sffamily\footnotesize\color{KDMGray}#1\par}%
  \end{minipage}%
}

\usepackage{hyperref}
\hypersetup{colorlinks=true,linkcolor=KDMTeal,citecolor=KDMTeal,urlcolor=KDMTeal,
  pdftitle={Action-On-Item Preference Flow: A Shared Event Schema for Predictive and Generative Personalization},
  pdfauthor={Parthiv Chatterjee, Kashish Kanjaria, Vashisth Purani, Sourish Dasgupta, Tanmoy Chakraborty}}
\usepackage[section]{placeins}
\usepackage[capitalize,noabbrev]{cleveref}

\theoremstyle{plain}
\newtheorem{theorem}{Theorem}[section]
\newtheorem{proposition}[theorem]{Proposition}
\newtheorem{lemma}[theorem]{Lemma}
\newtheorem{corollary}[theorem]{Corollary}
\theoremstyle{definition}

\theoremstyle{remark}

\newcommand{\pertide}{\texttt{PerTIDE}\xspace}

\providecommand{\RETURN}{\STATE \textbf{return }}

\title{Action-On-Item Preference Flow: A Shared Event Schema for Predictive and Generative Personalization}
\author{%
  Parthiv Chatterjee\textsuperscript{1} \quad
  Kashish Kanjaria\textsuperscript{1} \quad
  Vashisth Purani\textsuperscript{2}\\[0.35em]
  Sourish Dasgupta\textsuperscript{1} \quad
  Tanmoy Chakraborty\textsuperscript{3}%
}

\begin{document}
\addtocontents{toc}{\protect\setcounter{tocdepth}{-1}}
% \maketitle

\maketitle

\begin{tcolorbox}[
  colback=KDMWash, colframe=KDMRule, boxrule=0.35pt,
  arc=0pt, left=8pt, right=8pt, top=5pt, bottom=5pt
]
\footnotesize\textbf{Version note.} This author-prepared archival manuscript presents
work accepted to NeurIPS 2026. The discussion and interpretation are expanded
in this version; the research content, tables, and reported results are
preserved from the accepted manuscript source.
\end{tcolorbox}
\vspace{0.05in}

\begin{abstract}
A user's movie, news, and dialogue histories differ in their native actions and outputs, yet each interaction supplies evidence that can update user memory. We study whether these histories can train one reusable update mechanism. An action-on-item schema pairs a mapped interaction role with a content embedding, allowing shared update parameters to operate on separate user states. We establish invariance to native relabeling, bounded state changes under item-embedding perturbations, and a pooled-training bound under explicit compatibility conditions. The Multi-Timescale State Hypothesis (MTSH) specifies how this evidence enters, persists, and is consumed; \pertide{} implements it with action gating, three state-space traces, fusion, and command-conditioned readout. On PENS, the same history encoder supports both next-news prediction and personalized headline generation. In a controlled PENS-to-MovieLens experiment, a frozen source-trained core exceeds an identically structured random core by 15.23 MRR points after fitting the same target consumer. On MIND, \pertide{} retains a 4.12-point MRR advantage over a same-input three-branch state-space control. Action, readout, and trace interventions identify complementary contributions to these gains. Together, the theory and experiments support learning history updates across compatible sources and reusing them through predictive and generative consumers.
\end{abstract}

\section{Introduction}
\label{sec:intro}

Personalization systems use a user's interaction history to recommend or generate content. A user may watch or skip movies, click or ignore news, request headlines or summaries, or respond in dialogue. These behaviors differ in their action vocabularies, item types, temporal structures, and downstream objectives. Each interaction nevertheless provides evidence about an evolving user and can change what the system retains. We study whether their common form, an action on an item, allows one history encoder to learn from diverse trajectories and reuse its learned updates across predictive and generative tasks.

Suppose Alice watches a political thriller and a television series with similar themes, while Bob watches the series but skips the film. A common encoder can learn how an action modulates the content written to memory, how that evidence persists, and how a later task uses it. Alice and Bob share these operations, while their histories determine separate states. At a fixed timestep, holding the history, action, and command constant lets us ask a precise question: how much does the update change when the film is replaced by a similar series or by unrelated content?

Prior work provides several forms of sharing. P5 and NEO express multiple recommendation or user-understanding tasks through a common language interface~\citep{p5-recsys-2022,denadai2026neo}. UniSRec and RecGURU learn transferable item or sequence representations~\citep{unisrec-kdd-2022,recguru-wsdm-2022}, and LaMP and Persona-Plug condition generation on retrieved histories or user embeddings~\citep{lamp-2024-lamp,liu-etal-2025-llms}. These capabilities establish reuse at the task or representation level. Successful reuse at those levels alone does not establish how an evolving state responds to controlled changes in an individual action or item. We examine that response, its formal stability, and whether the learned update operations remain useful through different consumers.

Using the action-on-item interface of the user interaction graph (UIG)~\citep{peraugy-tmlr-2025}, we map native interactions to action roles and content embeddings. This abstraction enables one update mechanism to learn from different native histories. It has three precise properties: identical mapped events induce identical state evolution from equal initial states and commands; small item-embedding changes produce bounded state changes under stability and readout conditions; and compatible domains admit a pooled-training bound with explicit distribution and label-semantic discrepancies. An item can therefore change category without requiring a new update rule: its role and content representation determine how it enters the shared mechanism. The stability result describes that response in the supplied embedding geometry, while the pooling result specifies when evidence from different sources can support learning together.

We formulate the \textbf{Multi-Timescale State Hypothesis (MTSH)} to separate what enters memory, how it persists, and how a task uses the retained history. \pertide{} (\textbf{Per}sonalization via multi-\textbf{T}imescale \textbf{I}nteraction \textbf{D}rift \textbf{E}ncoder) first applies a learned action gate to the item embedding. Three state-space traces retain the resulting evidence with long-term, short-term, and episodic biases. Adaptive fusion combines their contributions, with a residual path carrying the current event directly. A command-conditioned readout then forms a task-ready user representation for a predictive head or a prefix-conditioned text decoder. Sharing these operations lets the encoder learn from pooled event trajectories while each user maintains a history-specific state. The traces provide different retention choices for the same evidence; their combinations are tested through temporal probes and ablations.

The results identify both reusable learning and the contributions of its components. On PENS, the same history encoder supports next-news prediction and personalized headline generation. In a controlled PENS-to-MovieLens experiment, a frozen source-trained core exceeds an identically structured random core by \textbf{15.23 MRR points} after fitting the same target consumer. Source training therefore supplies useful structure that target-head fitting alone does not recover from the random core; the low score with an untrained consumer locates the consumer's role in accessing that structure. On MIND, \pertide{} retains a \textbf{4.12-point MRR advantage} over a same-input three-branch SSM control, supporting the additional retention design. The trace ablations reveal why isolated strength is insufficient for judging a component: E alone reaches only 24.18 MRR, while adding E to L+S raises MRR from 41.26 to 46.65. The same addition raises PENS PSE-METEOR from 0.182 to 0.355. A weak isolated trace can thus provide useful complementary evidence to the fused state.

\textbf{Contributions.} We make three contributions. \emph{(i) A formal account of history-update reuse} establishes native-label invariance, item-embedding stability, and compatible pooled learning on the existing action-on-item interface. \emph{(ii) MTSH and its \pertide{} implementation} organize action-dependent entry, complementary retention, fusion, and command-conditioned consumption into a shared encoder. \emph{(iii) Controlled evidence of learned reuse} connects predictive and generative consumers, source-frozen domain transfer, same-input comparisons, and mechanism interventions. Together, these contributions explain what is shared, how the shared operations are realized, and where the reported evidence supports their usefulness.

\section{Related Work}
\label{sec:related_work}

\textbf{Sharing tasks and representations.}
P5 unifies recommendation tasks through language modeling, and NEO combines recommendation, search, and user understanding through language-steered generation~\citep{p5-recsys-2022,denadai2026neo}. Multi-behavior and multi-task methods share parameters or objectives across related signals~\citep{nmtr-icde-2019,mmoe-kdd-2018}. RecGURU aligns user representations adversarially, UniSRec transfers item and sequence representations, and GenCDR combines universal and domain-specific representations~\citep{recguru-wsdm-2022,unisrec-kdd-2022,hu2026gencdr}. We build on this established reuse. Our question concerns the action-dependent update applied to each event, its stability under item substitution, and the retention and readout operations learned on the common interface. UniSRec and RecGURU provide direct empirical comparisons for transfer.

\textbf{Temporal history encoders.}
News and sequential recommendation use profile pooling, recurrence, attention, and state-space models~\citep{naml,ebnr,nrms,gru4rec-2016,bert4rec-cikm-2019,mamba4rec-arxiv-2024}. Long/short-term models, time-aware attention, and multi-interest encoders already represent persistent interests and recent intent~\citep{lstur-acl-2019,yu-ijcai-2019,tisasrec-wsdm-2020,comirec-kdd-2020}. PRIME and Reflective Memory Management organize episodic, semantic, or dialogue memory, while generalized user representations aggregate behavior across temporal windows~\citep{zhang-etal-2025-prime,tan-etal-2025-prospect,fazelnia2025generalized}. Building on temporal modeling, MTSH separates event entry, persistence, fusion, and task consumption on one action-on-item interface; a same-input three-branch control tests the additional value of its retention parameterization.

\textbf{Personalized generation and event construction.}
Generation can use retrieved profile evidence, compact user embeddings, or user-specific parameters, as in LaMP, Persona-Plug, and OPPU~\citep{lamp-2024-lamp,liu-etal-2025-llms,tan-etal-2024-democratizing}. PENS and GTP study personalized news generation~\citep{pens-acl-2021,gtp-tacl-2023}. UIG and PerAugy supply the event representation and trajectory augmentation used here~\citep{peraugy-tmlr-2025}; they are precursors to the present encoder. \pertide{} learns shared update parameters and supplies a command-adapted history state to predictive heads or a prefix-conditioned decoder. Reusing that learned core is evaluated separately from training a consumer for a new output space.

\section{Action-On-Item Preference Flow}
\label{sec:action_item_preference_trajectories}

\subsection{A Shared Event Interface}
\label{subsec:action_item_interface}
For dataset $D$, a deterministic mapping $\rho_D$ assigns native interactions to positive, negative, request, or generated-response transitions. The last two distinguish the sides of a command-like interaction. Each event is $b_u^{(t)}=\langle a^{(t)},tl^{(t)}\rangle$, where $a^{(t)}$ is the mapped role and $tl^{(t)}$ is the item or response acted upon. Its embedding $\mathbf e_{tl}^{(t)}\in\mathbb R^d$ supplies the continuous content representation. The concrete task is specified separately by command seed $\mathbf g_q$. Roles are supplied by preprocessing; the encoder learns how each role modulates content.

UIG represents these events as action-labeled edges among user, content, and response nodes~\citep{peraugy-tmlr-2025}. A user's ordered path gives a trajectory $\tau_u$, consumed as a sequence of compact interaction blocks. Explicit clicks and non-clicks map directly to roles; ratings and dialogue stance require derived labels. Appendix~\ref{app:UIG Construction} records those construction rules and their different provenance.

\textbf{What remains shared when the item changes?}
Fix Alice's history, watch action, command, and current timestep. Replace the current movie first by a similar movie, then by a series, and finally by unrelated content. The update consumes the role and embedding, rather than the native item type. For equal initial states and commands, identical mapped trajectories give identical states (Appendix~\ref{app:schema_generalizability}, Proposition~\ref{thm: generalization - Invariance to Native Action Relabeling}). For the trace update $\mathbf m_t^r=A_r\mathbf m_{t-1}^r+B[\mu(a_t)\odot\mathbf e_t]$, replacing only that item changes its contribution at a later lag $k$ by at most
\begin{equation}\label{eq:item-substitution-main}
\|\Delta\mathbf m_{t+k}^r\|_2\leq\|B\|_2\beta_r^k\|\mathbf e_t-\mathbf e'_t\|_2,\qquad \|A_r\|_2\leq\beta_r<1,
\end{equation}
when subsequent inputs agree. The complete-state bound additionally accounts for fusion, the current-input residual, and recurrent readout (Appendix~\ref{app:schema_generalizability}, Proposition~\ref{thm: generalization - Stability Under Target-Locus Embedding Shifts}). Thus a similar item can inherit a similar update without a hand-built taxonomy of movies, series, and other content. For fixed history, action, and command, halving the embedding perturbation halves this upper bound. A movie-to-series replacement uses the same operations as a movie-to-movie replacement; embedding distance governs the guarantee. More distant items permit a larger change without forcing one.

\textbf{Why pool different trajectories?}
Alice's movie and series histories, and Bob's histories, train the same role-conditioned update parameters while maintaining separate user states. The pooled-training result in Appendix~\ref{app:schema_generalizability} decomposes target risk into pooled empirical risk, an estimation term, schema-distribution distance, and a discrepancy in labels or command semantics. When distributions and labels are compatible, examples from every participating domain contribute to the estimation bound for one shared update rule. The explicit mismatch terms identify what must remain comparable for that gain to carry to a target domain. Under matching schema distributions and conditional label semantics, a common model-class complexity bound $C/\sqrt n$ yields an estimation allowance that scales with $1/\sqrt N$ for $N$ independent pooled examples (Corollary~\ref{cor:schema-compatible-pooling}). Compatible histories can therefore contribute to estimating the shared rule even when they come from another native domain.

\paragraph{From History States To Preference Flow.}
\label{subsec:history_states_to_flow}
Profile aggregation (PAH) discards order, Markovian drift (MDH) carries the prefix through a recurrent state, and long-horizon dynamics (LDH) models learned dependencies over earlier events. These are overlapping design families, not mutually exclusive laws of user behavior. The action-on-item interface poses a further design question within any family: \textit{what should be written, how should it persist, and how should a task consume it?}

\subsection{Multi-Timescale State Hypothesis}
\label{subsec:mtsh}
MTSH retains action-conditioned evidence through complementary temporal traces before task consumption. Long-term ($L$), short-term ($S$), and episodic ($E$) describe intended slow, recent, and event-related retention. Their usefulness and differentiation require empirical tests. With entry $\Psi$, persistence $\mathcal P_r$, fusion $\Phi$, routing $\Pi$, and readout $G$, the implemented decomposition is
\begin{equation}\label{eq:mtsh}
\begin{aligned}
\mathbf b_u^{(t)}&=\Psi(\mathbf e_{tl}^{(t)},a^{(t)}),\qquad
\mathbf m_u^{(t,r)}=\mathcal P_r(\mathbf m_u^{(t-1,r)},\mathbf b_u^{(t)}),\\
\mathbf h_u^{(t)}&=\Phi(\mathbf m_u^{(t,L)},\mathbf m_u^{(t,S)},\mathbf m_u^{(t,E)},\mathbf b_u^{(t)}),\\
\boldsymbol\pi_u^{(t)}&=\Pi(\mathbf g_q,\mathbf z_u^{(t-1)}),\qquad
\mathbf z_u^{(t)}=G(\boldsymbol\pi_u^{(t)},\mathbf h_u^{(t)}).
\end{aligned}
\end{equation}
The current event can reach fusion through a residual path, and routing uses the preceding task-ready state. Section~\ref{sec:PerTIDE} specifies each operator. Sharing these operators leaves both the retained history and its command-adapted output specific to the user.

\section{\pertide{}: Instantiating MTSH}
\label{sec:PerTIDE}

\begin{figure}[t]
    \centering
    \includegraphics[width=0.80\linewidth]{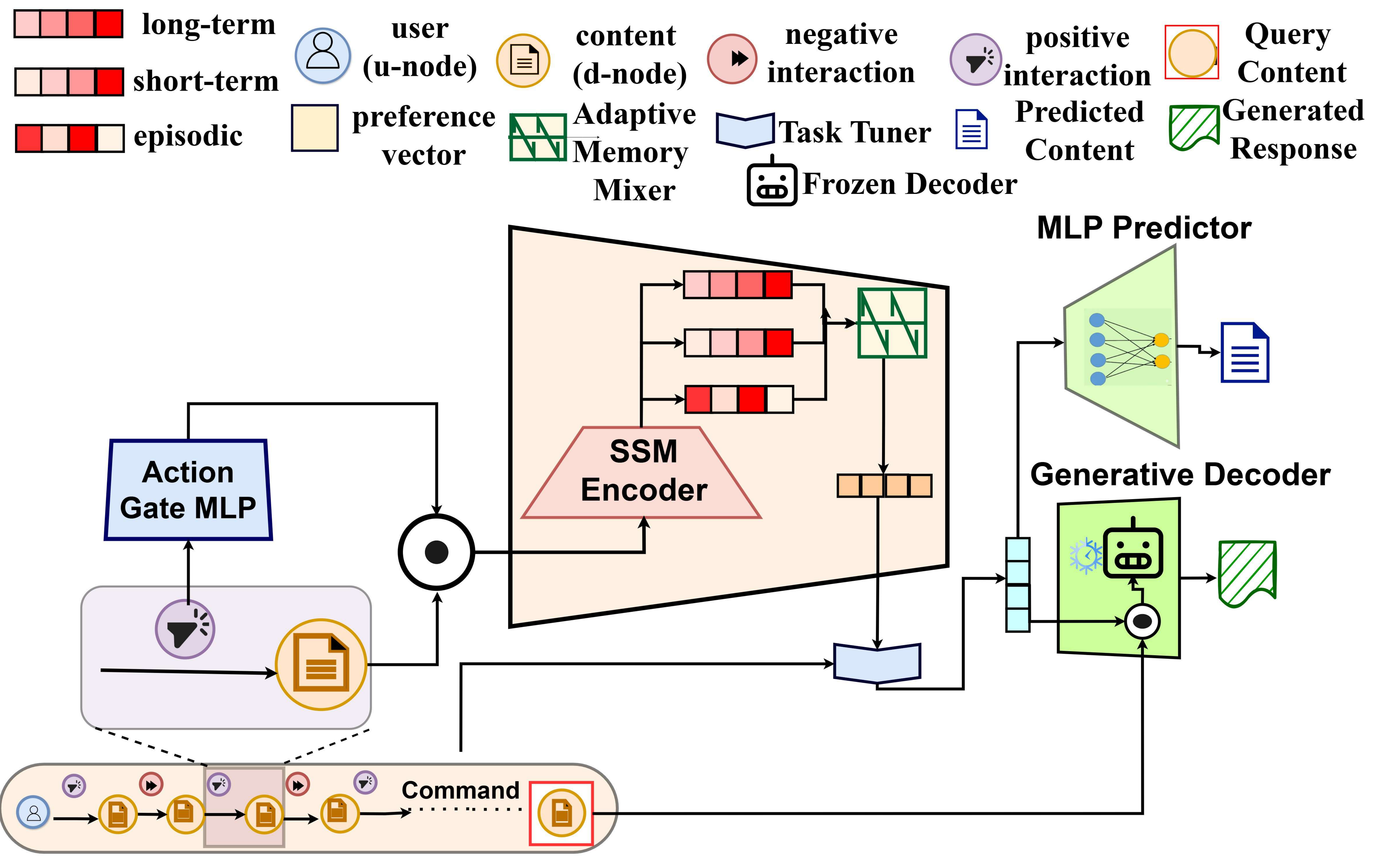}
    \caption{\textbf{\pertide\ pipeline}: a user's UIG trajectory is mapped to action-gated b-cells and encoded by a tri-trace SSM with an adaptive memory mixer; a command-conditioned readout produces a task-ready state that conditions task-specific heads for prediction (MLP) and generation (DistilGPT2 decoder). The equations in Section~\ref{sec:PerTIDE} additionally specify the shared eigenbasis, current-event residual, and routing through the preceding task-ready state.}
    \label{fig:pipeline-pertide-full}
\end{figure}

We instantiate MTSH through the four operations in Eq.~\ref{eq:mtsh}: action-conditioned evidence entry, trace-specific persistence, adaptive trace fusion, and command-conditioned consumption. We illustrate the architecture in Figure~\ref{fig:pipeline-pertide-full}. Each schema-normalized event is converted into action-conditioned evidence, propagated through long-, short-, and episodic traces, and fused into a user-specific preference-history state. The active command then converts this user-specific history state into a task-ready representation for a predictive or generative consumer.

\textbf{Action entry.}
We initialize item and response nodes with MPNet embeddings~\citep{mpnet-bert-neurips-2020}. At timestep $t$, the supplied action role $a^{(t)}$ controls how target-locus evidence enters the state:
\begin{equation}
\boldsymbol\mu^{(t)} = \sigma\!\left(\mathrm{MLP}_{\mathrm{gate}}(a^{(t)})\right),
\qquad
\mathbf b_u^{(t)} = \boldsymbol\mu^{(t)}\odot\mathbf e_{tl}^{(t)}.
\end{equation}
Here, $a^{(t)}$ comes from the deterministic event mapping in Section~\ref{subsec:action_item_interface}; the learned gate determines how the supplied role modulates content evidence. A negative-role gate still has nonnegative entries; it does not automatically subtract the item embedding or impose a monotone dislike score. Thus, a clicked and an ignored article may share the same content
space while contributing differently through their positive and negative roles. The resulting $\mathbf b_u^{(t)}$ instantiates the evidence-entry operator $\Psi$ in Eq.~\ref{eq:mtsh}.

\textbf{Trace persistence.}
The action-conditioned event enters three lag-biased state-space model (SSM) traces. The implementation shares an orthogonal eigenbasis $Q$ across traces and uses diagonal real spectra $\Lambda_r$. The transition matrices are defined by the following maps, with $Q_L=Q_S=Q_E=Q$ and $Q^{-1}=Q^\top$:
\begin{equation}
\label{eq:tristate}
\begin{aligned}
\mathbf m_u^{(t,L)} &= Q_L\exp(\delta\Lambda_L)Q_L^{-1}\mathbf m_u^{(t-1,L)}
+B\mathbf b_u^{(t)},\\
\mathbf m_u^{(t,S)} &= Q_S\sigma(-\kappa\Lambda_S)Q_S^{-1}\mathbf m_u^{(t-1,S)}
+B\mathbf b_u^{(t)},\\
\mathbf m_u^{(t,E)} &= Q_E(I+\eta\Lambda_E)Q_E^{-1}\mathbf m_u^{(t-1,E)}
+B\mathbf b_u^{(t)}.
\end{aligned}
\end{equation}
For the generic recurrence $\mathbf m_u^{(t,r)}=A_r\mathbf m_u^{(t-1,r)}+B\mathbf b_u^{(t)}$, unrolling yields the lag response $K_r(\Delta)=A_r^\Delta$. A modal base $a$ assigns weight $a^\Delta$ at lag $\Delta$. Slow and fast bases yield different retention, while event inputs leave decaying tails. The three formulas permit overlapping kernels; their names describe intended biases rather than guaranteeing an ordering or semantic partition. These kernels instantiate the persistence operators $\mathcal P_r$ in Eq.~\ref{eq:mtsh}.

\textbf{Trace fusion.}
The updated long-, short-, and episodic traces are adaptively fused into the preference-history state through
\begin{equation}
\mathbf h_u^{(t)}=\mathrm{MLP}_{\Phi}([\mathbf m_u^{(t,L)};\mathbf m_u^{(t,S)};\mathbf m_u^{(t,E)}])+W_\psi\mathbf b_u^{(t)}+\mathbf b_\psi.
\end{equation}
The residual supplies the current event directly without adding another recurrent trace.
This realizes the fusion operator $\Phi$ in Eq.~\ref{eq:mtsh}. See Appendix~\ref{app:stability-ssm}-\ref{app:kernel-timescale-proof} for stability conditions, modal responses, and the kernel derivations.

\textbf{Command consumption.}
The fused state $\mathbf h_u^{(t)}$ represents the user's preference history before resolving the concrete downstream operation. Given task seed $\mathbf g_q$, a supervised command classifier uses the preceding task-ready state as routing context and estimates
\begin{equation}
\boldsymbol\pi_u^{(t)}
=
\operatorname{SoftMax}\!\left(
W_{\mathrm{task}}
[\mathbf z_u^{(t-1)};\mathbf g_q]
\right).
\end{equation}
The corresponding task-ready state is
\begin{equation}
\mathbf z_u^{(t)}
=
(W_\pi\boldsymbol\pi_u^{(t)})
\odot
\mathbf h_u^{(t)}
+
\mathbf b_G.
\end{equation}
Thus, $\mathbf h_u^{(t)}$ is the user-specific history state learned under the shared event schema, while $\mathbf z_u^{(t)}$ resolves how that state is consumed for the active task. The same news history can therefore support next-news prediction or headline generation, while the command seed and downstream consumer determine the required task-ready
representation. This instantiates $G\circ\Pi$ in Eq.~\ref{eq:mtsh}.

\textbf{Recurrent user memory.} The encoder retains three 768-dimensional traces and the preceding task-ready state for each user. These four vectors require 12 KiB in FP32, independent of history length. Each new event updates this fixed recurrent cache, while the predictive or generative consumer accesses the resulting state. Appendix~\ref{app:detailed_results_efficiency} gives the calculation and separates it from model weights and runtime estimates.

\textbf{Training and heads.}
We train the shared encoder with temporal-position, next-event, and command-classification objectives:
$
\mathcal L_{\mathrm{enc}}
=
0.2\mathcal L_{\mathrm{pos}}
+
0.4\mathcal L_{\mathrm{next}}
+
0.4\mathcal L_{\mathrm{task}}.
$
One encoder is trained jointly on the MovieLens, MIND, PENS, and OpenAI training trajectories after schema construction. The source-only transfer experiments train a separate PENS-only encoder. The three loss weights are globally fixed choices, not learned optima. Predictive consumers optimize task-specific ranking or classification objectives. 

\textbf{Generative consumers.} For generation, we use a pretrained DistilGPT2 decoder conditioned through a learned prefix. User-fused injection (UFI) combines the predicted query-conditioned representation with the query-document embedding, while user-gated injection (UGI) additionally uses the task-ready user state to apply feature-wise affine modulation to the query-conditioned representation before
fusion. The resulting control is combined with query-document evidence and projected to the decoder prefix. During generation training, the contextualization modules, final two decoder blocks, and final layer normalization are optimized, while the remaining decoder parameters stay frozen. The generation-specific stage also keeps the history encoder fixed. Personalization therefore enters through the encoded preference-history state rather than a user-specific decoder. Architecture and decoder variants are given in Appendix~\ref{app:arch-details}; objectives, dimensions, and training are in Appendix~\ref{app:training-details}.
\section{Evaluation}
\label{sec:evaluation}

\textbf{Datasets and tasks.} Table~\ref{tab:evaluation_coverage} in Appendix~\ref{app:datasets} maps the five main settings to their event evidence, consumers, and metrics.
We evaluate naturally ordered predictive histories on MovieLens-1M, MIND~\citep{dataset-mind-acl-2020},
and PENS~\citep{pens-acl-2021}; PENS also provides personalized headline generation from the same news-side
history abstraction. OpenAI-Reddit is used only as an \emph{eventization stress test} since it lacks observed
chronology. For OpenAI, native train/evaluation pools are separated before construction, and PerAugy is applied independently within both pools. PENS surrogate responses and augmentation are used in training; MIND uses its native interaction histories without that surrogate construction. MPChat is tested as a boundary case where event roles must be inferred from dialogue/discourse.

\textbf{Task-specific baselines.}
Task-native baselines cover PAH, MDH, and LDH families, including NAML, PLM-NR, EBNR (GRU), NRMS, MINER, SASRec, BERT4Rec, Mamba4Rec, and other dataset-standard systems. Prompted instruction-tuned LLMs provide a fixed-parameter functional alternative for generation. To separate input construction from state design, we additionally evaluate schema-matched MeanPool-B, MaxPool-B, GRU-B, LSTM-B, D-FM-Attention-B, Discretized-SSM-B, and Tri-SSM-B controls on the same ordered interaction-block trajectories as \pertide{}. Tri-SSM-B further matches the three-branch organization while omitting the explicit L/S/E retention maps. These MIND controls hold the input representation fixed; task-native comparisons may use their original input constructions.

\textbf{Transfer stress-test baselines.} For the source-only reuse stress test, we train on PENS and withhold MovieLens from source-encoder training. We compare \pertide{} with two methods explicitly designed for transfer: UniSRec
~\citep{unisrec-kdd-2022} and RecGURU~\citep{recguru-wsdm-2022}. Each method uses the same PENS/MovieLens content-embedding interface and receives either 500 or 3,000 MovieLens trajectories for target-side alignment. The evaluation uses the same held-out MovieLens protocol.

\textbf{Metrics and protocol.}
Predictive settings use one positive target and 699 sampled negatives and report MRR, nDCG, and HR@10.
Generation uses PerSEval~\citep{perseval-tmlr-2024} (PSE-JSD, PSE-SU4, PSE-METEOR), together with reference-based content metrics and human ratings of headline similarity (Appendices~\ref{app:content accuracy metrics} and~\ref{app:hj}). These evaluate agreement with references, rather than directly measuring article factuality. Main comparisons use paired bootstrap evaluation with
10,000 resamples; robustness studies keep trajectories and candidate sets fixed so only the perturbed factor changes.
Full dataset, preprocessing, hyperparameter, statistical, and evaluation details are in
Appendices~\ref{app:datasets}--\ref{app:training-details}.
\section{Results And Analysis}
\label{sec:results}

\subsection{RQ1: Can The Shared Encoder Serve Prediction And Generation?}
\begin{table}[t]
\centering
\caption{\textbf{One encoder supports prediction and generation.} Predictive scores are MRR/nDCG (0--100), with nDCG@5 for MIND and nDCG@10 for MovieLens and PENS; generative scores are PSE-JSD/PSE-METEOR (0--1). Each row uses its strongest listed task-native comparator. Full scores appear in Appendix~\ref{app:detailed_results_predictive}, Tables~\ref{tab:ml1m_recommendation}--\ref{tab:rq1_pens_structured_recommendation}; Appendix~\ref{app:detailed_results_generation}, Table~\ref{tab:pse_icml}; and Appendix~\ref{app:detailed_results_openai_reddit}, Table~\ref{tab:rq2_openai_only}. Table~\ref{tab:iclr_schema_matched} gives the same-input controls.}
\label{tab:iclr_main_scoreboard}
\setlength{\tabcolsep}{3pt}
\setlength{\extrarowheight}{4pt}
\renewcommand{\arraystretch}{1.1}
\begin{tabularx}{\linewidth}{@{}p{0.19\linewidth}Xccc@{}}
\toprule
\textbf{Setting} & \textbf{Comparator} & \textbf{Baseline} & \textbf{\pertide{}} & \textbf{Gain} \\
\midrule
MovieLens-1M movie ranking & SSD4Rec & 14.95 / 20.31 & \textbf{17.31 / 23.78} & +2.36 / +3.47 \\
\hdashline
MIND news ranking & MINER & 36.60 / 40.20 & \textbf{46.65 / 49.21} & +10.05 / +9.01 \\
\hdashline
PENS next-news ranking & LSTUR & 9.04 / 10.16 & \textbf{9.94 / 11.58} & +0.90 / +1.42 \\
\midrule
PENS headline generation & DeepSeek-R1-32B (2-shot) & 0.263 / 0.174 & \textbf{0.366 / 0.355} & +0.103 / +0.181 \\
\hdashline
OpenAI-Reddit TL;DR generation & DeepSeek-14B (2-shot) & 0.243 / 0.109 & \textbf{0.285 / 0.294} & +0.042 / +0.185 \\
\bottomrule
\end{tabularx}
\end{table}
Across the five main settings in Table~\ref{tab:iclr_main_scoreboard}, \pertide{} improves the reported primary metrics over the listed task-native comparators. MovieLens gains are 2.36/3.47 MRR/nDCG points over SSD4Rec; MIND gains are 10.05/9.01 over MINER; and PENS prediction gains are 0.90/1.42 over LSTUR. On PENS, the same history encoder also supports headline generation through a different consumer. UGI reaches 0.366/0.341/0.355 PSE-JSD/PSE-SU4/PSE-METEOR, compared with 0.263/0.125/0.174 for the strongest listed two-shot LLM, DeepSeek-R1-Distill-Qwen-32B. These complete-system gains use trained contextualization and partial decoder tuning for \pertide{} and two-shot prompting for the LLM. OpenAI's gains concern constructed trajectories; dialogue boundary results and metric-specific exceptions are reported in Appendix~\ref{app:detailed_results}.

\paragraph{Generation With Sparse Positive Evidence.} The encoder retains an advantage when fewer positive events are available. On the click-only sparse subset, \pertide{} achieves 0.211/0.244/0.262 PSE-JSD/PSE-SU4/PSE-METEOR, compared with 0.152/0.087/0.092 for the prompted DeepSeek-R1-32B baseline (Appendix~\ref{app:detailed_results_sparse}). The SU4 and METEOR drops are 28.45\% and 26.2\%, versus 30.4\% and 47.2\% for that baseline. JSD falls by about 42\% for both systems. Thus the advantage persists with limited positive evidence, with stronger relative retention on two personalization metrics. The result concerns this sparse subset and does not isolate which trace supplies the retained benefit.

\subsection{RQ2: Do Entry, Retention, And Readout Explain The Gains?}
\begin{table}[t]
\centering
\caption{\textbf{Same-input state-model control on MIND.} All ``-B'' controls consume the same ordered schema-normalized interaction-block embeddings as \pertide{}. Tri-SSM-B provides a structurally matched three-branch SSM control without the explicit long-, short-, and episodic retention maps used by \pertide{}. Full task-native results appear in Appendix~\ref{app:detailed_results_mind}, Table~\ref{tab:history_encoder_results}.}
\label{tab:iclr_schema_matched}
\setlength{\tabcolsep}{5pt}
\renewcommand{\arraystretch}{1.02}
\scalebox{1}{
\begin{tabular}{llccc}
\toprule
\textbf{State family} & \textbf{Model} & \textbf{MRR} & \textbf{nDCG@5} & \textbf{HR@10} \\
\midrule
PAH & MeanPool-B & 29.43 & 31.22 & 33.17 \\
PAH & MaxPool-B & 28.32 & 31.16 & 32.41 \\
MDH & GRU-B & 32.64 & 35.17 & 37.92 \\
MDH & LSTM-B & 32.17 & 34.92 & 37.05 \\
LDH & D-FM Attention-B & 32.17 & 34.56 & 36.21 \\
LDH & Discretized SSM-B & 35.14 & 38.53 & 37.27 \\
LDH & Tri-SSM-B & 42.53 & 43.84 & 46.33 \\
\midrule
MTSH & \textbf{\pertide{}} & \textbf{46.65} & \textbf{49.21} & \textbf{55.48} \\
\bottomrule
\end{tabular}
}
\end{table}
\begin{figure}[t]
\centering
\begin{minipage}[t]{0.32\linewidth}\centering
\begin{tikzpicture}[x=0.032cm,y=0.17cm,font=\scriptsize]
\draw[rounded corners=2pt,draw=black!30,fill=black!1] (-20,-4.6) rectangle (111,17);
\node[font=\scriptsize\bfseries] at (45.5,15.7) {(a) PENS Role Corruption};
\foreach \y in {0,4,8,12} {\draw[black!12] (0,\y)--(100,\y); \node[anchor=east] at (-3,\y) {\y};}
\draw[black!60] (0,14)--(0,0)--(100,0);
\foreach \x in {0,25,50,75,100} {\draw (\x,0)--(\x,-0.3); \node[anchor=north] at (\x,-0.55) {\x};}
\draw[blue!65!black,thick] plot coordinates {(0,12.14) (10,10.13) (25,9.17) (50,7.71) (100,2.81)};
\foreach \x/\y in {0/12.14,10/10.13,25/9.17,50/7.71,100/2.81} {\fill[blue!65!black] (\x,\y) circle[radius=1.4pt];}
\node[anchor=south west] at (0,12.55) {12.14};
\node[anchor=south east] at (100,3.2) {2.81};
\node[rotate=90] at (-15,7) {MRR};
\node at (50,-3.35) {Corrupted roles (\%)};
\end{tikzpicture}
\end{minipage}\hfill
\begin{minipage}[t]{0.32\linewidth}\centering
\begin{tikzpicture}[x=0.95cm,y=0.048cm,font=\scriptsize]
\draw[rounded corners=2pt,draw=black!30,fill=black!1] (-0.8,-16.3) rectangle (3.6,60.2);
\node[font=\scriptsize\bfseries] at (1.4,55.5) {(b) MIND Trace Subsets};
\foreach \y in {0,10,20,30,40,50} {\draw[black!12] (0.35,\y)--(3.4,\y); \node[anchor=east] at (0.25,\y) {\y};}
\draw[black!60] (0.35,50)--(0.35,0)--(3.4,0);
\fill[black!35] (0.7,0) rectangle (1.3,24.18);
\fill[blue!55!black] (1.7,0) rectangle (2.3,41.26);
\fill[teal!75!black] (2.7,0) rectangle (3.3,46.65);
\foreach \x/\y in {1/24.18,2/41.26,3/46.65} {\node[anchor=south] at (\x,\y+1) {\y};}
\node[anchor=north] at (1,-2) {E}; \node[anchor=north] at (2,-2) {L+S}; \node[anchor=north] at (3,-2) {L+S+E};
\node[rotate=90] at (-0.58,25) {MRR};
\node at (1.85,-12) {Retained traces};
\end{tikzpicture}
\end{minipage}\hfill
\begin{minipage}[t]{0.32\linewidth}\centering
\begin{tikzpicture}[x=0.95cm,y=6cm,font=\scriptsize]
\draw[rounded corners=2pt,draw=black!30,fill=black!1] (-0.8,-0.1304) rectangle (3.6,0.4816);
\node[font=\scriptsize\bfseries] at (1.4,0.444) {(c) PENS Trace Subsets};
\foreach \y in {0,0.1,0.2,0.3,0.4} {\draw[black!12] (0.35,\y)--(3.4,\y); \node[anchor=east] at (0.25,\y) {\y};}
\draw[black!60] (0.35,0.4)--(0.35,0)--(3.4,0);
\fill[black!35] (0.7,0) rectangle (1.3,0.045);
\fill[blue!55!black] (1.7,0) rectangle (2.3,0.182);
\fill[teal!75!black] (2.7,0) rectangle (3.3,0.355);
\foreach \x/\y in {1/0.045,2/0.182,3/0.355} {\node[anchor=south] at (\x,\y+0.008) {\y};}
\node[anchor=north] at (1,-0.016) {E}; \node[anchor=north] at (2,-0.016) {L+S}; \node[anchor=north] at (3,-0.016) {L+S+E};
\node[rotate=90] at (-0.62,0.20) {PSE-METEOR};
\node at (1.85,-0.096) {Retained traces};
\end{tikzpicture}
\end{minipage}
\caption{\textbf{Correct action roles and complementary traces improve personalization.} (a) On fixed PENS trajectories, targets, and candidates, only positive/negative roles change. (b,c) The episodic trace is weak alone, yet adding it to L+S improves MIND MRR by 5.39 points and PENS PSE-METEOR by 0.173. The panels use separate protocols and metric scales. Full interventions are in Tables~\ref{tab:iclr_role_corruption} and~\ref{tab:trace_subset_ablation}.}
\label{fig:iclr_wow_panels}
\end{figure}
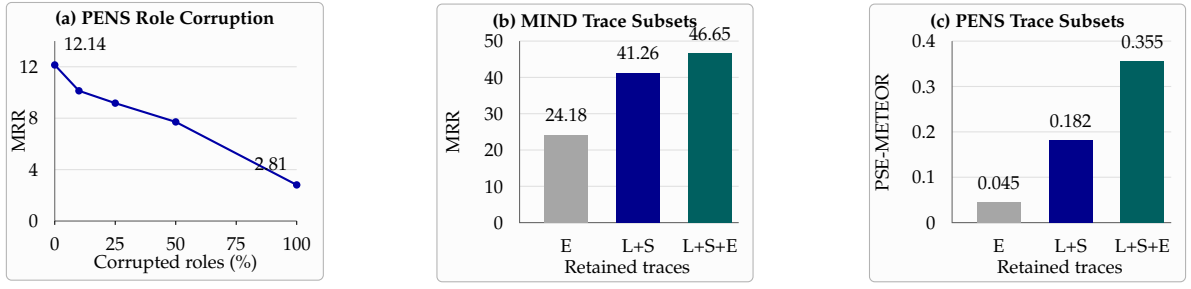
\textbf{Holding the input and branch count fixed.}
In the MIND comparison, a single Discretized-SSM-B reaches 35.14 MRR and three ordinary SSM branches reach 42.53. \pertide{} reaches 46.65, adding 4.12 MRR, 5.37 nDCG@5, and 9.15 HR@10 points over Tri-SSM-B (Table~\ref{tab:iclr_schema_matched}). The advantage survives both common input construction and a three-branch control, supporting the additional retention parameterization in this setting.

\textbf{Action-dependent entry and command-dependent readout.}
Removing action conditioning reduces MIND MRR/HR@10 from 46.65/55.48 to 40.13/46.35; removing command-conditioned readout gives 42.51/51.06. These interventions locate useful contributions at two different stages of the update. A separate fixed-trajectory role-corruption experiment progressively flips positive and negative labels: MRR falls from 12.14 to 10.13, 9.17, 7.71, and 2.81 at 0\%, 10\%, 25\%, 50\%, and 100\% corruption (Appendix~\ref{app:detailed_results}, Table~\ref{tab:iclr_role_corruption}). Correct action roles therefore carry predictive information that the encoder uses (Figure~\ref{fig:iclr_wow_panels}a).

\textbf{Retention behavior and complementary traces.}
Target-fixed probes alter where preference-similar evidence occurs (Appendix~\ref{app:detailed_results_trace_subset}, Table~\ref{tab:controlled_temporal_split_pertide_variants}). L-Tr retains the top 75\% of history nodes ranked by embedding proximity to the target in their original order; S-Tr concentrates such evidence in the final 30\%; E-Tr inserts a consistent 3--4-step burst among unrelated content. On MIND, each isolated trace is strongest in its corresponding regime: L obtains 49.54/54.78 on L-Tr, S obtains 50.07/55.35 on S-Tr, and E obtains 50.35/55.84 on E-Tr in MRR/nDCG@10. MovieLens has an exception: L slightly exceeds S on S-Tr, 27.18/29.35 versus 27.13/29.22. The fused state leads every reported probe, including the setting where isolated specialization overlaps.

Natural-task ablations further separate isolated strength from contribution to a combination (Appendix~\ref{app:detailed_results}, Table~\ref{tab:trace_subset_ablation}). On MIND, E alone reaches only 24.18 MRR, yet adding E to L+S raises MRR from 41.26 to 46.65. On PENS generation, the same addition raises PSE-METEOR from 0.182 to 0.355. The full model benefits from the episodic contribution even though E is weak alone (Figure~\ref{fig:iclr_wow_panels}b,c); the appendix also reports smaller combinations that hurt performance.

\subsection{RQ3: Do Learned Updates Survive A New Consumer Or Domain?}
\textbf{One history, different consumers.}
PENS prediction reaches 9.94/11.58 MRR/nDCG@10, while the same encoded news histories condition personalized headlines. UGI improves on UFI by 0.034/0.052/0.052 in the three PSE metrics (Appendix~\ref{app:decoder-variants}; Table~\ref{tab:pse_icml}). The state is reusable across the evaluated objectives, and its decoder conditioning still affects how effectively generation uses it.

\textbf{Source-trained reuse with target alignment.}
A PENS-only encoder, trained with MovieLens excluded, is aligned using either 500 or 3,000 target trajectories. Against the strongest transfer baseline at each budget, gains are 0.15/0.21/1.47 and 1.43/2.98/5.26 MRR/nDCG@10/HR@10 points, respectively (Table~\ref{tab:iclr_transfer}). More target alignment substantially improves all three methods. A fitted consumer thus makes source-trained updates useful in the new output space.

\begin{table}[!htbp]
\centering
\caption{\textbf{PENS-trained state reuse with target alignment.} MovieLens is withheld from source-encoder training. Methods receive the same target budget and one-positive/699-negative protocol. Scores use the 0--100 scale. The frozen-core control in Table~\ref{tab:iclr_frozen_transfer} uses a different slice, so absolute scores should be compared within each table.}
\label{tab:iclr_transfer}
\small
\setlength{\tabcolsep}{7pt}
\renewcommand{\arraystretch}{1.08}
\begin{tabular}{lccc ccc}
\toprule
& \multicolumn{3}{c}{\textbf{500 Target Trajectories}} & \multicolumn{3}{c}{\textbf{3K Target Trajectories}} \\
\cmidrule(lr){2-4}\cmidrule(lr){5-7}
\textbf{Method} & \textbf{MRR} & \textbf{nDCG@10} & \textbf{HR@10} & \textbf{MRR} & \textbf{nDCG@10} & \textbf{HR@10} \\
\midrule
UniSRec & 4.31 & 3.79 & 8.23 & 15.18 & 16.92 & 33.49 \\
RecGURU & 6.03 & 5.73 & 11.49 & 14.21 & 16.17 & 29.44 \\
\textbf{\pertide{}} & \textbf{6.18} & \textbf{5.94} & \textbf{12.96} & \textbf{16.61} & \textbf{19.90} & \textbf{38.75} \\
\bottomrule
\end{tabular}
\end{table}

\Needspace{18\baselineskip}
\begin{wraptable}{r}{0.61\textwidth}
\centering
\caption{\textbf{Learned updates survive frozen-core transfer.} PENS$\rightarrow$MovieLens with \pertide{} cores and the same head architecture. Trained heads use 3K target rows and matched candidates, loss, and optimization. Scores use the 0--100 scale. Table~\ref{tab:iclr_transfer} gives the separate transfer-baseline comparison.}
\label{tab:iclr_frozen_transfer}
\setlength{\tabcolsep}{2.5pt}
\begin{tabular}{@{}llccc@{}}
\toprule
\textbf{Frozen core} & \textbf{Head} & \textbf{MRR} & \shortstack{\textbf{nDCG}\\\textbf{@10}} & \textbf{HR@10} \\
\midrule
Random & Trained & 1.115 & 0.780 & 1.44 \\
PENS-trained & Untrained & 1.371 & 0.960 & 2.18 \\
\textbf{PENS-trained} & \textbf{Trained} & \textbf{16.344} & \textbf{16.002} & \textbf{20.28} \\
\bottomrule
\end{tabular}
\end{wraptable}
\textbf{Separating learned updates from an adaptable random representation.}
The controlled frozen-core experiment fixes the MovieLens head, 3,000 target rows, negatives, loss, and optimization. A frozen PENS-trained core with the trained head reaches 16.344 MRR; a frozen random core with the same trained head reaches 1.115; and the learned core with an untrained head reaches 1.371 (Table~\ref{tab:iclr_frozen_transfer}). The 15.23-point advantage isolates information learned in the source encoder, while the untrained-head result exposes the need for target alignment. This protocol is separate from the transfer-baseline comparison, so their absolute scores should not be combined. 

\paragraph{Schema Proximity And Transfer.} Nearby schema representations are associated with easier transfer in the separate diagnostic on 500 MovieLens trajectories from 99 users. Schema distance has Spearman correlations of 0.11 and 0.16 with frozen and adapted target risk, and $-0.09$ and $-0.11$ with reciprocal-rank gain, respectively. All four user-clustered 95\% confidence intervals exclude zero (Appendix~\ref{app:reuse_controls}, Table~\ref{tab:iclr_schema_distance}). The directions connect schema proximity to lower risk and greater benefit from reuse. The associations are modest and do not establish that the theorem's stability or conditional-label assumptions hold in these runs.

\section{Discussion And Limitations}
\label{sec:discussion-limitations}
The reusable object is the learned history update: different users retain separate states, and aligned consumers access those states for prediction or generation. The frozen-core comparison demonstrates useful source learning, while same-input controls locate an advantage beyond schema construction and parallel branches. On PENS, reference-content agreement also improves over DeepSeek-32B on RG-SU4 (41.32 versus 29.42) and BLEU-4 (27.64 versus 19.31). Human similarity ratings are 3.21 versus 3.12 on a six-point scale, while BERTScore is slightly lower, 90.84 versus 90.87 (Appendix~\ref{app:detailed_results_content_fidelity}). These findings support reference agreement rather than article factuality.

This reuse depends on supplied action mappings, informative content embeddings, and compatible downstream semantics. OpenAI evaluates constructed chronology, and MPChat combines useful generation results with deficits in response prediction and speaker identification (Appendix~\ref{app:mpchat boundary}). Stable traces and the three branch names alone do not certify readout contraction or semantic disentanglement. The objectives use globally fixed weights, and reported resource figures are estimates rather than matched hardware measurements. Uncertainty-aware action mappings, richer discourse representations, and adaptive objective weighting are concrete extensions of the current design.

\section{Conclusion}
\label{sec:conclusion}
Heterogeneous personalization histories can teach a common way to update user memory when their interactions are represented as actions on embedded items. The schema gives an exact invariance property, a conditional continuity bound under item changes, and a pooled-training bound that makes compatibility explicit. MTSH specifies how this evidence enters, persists, and is consumed, and \pertide{} implements those operations in a shared encoder. Prediction and generation results demonstrate multiple uses of the learned history state; the controlled frozen-core experiment provides direct evidence that source-trained updates remain useful after a domain change and consumer alignment. Same-input controls and trace interventions connect that reuse to the proposed update structure. The resulting opportunity is to learn history updates across compatible data sources and reuse them through task-appropriate consumers.

\clearpage

\section*{AI Use Statement}
Generative AI tools were used to cross-check technical reasoning, assist with literature retrieval, mathematical exposition and verification, manuscript drafting and editing, and code cleaning/documentation. The authors determine the research questions, experimental decisions, scientific claims, and final interpretations, and retain responsibility for the veracity, originality, and integrity of the paper.

\bibliographystyle{plainnat}
\bibliography{custom}

\clearpage
\phantomsection
\renewcommand{\contentsname}{Appendix Contents}
\setcounter{tocdepth}{2}
\tableofcontents
\clearpage
\appendix
\addtocontents{toc}{\protect\setcounter{tocdepth}{2}}
\section{Datasets}\label{app:datasets}
\begin{table}[!htbp]
\centering
\caption{\textbf{One event interface, five predictive and generative settings.} Ranking uses one positive target and 699 sampled negatives. Generation appends produced responses to the running trajectory. PSE denotes PerSEval; metric definitions are in Appendices~\ref{app:perseval} and~\ref{app:content accuracy metrics}.}
\label{tab:evaluation_coverage}
\setlength{\tabcolsep}{4pt}
\renewcommand{\arraystretch}{1.08}
\begin{tabularx}{\linewidth}{@{}p{0.16\linewidth}XX@{}}
\toprule
\textbf{Dataset} & \textbf{Event evidence and consumer} & \textbf{Primary metrics} \\
\midrule
MovieLens-1M & Positive/negative movie events; movie ranking & MRR, nDCG@10, HR@10 \\
MIND & Click/ignore news events; news ranking & MRR, nDCG@5, HR@10 \\
PENS & Click/skip news events; next-news ranking & MRR, nDCG@10 \\
PENS & News history and headline request; personalized generation & PSE-JSD, PSE-SU4, PSE-METEOR \\
OpenAI-Reddit & Feedback-derived summary events and request; eventized TL;DR generation & PSE-JSD, PSE-SU4, PSE-METEOR \\
\bottomrule
\end{tabularx}
\end{table}
\subsection{MIND-Large Dataset}\label{app:mind large dataset}
MIND (Large) \citep{dataset-mind-acl-2020} is the full-scale release of the Microsoft News recommendation dataset, collected from anonymized user behavior logs. It contains 1{,}000{,}000 users and 161{,}013 English news articles, and the associated interaction logs include over 15 million impressions on the platform \citep{dataset-mind-acl-2020}. The data is organized around impression logs that record the impression ID, anonymized user ID, timestamp, a user click-history list before the impression, and the displayed candidate news with click/non-click labels; the release also includes per-article metadata (e.g., title, abstract, category, and entities) \citep{dataset-mind-acl-2020}. A small subsampled variant (MIND-small) is also provided by the dataset maintainers by randomly sampling 50{,}000 users and their behavior logs from the full dataset \citep{dataset-mind-acl-2020}.

\subsection{PENS Dataset}\label{app:pens dataset}
PENS \citep{pens-acl-2021} contains 113{,}762 news articles spanning 15 topical categories. Each article includes an ID, title (avg.\ 10.5 words), body (avg.\ 549 words), and category label; titles are also linked to WikiData entities. The dataset further provides user interaction logs (impressions and clicks) and extends MIND news records with article bodies and personalized headline annotations \citep{dataset-mind-acl-2020}.

\textbf{PENS training set.}\label{sec:PENS dataset - training}
For training, 500k user--news impressions were sampled from June 13 to July 3, 2019. Each log entry is recorded as $[\mathrm{uID}, \mathrm{tmp}, \mathrm{clkNews}, \mathrm{uclkNews}, \mathrm{clkedHis}]$, where $\mathrm{clkNews}$ and $\mathrm{uclkNews}$ denote clicked and unclicked news in the impression, and $\mathrm{clkedHis}$ is the user’s previously clicked history sorted by click time. The resulting training pool for \pertide\ exhibits substantial temporal variation in user interactions, motivating explicit modeling of preference dynamics in offline personalization settings.

\textbf{PENS test set.}\label{sec:PENS dataset - testing}
To construct an offline evaluation set, 103 English-speaking students first reviewed 1{,}000 headlines and selected at least 50 articles to indicate their interests. Each participant was then asked to write preferred headlines for 200 previously unseen articles; assignments overlapped across participants.

Each test article has four participant-written references on average. Editorial checks verified factual correctness and removed incorrect headlines, yielding a high-quality set of personalized gold references for evaluation.

\subsection{OpenAI (Reddit) Dataset}\label{app:OpenAI Reddit Dataset}
The OpenAI feedback dataset \citep{stiennon2020summarize}, built from the Reddit TL;DR corpus \citep{openai-reddit}, contains 123{,}169 Reddit posts drawn from 29 subreddits. It provides both human-written and OpenAI-generated summaries and is organized into two splits: \textit{Comparisons} (training/validation) and \textit{Axis} (validation/testing). A curated subset of 1{,}038 posts was processed under 13 summarization policies, producing 7{,}713 summaries, which were evaluated by 64 annotators via paired comparisons with auxiliary signals (e.g., selection preference, confidence, and ratings along dimensions such as accuracy, coherence, coverage, and overall quality). In contrast to PENS, these annotations are not indexed by stable user identities or reading histories.

\subsection{MPChat Dataset}\label{app:mpchat dataset}
MPChat \citep{mpchat-2023-acl} is a multimodal persona-grounded dialogue dataset in which personas are represented as image--sentence pairs intended to capture episodic memories, and dialogues are curated to be grounded in these multimodal persona elements. The dataset is released with retrieval-style benchmarks for multimodal persona-grounded dialogue understanding, including next response prediction, grounding persona prediction, and speaker identification. In total, MPChat contains 15{,}000 multi-turn dialogues with 42{,}531 utterances written by 25{,}877 users, and it is split chronologically into 11{,}975/1{,}516/1{,}509 dialogues for train/validation/test \citep{mpchat-2023-acl}.
\section{Schema Construction And Generalizability}
\label{app:UIG and Schema-level Generalizability}

\subsection{Dataset-Specific Construction}\label{app:UIG Construction}
UIG represents an interaction as an action directed at a document or response node. Its construction depends on what a dataset records: explicit interactions, ratings of generated text, or dialogue turns. The following mappings preserve these differences in provenance while providing the encoder with the same event interface.

\textbf{PENS.}\label{app:UIG Construction from PENS-styled}
The recorded \textit{clkNews} and \textit{uclkNews} events become click and skip edges to document nodes. In the training pool, selected sentences from the corresponding articles supply surrogate response nodes, linked by request (\textit{genRes}) and generated-response (\textit{resGen}) transitions. These responses provide training targets; they are not observed user-authored summaries or evidence that the user's interests actually evolved in the sampled order. PENS has user-written test headlines, whereas its training histories primarily record article interactions~\citep{pens-acl-2021}.

\textbf{OpenAI--Reddit.}\label{app:UIG Construction from OpenAI-styled}
The source training and evaluation records are kept in separate pools before construction and augmentation. For each rater identifier, an article receives a derived positive label if at least one associated generated summary has a feedback-confidence score above the chosen threshold, 6 on the nine-point scale; the remaining articles receive a derived negative label. Among a rater's summaries for an article, the highest-rated summary supplies the response node. Article--response pairs and negative article nodes are randomly ordered into a trajectory. PerAugy is applied separately within both pools. The confidence threshold, rating-based selection, and random order are construction rules, rather than observed clicks, skips, or longitudinal preferences. Rater identifiers therefore index the constructed histories without establishing stable real-world user identity across time.

\textbf{MPChat.}\label{app:UIG Construction from MPChat-styled}
Dialogue histories pair recent context utterances with a speaker's subsequent response. Algorithm~\ref{alg:mpchat-stance} assigns positive or negative edges to context nodes using response cues, sentiment, and text similarity; ties and neutral cases are assigned negative. These are inferred stance labels, not explicit preference judgments. A response used to annotate a historical event belongs to the observed prefix; a response being predicted must remain outside that prefix. This cutoff is necessary to interpret the response-prediction task as prediction from history.

\begin{algorithm}[t]
\caption{\textsc{HybridStanceInfer}: MPChat-Style Turn-Level Stance Inference}
\label{alg:mpchat-stance}
\begin{algorithmic}[1]
\REQUIRE Context utterances $\{d\}$; target response $s$; sentiment encoder $\mathrm{SentEnc}$; text encoder $\mathrm{TxtEnc}$; cue lexicons $\{\mathcal{C}\}$
\ENSURE Binary stance labels for linked context turns

\STATE $x_s \gets \mathrm{lower}(s)$;\;\; $\mathbf{e}_s \gets \mathrm{TxtEnc}(x_s)$;\;\; $\sigma_s \gets \mathrm{SentEnc}(x_s)$
\STATE $\kappa \gets \mathrm{splitClauses}(x_s)$

\FOR{each recent context turn $d \in \{d\}$}
    \STATE $x_d \gets \mathrm{lower}(d)$;\;\; $\mathbf{e}_d \gets \mathrm{TxtEnc}(x_d)$
    \STATE $p \gets 0,\; n \gets 0$ \hfill{\small (positive / negative evidence)}
    \FOR{each clause $\kappa_i \in \kappa$}
        \STATE $(p,n) \gets \mathrm{CueScore}(\kappa_i,\{\mathcal{C}\},p,n)$
    \ENDFOR
    \STATE $(p,n) \gets \mathrm{SentScore}(\sigma_s,p,n)$
    \STATE $(p,n) \gets \mathrm{SimScore}(\cos(\mathbf{e}_s,\mathbf{e}_d),p,n)$
    \STATE $\ell_d \gets \mathbb{I}[p > n]$ \hfill{\small (neutral $\rightarrow$ negative)}
\ENDFOR

\RETURN $\{\ell_d\}$
\end{algorithmic}
\end{algorithm}

\begin{figure}[t]
\centering
\includegraphics[width=0.5\linewidth]{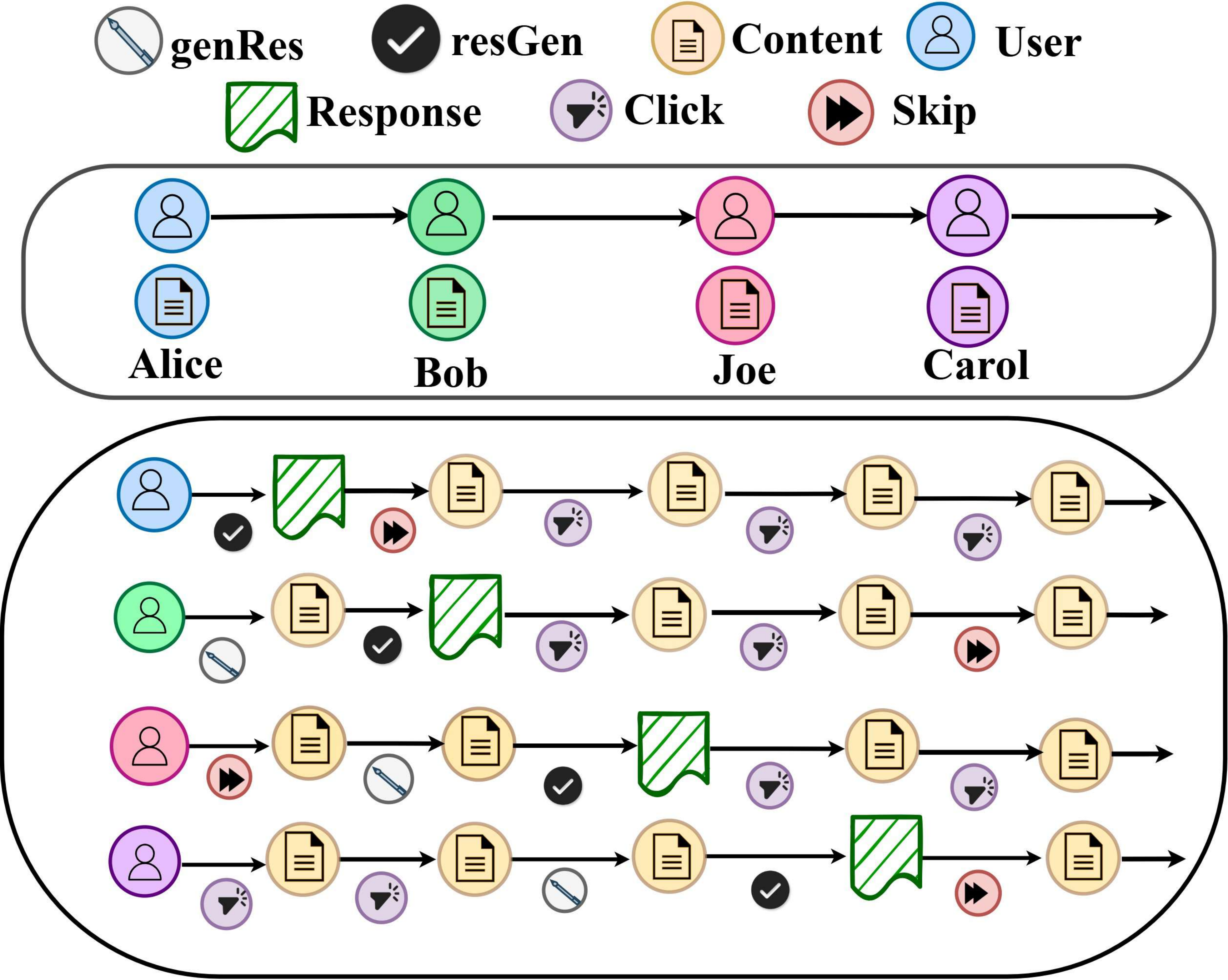}
\caption{Construction of the User Interaction Graph (UIG) from multi-user, multi-turn dialogue. Turn-level interactions are annotated with behavioral actions and sentiment signals, then projected into a unified UIG over participating users and content nodes.}
\label{fig:pipeline}
\end{figure}

\subsection{What The Schema Shares}
\label{app:schema_generalizability}
Let domain $D$ have native actions $\mathcal A_D$ and items $\mathcal X_D$. An item may be a document, movie, or response node. The schema maps a native event $(\alpha,x)$ to
\[
S_D(\alpha,x)=\big(\rho_D(\alpha),\psi_D(x)\big),\qquad
\rho_D:\mathcal A_D\to\mathcal A,\quad \psi_D:\mathcal X_D\to\mathbb R^d.
\]
Here $\mathcal A$ contains positive, negative, request, and generated-response transitions; the latter two distinguish the input and output sides of a command-like interaction. The same parameterized update consumes these roles and item embeddings across users and domains. User states remain separate. Thus Alice's watched movie and Bob's watched series can exercise the same update rule without asserting that Alice and Bob have the same preferences.

For fixed parameters, write the implemented update in compact form as
\begin{align}
 b_i&=\mu(a_i)\odot e_i,\quad \mu(a)=\sigma(\mathrm{MLP}_{\rm gate}(a)),\nonumber\\
 m_i^r&=A_r m_{i-1}^r+B b_i,\quad r\in\{L,S,E\},\nonumber\\
 h_i&=\Phi([m_i^L;m_i^S;m_i^E])+W_\psi b_i+b_\psi,\label{eq:schema-implemented-update}\\
 z_i&=F_{g_i}(h_i,z_{i-1}).\nonumber
\end{align}
The last line includes the command-conditioned routing through the preceding task-ready state (Appendix~\ref{app:arch-details}). Keeping the residual input and this state dependence explicit is necessary when extending a trace bound to the complete encoder.

\begin{proposition}[Invariance to native action relabeling]\label{thm: generalization - Invariance to Native Action Relabeling}
Suppose two trajectories have identical schema pairs $(a_i,e_i)$ at every step, identical command seeds $g_i$, and equal initial trace and task-ready states. Their $m_i^r$, $h_i$, and $z_i$ are equal at every step under the same deterministic encoder parameters.
\end{proposition}
\textit{Proof.} Equal roles and embeddings give equal $b_i$. Equal previous states and inputs give equal trace updates, residual fusion, and readout in Eq.~\eqref{eq:schema-implemented-update}. Induction from the equal initial states proves the result. $\square$

Native names such as watch, click, and open therefore cease to matter once both the role and item representation match. A role match alone does not make two items, histories, or users equivalent.

\begin{proposition}[Stability under item-embedding changes]\label{thm: generalization - Stability Under Target-Locus Embedding Shifts}
Consider equal-length trajectories with equal initial states, roles, and command seeds. Let their item-embedding differences be $\delta_i=\|e_i-e'_i\|_2$. Assume $\|A_r\|_2\leq\beta_r<1$. Then
\[
\|m_i^r-m_i^{\prime r}\|_2\leq\|B\|_2\sum_{j=1}^i\beta_r^{i-j}\delta_j.
\]
If $\Phi$ is $L_\Phi$-Lipschitz, define
\[
H_i=L_\Phi\|B\|_2\sum_r\sum_{j=1}^i\beta_r^{i-j}\delta_j+\|W_\psi\|_2\delta_i.
\]
Then $\|h_i-h'_i\|_2\leq H_i$. If the readout satisfies
\[
\|F_g(h,z)-F_g(h',z')\|_2\leq L_H\|h-h'\|_2+L_Z\|z-z'\|_2
\]
on the states under consideration, then
\[
\|z_i-z'_i\|_2\leq L_H\sum_{j=1}^i L_Z^{i-j}H_j.
\]
\end{proposition}
\textit{Proof.} Since the shared sigmoid gate has entries in $[0,1]$, $\|b_i-b'_i\|_2\leq\delta_i$. Unrolling the trace difference from zero initial difference gives $\sum_{j=1}^i A_r^{i-j}B(b_j-b'_j)$; the operator-norm bound yields the first inequality. Apply the Lipschitz bound to concatenated traces, upper-bound their concatenated norm by the sum of component norms, and add the residual difference to obtain $H_i$. Finally, the readout gives $D_i\leq L_HH_i+L_ZD_{i-1}$ with $D_0=0$; unrolling gives the stated bound. $\square$

For the implemented readout, $\pi=\operatorname{softmax}(W_z z+W_g g)$ and $F_g(h,z)=(W_\pi\pi)\odot h+b_G$. On a set with $\|h\|_\infty\leq H$, valid constants are $L_H=\|W_\pi\|_2$ and $L_Z=H\|W_\pi\|_2L_{\rm sm}\|W_z\|_2$, where $L_{\rm sm}$ is a Lipschitz constant for softmax. The finite-prefix bound does not require $L_Z<1$; a uniform long-horizon readout bound needs an additional contraction condition. Stable traces alone do not establish that condition.

\textbf{Changing only the item at a fixed step.}
Fix Alice's preceding history, action, command, and timestep $t$, and replace the current movie by another item. At that step, the trace difference is $B\operatorname{diag}(\mu(a_t))(e_t-e'_t)$. If all subsequent inputs agree, its contribution at lag $k$ has norm at most $\|B\|_2\beta_r^k\|e_t-e'_t\|_2$. Nearby embeddings therefore induce nearby continuous states under the stated assumptions. Moving progressively farther away enlarges this upper bound; it does not force the states to separate monotonically. Semantic similarity must be represented by the embedding map, and continuity of a state does not imply continuity of a discrete recommendation or generated token.

\subsection{Transfer And Pooled Training}
For equal-length schema trajectories and a command seed $g$, let $x=(S_D(\tau),g)$ and define
\[
d_{\rm sch}(x,x')=\sum_{i=1}^T\omega_i\left[\mathbf1(a_i\ne a'_i)+\|e_i-e'_i\|_2\right]+\lambda_g\|g-g'\|_2.
\]
Positive $\omega_i$ and $\lambda_g$ give a metric on these schema representations. Allowing zero weights gives a pseudometric, interpreted on its quotient. Comparisons of variable-length histories require a common padding or alignment convention. Trace-influence weights, such as $\omega_i=\sum_r\beta_r^{T-i}$, emphasize events retained at the endpoint; they do not by themselves certify a loss Lipschitz constant for the stateful readout.

\begin{proposition}[Schema-level transfer bound]\label{thm: generalization - Schema-Level Transfer Bound.}
Let source and target joint distributions over $(x,y)$ have input marginals $P_S,P_T$ with finite first moments under $d_{\rm sch}$. For fixed $\theta$, let $q_D(x)=\mathbb E_D[\ell(f_\theta(x),y)\mid x]$ for a bounded loss. Assume a version of $q_S$ is defined on both supports and is $L$-Lipschitz. Set $\Delta_{S,T}(\theta)=\mathbb E_{P_T}|q_T(x)-q_S(x)|$. Then
\[
R_T(\theta)\leq R_S(\theta)+L W_1(P_S,P_T)+\Delta_{S,T}(\theta).
\]
\end{proposition}
\textit{Proof.} Add and subtract $\mathbb E_{P_T}q_S$ in $R_T=\mathbb E_{P_T}q_T$. The conditional-loss difference is at most $\Delta_{S,T}$, while Kantorovich--Rubinstein duality bounds $\mathbb E_{P_T}q_S-\mathbb E_{P_S}q_S$ by $L W_1(P_S,P_T)$. $\square$

The bound requires both comparable input geometry and compatible target or command semantics. A small embedding distance alone cannot control conflicting labels. If a source-risk bound $R_S\leq\widehat R_S+\mathcal G_S(n,\delta)$ holds uniformly over the selected model class, it can be substituted into this inequality for a model learned from the same sample.

For $M$ domains, let $n_m$ be fixed sample counts, $N=\sum_m n_m$, and $\pi_m=n_m/N$. Let $P_\pi=\sum_m\pi_mP_m$ be the mixture of input marginals, with joint mixture and risk defined using the same weights. For transfer from a pooled mixture to target $t$, write $\Delta_{t,\pi}(\theta)=\mathbb E_{P_t}|q_t(x)-q_\pi(x)|$, where $q_\pi$ is the conditional expected loss under the joint mixture. The pooled empirical objective is
\[
\widehat R_\pi(\theta)=\frac1N\sum_{m=1}^M\sum_{i=1}^{n_m}\ell(f_\theta(x_{m,i}),y_{m,i}),\qquad
\widehat\theta_\pi\in\arg\min_{\theta\in\Theta}\widehat R_\pi(\theta).
\]
\begin{proposition}[Common-substrate training bound]\label{thm: generalization - Common-Substrate Training Bound.}
Assume examples are independent within and across domains, with each domain's examples drawn from its specified joint distribution. Let $\ell\in[0,1]$ and
$\mathcal F=\{(x,y)\mapsto\ell(f_\theta(x),y):\theta\in\Theta\}$.
Let $\widehat{\mathfrak R}_N(\mathcal F)$ be its empirical Rademacher complexity on the pooled sample. Assume the preceding transfer proposition applies to the mixture and target for every $\theta\in\Theta$, with a common constant $L$. With probability at least $1-\delta$,
\begin{align*}
R_t(\widehat\theta_\pi)\leq{}&\widehat R_\pi(\widehat\theta_\pi)+2\widehat{\mathfrak R}_N(\mathcal F)+3\sqrt{\frac{\log(2/\delta)}{2N}}\\
&+L W_1(P_t,P_\pi)+\Delta_{t,\pi}(\widehat\theta_\pi).
\end{align*}
The empirical-risk term can also be replaced by $\widehat R_\pi(\theta^\star)$ for any comparator $\theta^\star\in\Theta$.
\end{proposition}
\textit{Proof.} Symmetrization with independent ghost observations from each corresponding domain bounds the expected uniform risk gap by twice the expected Rademacher complexity. Changing one observation changes the uniform gap by at most $1/N$ and the empirical complexity by at most $1/N$. Applying bounded differences to each, with failure probability $\delta/2$, gives the uniform bound $R_\pi(\theta)\leq\widehat R_\pi(\theta)+2\widehat{\mathfrak R}_N(\mathcal F)+3\sqrt{\log(2/\delta)/(2N)}$. It holds at the data-dependent $\widehat\theta_\pi$. Apply the transfer inequality with its common $L$; empirical optimality gives the comparator form. $\square$

\begin{corollary}[Estimation bound under compatible pooling]\label{cor:schema-compatible-pooling}
If $P_t=P_\pi$ and $\Delta_{t,\pi}(\theta)=0$, only empirical risk and the estimation terms remain. If the same class admits $\widehat{\mathfrak R}_n(\mathcal F)\leq C/\sqrt n$ for the samples considered, the resulting estimation upper bound for $N\geq n_t$ is no larger than the corresponding bound for a domain-only sample of size $n_t$.
\end{corollary}
This comparison concerns estimation bounds, not a guarantee of lower empirical risk or higher accuracy after pooling. Correlated augmentations do not automatically satisfy the independent-example assumption, and the bounded analysis loss is not the unbounded cross-entropy used in training. The theory identifies conditions under which shared training is compatible; the frozen-encoder and schema-matched experiments test whether learned updates are useful in the reported settings. MTSH then specifies how those updates enter, persist in, and are read from the shared encoder.

\section{UIG Augmentation}
\label{app:peraugy_aug}
After inserting surrogate response-nodes via \textit{genRes} and \textit{resGen} interactions in the base trajectories, we apply \texttt{PerAugy}-style \cite{peraugy-tmlr-2025} augmentation to (i) increase trajectory diversity while preserving selected local order constraints, and (ii) reduce local incoherence introduced by surrogate responses (r-nodes). We apply augmentation to u-tier trajectories $\tau^{u_j}$ before training, yielding an expanded trajectory pool used to instantiate $\mathcal{T}^{\text{PENS}}$ (and its perturbed variants).
\begin{figure*}[t]
\centering
\includegraphics[width=0.65\linewidth]{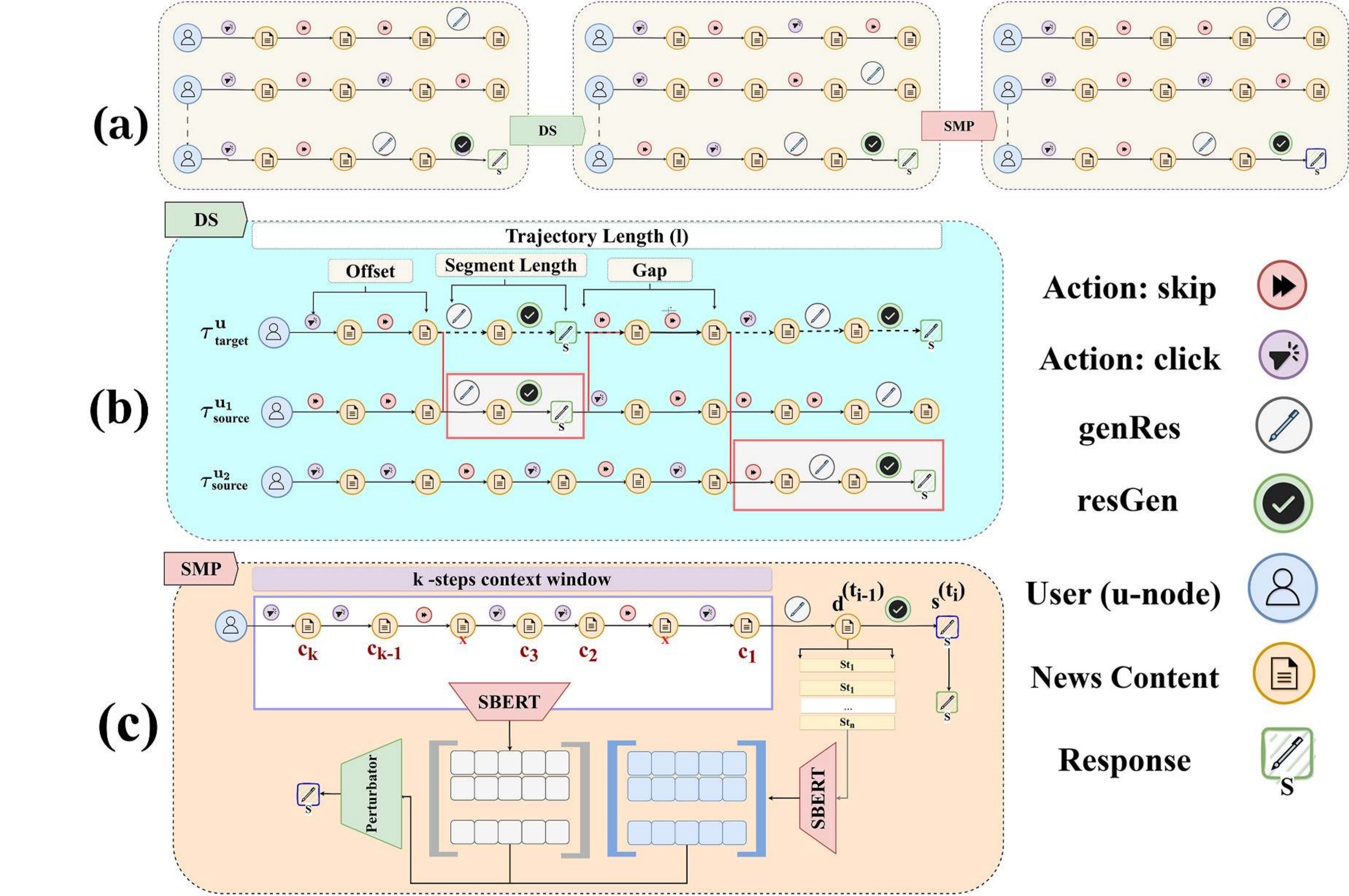}
\caption{\textbf{UIG Augmentation Pipeline:} \cite{peraugy-tmlr-2025} proposed framework -- (a) \textit{Pipeline overview depicting the two-step augmentation}, (b) \textit{Double Shuffling (DS) to ensure cross-trajectory augmentation and induce diffusion}, (c) \textit{Stochastic Markovian Perturbation (SMP) to smoothen the response-nodes and modulate random diffusion incorporated in DS stage}.}
\label{fig:peraugy pipeline}
\end{figure*}

\paragraph{Double shuffling.}
Given $m$ trajectories sampled within the same permitted pool, we select one target trajectory $\tau^{u_j}_{\text{target}}$ and treat the remaining $m{-}1$ trajectories as sources. We construct a synthetic trajectory by (i) preserving an early prefix of the target determined by an offset $O$, and (ii) substituting later contiguous segments using source trajectory snippets. To preserve local temporal realism, we enforce a gap hyperparameter $G$ such that after each substitution, $G$ consecutive steps from the target are kept intact.

\begin{algorithm}[t]
\caption{\textsc{DoubleShuffling}: Cross-Trajectory Segment Substitution}
\label{alg:DS}
\begin{algorithmic}[1]
\REQUIRE Target trajectory $\tau^{u_j}$; source trajectories $\{\tau^{u_i}\}_{i=1}^{m-1}$; offset distribution $\mathcal{P}_O$; gap size $G$
\ENSURE Synthetic trajectory $\tilde{\tau}^{u_j}$

\STATE Sample offset $O \sim \mathcal{P}_O$
\STATE $\tilde{\tau}^{u_j} \gets \tau^{u_j}[1{:}O]$ \hfill{\small (preserve prefix)}
\STATE $t \gets O$; $T\gets|\tau^{u_j}|$

\FOR{each source trajectory $\tau^{u_i}$}
    \STATE If $t\geq T$, stop; otherwise sample a segment of length $s\leq T-t$ from $\tau^{u_i}$
    \STATE $\tau_{\text{seg}}\gets\mathrm{SampleSegment}(\tau^{u_i},s)$; $g\gets\min(G,T-t-s)$
    \STATE $\tilde{\tau}^{u_j} \gets \tilde{\tau}^{u_j} \oplus \tau_{\text{seg}}$ \hfill{\small (substitute)}
    \STATE $\tilde{\tau}^{u_j} \gets \tilde{\tau}^{u_j} \oplus \tau^{u_j}[t{+}s{+}1{:}t{+}s{+}g]$ \hfill{\small (protected gap)}
    \STATE $t \gets t+s+g$
\ENDFOR

\STATE $\tilde{\tau}^{u_j} \gets \tilde{\tau}^{u_j} \oplus \tau^{u_j}[t+1{:}T]$
\RETURN $\tilde{\tau}^{u_j}$
\end{algorithmic}
\end{algorithm}

\paragraph{Stochastic Markovian perturbation.}
While shuffling increases structural diversity, it can introduce surrogate r-nodes whose content is poorly aligned with the target user’s recent context, reducing temporal coherence. We therefore apply a local, context-aware perturbation that replaces a substituted r-node with a sentence that better matches the immediate history of the \emph{target} trajectory. Consider a substituted summary node $r^{(t_i)}$ in $\tau^{u_j}_{\text{DS}}$. Let $\tau^{u_j}_{c_k}$ denote a backward context window of the previous $k$ nodes $\{c_q\}_{q=1}^{k}$ (most recent first), and let $d^{(t_{i-1})}$ be the preceding d-node. We segment $d^{(t_{i-1})}$ into candidate sentences $\{st_p\}_{p=1}^{n_{d^{(t_{i-1})}}}$. Using SBERT embeddings, we compute an influence-weighted distance between each candidate sentence and the context window, with temporal relevance enforced by an exponential decay weight $w_q = \exp(-\lambda\cdot \mathrm{pos}(c_q))$ and $\mathrm{pos}(c_1)=0$.

Let $\mathbf{\sigma} \in \mathbb{R}^{n_{d^{(t_{i-1})}} \times k}$ be the distance matrix with entries
$\mathbf{\sigma}_{(p,q)}=\sigma(\mathbf{e}_{st_p},\mathbf{e}_{c_q})$, where $\sigma(v,w)=\|v-w\|_2/\sqrt d$ is root mean squared distance (RMSD) for $d$-dimensional embeddings. We select:
\begin{equation}
\label{eq:SMP_app}
\hat{r}^{(t_i)}=st_{p^*},\qquad p^*
=
\arg\min_{1\leq p\leq n_{d^{(t_{i-1})}}}
\sum_{q=1}^{k}
\exp\!\big(-\lambda \cdot \mathrm{pos}(c_q)\big)\;
\sigma(\mathbf{e}_{st_p},\mathbf{e}_{c_q}).
\end{equation}
Equivalently, per-sentence scores can be written as $\mathbf{\sigma}\,\mathbf{w}$ where $\mathbf{w}_q=\exp(-\lambda\cdot \mathrm{pos}(c_q))$.

\begin{algorithm}[t]
\caption{\textsc{SummaryPerturbation}: Local Summary Rewriting}
\label{alg:SMP}
\begin{algorithmic}[1]
\REQUIRE Trajectory $\tau^{u_j}$; window size $k$; decay $\lambda$; perturbation rate $p_{\mathrm{SMP}}$
\ENSURE Perturbed trajectory $\tilde{\tau}^{u_j}$

\STATE $\tilde{\tau}^{u_j} \gets \tau^{u_j}$
\FOR{each substituted summary node $s^{(t_i)} \in \tilde{\tau}^{u_j}$}
    \STATE Sample $z \sim \mathrm{Bernoulli}(p_{\mathrm{SMP}})$
    \IF{$z=1$}
        \STATE $\mathcal{C} \gets \mathrm{ContextWindow}(\tilde{\tau}^{u_j},t_i,k)$
        \STATE $\hat{s}^{(t_i)} \gets \mathrm{RewriteSummary}(s^{(t_i)},\mathcal{C},\lambda)$ \hfill{\small (Eq.~\eqref{eq:SMP_app})}
        \STATE Replace $s^{(t_i)} \leftarrow \hat{s}^{(t_i)}$
    \ENDIF
\ENDFOR
\RETURN $\tilde{\tau}^{u_j}$
\end{algorithmic}
\end{algorithm}

For PENS, these operations construct the augmented training pool $\tau^{\text{PENS-D}}$. For OpenAI--Reddit, they are applied separately within the native training and evaluation pools, as described in Appendix~\ref{app:UIG Construction from OpenAI-styled}; source segments are not exchanged between pools. The perturbed s-nodes provide additional, coherence-aware supervision for personalized summarization tasks.
\section{\pertide\ Architectural Details}
\label{app:arch-details}

\subsection{Discrete-Time SSM Stability}
\label{app:stability-ssm}

We first state the stability condition used by the discrete state-space traces in \pertide. Consider the linear recursion
\[
\mathbf{h}_{t}
=
\mathbf{A}\mathbf{h}_{t-1}
+
\mathbf{B}\mathbf{x}_{t},
\]
with homogeneous dynamics
\[
\mathbf{h}_{t}
=
\mathbf{A}\mathbf{h}_{t-1}.
\]
The stability of the homogeneous part determines whether past perturbations decay or persist without bound.

\begin{proposition}[Discrete-time stability by spectral radius]
\label{prop:discrete_ssm_stability}
Let $\mathbf{A}\in\mathbb{C}^{d\times d}$. If
\[
\rho(\mathbf{A})<1,
\]
where $\rho(\mathbf{A})=\max_i |\lambda_i(\mathbf{A})|$ is the spectral radius, then the homogeneous recursion
\[
\mathbf{h}_{t}=\mathbf{A}\mathbf{h}_{t-1}
\]
is asymptotically stable, i.e.,
\[
\lim_{t\to\infty}\mathbf{A}^{t}\mathbf{h}_{0}=\mathbf{0}
\]
for every initial state $\mathbf{h}_{0}$.
\end{proposition}

\begin{proof}
Since $\rho(\mathbf{A})<1$, all eigenvalues of $\mathbf{A}$ lie strictly inside the unit disk. By the standard spectral-radius characterization of matrix powers, this implies
\[
\lim_{t\to\infty}\mathbf{A}^{t}=\mathbf{0}.
\]
Therefore,
\[
\lim_{t\to\infty}\mathbf{h}_{t}
=
\lim_{t\to\infty}\mathbf{A}^{t}\mathbf{h}_{0}
=
\mathbf{0},
\]
which proves asymptotic stability.
\end{proof}

\begin{corollary}[Modal stability for diagonalizable transitions]
\label{cor:modal_stability}
Suppose $\mathbf{A}$ is diagonalizable as
\[
\mathbf{A}
=
\mathbf{Q}\operatorname{diag}(a_1,\ldots,a_d)\mathbf{Q}^{-1}.
\]
Then the homogeneous dynamics are asymptotically stable if
\[
|a_i|<1
\qquad
\text{for all } i=1,\ldots,d.
\]
For $0<|a_i|<1$, the contribution of mode $i$ after lag $\Delta$ decays as
\[
|a_i|^{\Delta}
=
\exp\!\left(-\frac{\Delta}{\tau_i}\right),
\qquad
\tau_i
=
-\frac{1}{\log |a_i|}.
\]
\end{corollary}

\begin{proof}
For a diagonalizable transition,
\[
\mathbf{A}^{\Delta}
=
\mathbf{Q}
\operatorname{diag}(a_1^{\Delta},\ldots,a_d^{\Delta})
\mathbf{Q}^{-1}.
\]
Thus $\mathbf{A}^{\Delta}\to\mathbf{0}$ whenever $|a_i|<1$ for every mode. The effective decay relation follows by writing
\[
|a_i|^{\Delta}
=
\exp(\Delta\log |a_i|)
=
\exp\!\left(-\frac{\Delta}{\tau_i}\right),
\]
which gives
\[
\tau_i
=
-\frac{1}{\log |a_i|}.
\]
\end{proof}

A zero modal base removes its homogeneous contribution after one step; the logarithmic time constant above is defined only for nonzero bases. For the shared orthogonal eigenbasis used here, the spectral norm equals the largest absolute modal base, so modal bounds also supply the contraction constants used in Appendix~\ref{app:schema_generalizability}.

\paragraph{Connection to continuous-time eigenvalues.}
If the discrete transition is induced by a continuous-time generator $\mathbf{\Lambda}$ with eigenvalues $\lambda_i\in\mathbb{C}$, then continuous-time stability requires
\[
\operatorname{Re}(\lambda_i)<0.
\]
Under exponential discretization,
\[
a_i=\exp(\lambda_i\Delta t),
\]
so
\[
|a_i|
=
\exp(\operatorname{Re}(\lambda_i)\Delta t).
\]
Hence,
\[
|a_i|<1
\quad
\Longleftrightarrow
\quad
\operatorname{Re}(\lambda_i)<0.
\]
Under Euler-style discretization,
\[
a_i=1+\eta\lambda_i,
\]
and stability requires
\[
|1+\eta\lambda_i|<1.
\]
Thus Euler-style maps impose an explicit step-size constraint coupling $\eta$ and $\lambda_i$. The following analysis assumes these discrete-mode constraints. The spectral formulas alone do not ensure the Euler constraint or order the retention ranges of the three branches.

\subsection{SSM Hyperparameter Choices}
\label{app:ssm-hparams}

We next give sufficient rules for selecting the short-memory horizon and the event-step size. These rules are not claimed to be necessary; they provide simple conservative choices that make the intended timescale behavior explicit.

\paragraph{Notation.}
Let $\varepsilon\in(0,1/2)$ be a small tolerance, let
\[
\sigma(x)=\frac{1}{1+e^{-x}}
\]
be the sigmoid function, and let $\tau>0$ denote a desired exponential decay time constant.

\subsubsection*{Short-Memory Retention And Lag}
The short-trace modal base is $a_S=\sigma(-\kappa\lambda_S)$, with $\lambda_S>0$. The parameter $\kappa$ controls per-step retention; it is not the lag. Retention at lag $H$ is $a_S^H$.
\begin{lemma}[Sufficient per-step suppression]\label{lem:choose_k}
For $\lambda_S>0$ and $\varepsilon\in(0,1/2)$,
\begin{equation}\label{eq:choose_k_solution}
\kappa\geq\frac1{\lambda_S}\log\frac{1-\varepsilon}{\varepsilon}
\end{equation}
implies
\begin{equation}\label{eq:choose_k_goal}
a_S=\sigma(-\kappa\lambda_S)\leq\varepsilon.
\end{equation}
\end{lemma}
\begin{proof}
The inequality $1/(1+e^{\kappa\lambda_S})\leq\varepsilon$ is equivalent to $e^{\kappa\lambda_S}\geq(1-\varepsilon)/\varepsilon$. Taking logarithms gives the stated condition.
\end{proof}
For a fixed base $a_S\in(0,1)$, a lag $H\geq\log\varepsilon/\log a_S$ instead ensures $a_S^H\leq\varepsilon$. These are distinct choices: one selects a base, the other reads its resulting horizon.

\begin{corollary}[Rapid suppression beyond the short-memory horizon]
\label{cor:short_tail_suppression}
Let $a_S=\sigma(-\kappa\lambda_S)\in(0,1)$ be a short-trace modal base. If $|u_{t-\ell}|\leq R$ for all $\ell$, then the total contribution from lags at least $W$ is bounded by
\[
\left|
\sum_{\ell=W}^{\infty}
a_S^{\ell}u_{t-\ell}
\right|
\leq
R\frac{a_S^{W}}{1-a_S}.
\]
Thus, when $a_S$ is small, evidence outside a short window is geometrically suppressed.
\end{corollary}

\begin{proof}
By the triangle inequality,
\[
\left|
\sum_{\ell=W}^{\infty}
a_S^{\ell}u_{t-\ell}
\right|
\leq
\sum_{\ell=W}^{\infty}
a_S^{\ell}|u_{t-\ell}|
\leq
R\sum_{\ell=W}^{\infty}a_S^{\ell}.
\]
Since $a_S\in(0,1)$,
\[
\sum_{\ell=W}^{\infty}a_S^{\ell}
=
\frac{a_S^{W}}{1-a_S}.
\]
This proves the bound.
\end{proof}

\subsubsection*{Event-Kernel Step Size Selection}

For the episodic trace, let $\lambda_E<0$ denote a stable continuous-time decay rate. Under Euler-style discretization with step size $\eta>0$, the per-step base is
\[
a_E=1+\eta\lambda_E.
\]
The event influence at lag $\Delta=t-k$ is therefore
\begin{equation}
\label{eq:event_kernel_discrete}
w^{(E)}_{t,k}
=
(1+\eta\lambda_E)^{t-k}.
\end{equation}
We choose $\eta$ so that this discrete kernel exactly matches the target exponential decay
\[
\exp\!\left(-\frac{t-k}{\tau}\right).
\]

\begin{lemma}[Exact event-kernel matching]
\label{lem:choose_eta}
Let $\lambda_E<0$ and $\tau>0$. Define
\begin{equation}
\label{eq:choose_eta_solution}
\eta
=
\frac{e^{-1/\tau}-1}{\lambda_E}.
\end{equation}
Then
\[
1+\eta\lambda_E
=
e^{-1/\tau},
\]
and for every lag $\Delta=t-k$,
\[
(1+\eta\lambda_E)^{t-k}
=
\exp\!\left(-\frac{t-k}{\tau}\right).
\]
Moreover, $\eta>0$ and
\[
0<1+\eta\lambda_E<1,
\]
so the induced discrete event mode is stable.
\end{lemma}

\begin{proof}
Substituting Eq.~\eqref{eq:choose_eta_solution} gives
\[
1+\eta\lambda_E
=
1+
\left(\frac{e^{-1/\tau}-1}{\lambda_E}\right)\lambda_E
=
e^{-1/\tau}.
\]
Raising both sides to the power $t-k$ yields
\[
(1+\eta\lambda_E)^{t-k}
=
e^{-(t-k)/\tau}.
\]
Since $\tau>0$, we have $e^{-1/\tau}\in(0,1)$. Therefore $e^{-1/\tau}-1<0$. Because $\lambda_E<0$, Eq.~\eqref{eq:choose_eta_solution} implies $\eta>0$. Finally,
\[
1+\eta\lambda_E=e^{-1/\tau}\in(0,1),
\]
which proves stability.
\end{proof}

\subsection{Kernel Interpretation Of The Three \pertide\ Traces}
\label{app:kernel-timescale-proof}

Each branch induces a geometric lag kernel. Different spectral parameterizations can select different bases. Their attainable kernels nevertheless overlap. Distinct timescales require different learned bases; branch names do not prove semantic disentanglement. Trace-subset and temporal-regime experiments test their functional differentiation.

\paragraph{General tri-trace recursion.}
For each trace $r\in\{L,S,E\}$, \pertide updates a memory state as
\begin{equation}
\label{eq:general_trace_recursion}
m_t^{(r)}
=
A_r m_{t-1}^{(r)}
+
B b_t,
\end{equation}
where $b_t$ is the action-gated event cell, $B$ is the input operator, and $A_r$ is the trace-specific transition matrix. Assume that
\[
A_r
=
Q_r
\operatorname{diag}
(a^{(r)}_1,\ldots,a^{(r)}_d)
Q_r^{-1}.
\]
The three \pertide spectral maps induce
\[
a_i^{(L)}
=
\exp(\delta\lambda_{L,i}),
\qquad
a_i^{(S)}
=
\sigma(-\kappa\lambda_{S,i}),
\qquad
a_i^{(E)}
=
1+\eta\lambda_{E,i}.
\]
For clarity, we state the results for real stable modes satisfying
\[
0\leq a_i^{(r)}<1.
\]
For complex modes, the same statements hold for the decay envelope after replacing $a_i^{(r)}$ by $|a_i^{(r)}|$.

\begin{proposition}[Mode-wise lag kernels induced by the trace recursion]
\label{prop:trace_mode_kernels}
For each trace $r\in\{L,S,E\}$ and mode $i$, the contribution of an input at lag $\ell=t-s$ to the modal state is weighted by
\[
k_i^{(r)}(\ell)
=
\left(a_i^{(r)}\right)^\ell,
\qquad
\ell\geq 0.
\]
Thus each trace implements an explicit geometric lag kernel in its eigenbasis.
\end{proposition}

\begin{proof}
Move Eq.~\eqref{eq:general_trace_recursion} to the eigenbasis of trace $r$:
\[
\tilde m_t^{(r)}
=
Q_r^{-1}m_t^{(r)},
\qquad
u_t^{(r)}
=
Q_r^{-1}Bb_t.
\]
Then
\[
\tilde m_t^{(r)}
=
\operatorname{diag}
(a^{(r)}_1,\ldots,a^{(r)}_d)
\tilde m_{t-1}^{(r)}
+
u_t^{(r)}.
\]
Therefore, each mode evolves independently:
\[
\tilde m_{t,i}^{(r)}
=
a_i^{(r)}
\tilde m_{t-1,i}^{(r)}
+
u_{t,i}^{(r)}.
\]
Unrolling gives
\[
\tilde m_{t,i}^{(r)}
=
\left(a_i^{(r)}\right)^t
\tilde m_{0,i}^{(r)}
+
\sum_{s=1}^{t}
\left(a_i^{(r)}\right)^{t-s}
u_{s,i}^{(r)}.
\]
Hence an input at lag $\ell=t-s$ is weighted by
\[
\left(a_i^{(r)}\right)^\ell.
\]
This proves the claim.
\end{proof}

\begin{proposition}[Effective memory length of a stable modal kernel]
\label{prop:effective_memory_length}
Let $a\in(0,1)$ and define the unnormalized lag kernel
\[
k_a(\ell)=a^\ell,
\qquad
\ell=0,1,2,\ldots.
\]
Its normalized version is
\[
\bar k_a(\ell)
=
(1-a)a^\ell.
\]
Then the expected lag is
\[
\mathbb{E}_{\bar k_a}[\ell]
=
\frac{a}{1-a}.
\]
The half-life is
\[
h(a)
=
\frac{\log(1/2)}{\log a}.
\]
Both quantities increase as $a$ approaches one.
\end{proposition}

\begin{proof}
Since $\bar k_a$ is the probability mass function of a geometric distribution over nonnegative lags,
\[
\sum_{\ell=0}^{\infty}(1-a)a^\ell=1.
\]
Its expected lag is
\[
\mathbb{E}_{\bar k_a}[\ell]
=
(1-a)\sum_{\ell=0}^{\infty}\ell a^\ell
=
(1-a)\frac{a}{(1-a)^2}
=
\frac{a}{1-a}.
\]
The half-life is the lag $h$ satisfying
\[
a^h=\frac{1}{2}.
\]
Taking logarithms gives
\[
h(a)=\frac{\log(1/2)}{\log a}.
\]
For $a\in(0,1)$, increasing $a$ slows the geometric decay, and therefore increases both the expected lag and the half-life.
\end{proof}

\begin{corollary}[Trace-specific retention profiles]
\label{cor:trace_retention_profiles}
The three \pertide spectral maps induce the following modal kernels:
\[
k_i^{(L)}(\ell)
=
\left(\exp(\delta\lambda_{L,i})\right)^\ell,
\]
\[
k_i^{(S)}(\ell)
=
\left(\sigma(-\kappa\lambda_{S,i})\right)^\ell,
\]
\[
k_i^{(E)}(\ell)
=
\left(1+\eta\lambda_{E,i}\right)^\ell.
\]
Their effective retention differs when these bases differ. For example, $0<a_S<a_E<a_L<1$ gives correspondingly ordered half-lives; this ordering is an additional condition, not a consequence of the three formulas.
\end{corollary}

\begin{proof}
The result follows directly by substituting the three trace-specific spectral maps into Proposition~\ref{prop:trace_mode_kernels}.
\end{proof}

\subsubsection*{Long-Term Trace: Persistent Evidence Accumulation}

\begin{proposition}[Slow modes accumulate repeated evidence]
\label{prop:long_trace_accumulation}
Consider a long-term modal direction $i$ with
\[
a_i^{(L)}
=
\exp(\delta\lambda_{L,i})
\in(0,1).
\]
Suppose repeated evidence enters this mode as
\[
u_{s,i}^{(L)}=\mu_i>0,
\qquad
s=1,\ldots,t,
\]
and let $\tilde m_{0,i}^{(L)}=0$. Then
\[
\tilde m_{t,i}^{(L)}
=
\mu_i
\frac{1-\left(a_i^{(L)}\right)^t}{1-a_i^{(L)}},
\]
and
\[
\lim_{t\to\infty}
\tilde m_{t,i}^{(L)}
=
\frac{\mu_i}{1-a_i^{(L)}}.
\]
Hence, when $a_i^{(L)}$ is close to one, repeated evidence accumulates into a large persistent state.
\end{proposition}

\begin{proof}
By Proposition~\ref{prop:trace_mode_kernels},
\[
\tilde m_{t,i}^{(L)}
=
\sum_{s=1}^{t}
\left(a_i^{(L)}\right)^{t-s}
\mu_i.
\]
Let $\ell=t-s$. Then
\[
\tilde m_{t,i}^{(L)}
=
\mu_i
\sum_{\ell=0}^{t-1}
\left(a_i^{(L)}\right)^\ell.
\]
Using the finite geometric sum,
\[
\sum_{\ell=0}^{t-1}
\left(a_i^{(L)}\right)^\ell
=
\frac{1-\left(a_i^{(L)}\right)^t}{1-a_i^{(L)}}.
\]
Since $a_i^{(L)}\in(0,1)$,
\[
\left(a_i^{(L)}\right)^t\to 0,
\]
which yields the stated limit.
\end{proof}

This formalizes the role of the long-term trace. If a user repeatedly produces evidence in a stable preference direction, such as repeated climate-policy interactions, the long-term kernel retains and accumulates this evidence rather than rapidly discarding it.

\subsubsection*{Short-Term Trace: Recent Burst Tracking}

\begin{proposition}[Small short-trace bases emphasize recent evidence]
\label{prop:short_trace_recency}
Consider a short-term modal direction $j$ with
\[
a_j^{(S)}
=
\sigma(-\kappa\lambda_{S,j})
\in(0,1).
\]
With zero initial state and the sum truncated at lag $t-1$, the modal state can be written as
\[
\tilde m_{t,j}^{(S)}
=
u_{t,j}^{(S)}
+
a_j^{(S)}u_{t-1,j}^{(S)}
+
\left(a_j^{(S)}\right)^2u_{t-2,j}^{(S)}
+
\cdots.
\]
If $a_j^{(S)}$ is small, then the newest inputs dominate and older inputs are geometrically suppressed.
\end{proposition}

\begin{proof}
The expansion follows from Proposition~\ref{prop:trace_mode_kernels}. The relative weight assigned to an input $\ell$ steps in the past is
\[
\left(a_j^{(S)}\right)^\ell.
\]
For $a_j^{(S)}\in(0,1)$, this quantity decays geometrically in $\ell$. If $a_j^{(S)}$ is small, the decay is rapid, so recent inputs dominate the modal state.
\end{proof}

This formalizes the role of the short-term trace. A recent burst of sports-highlight interactions, for example, enters the short trace with high weight at the current timestep, while older evidence is rapidly suppressed.

\subsubsection*{Episodic Trace: Decaying Event Imprint}

\begin{proposition}[Finite-window events leave controlled decaying imprints]
\label{prop:event_trace_imprint}
Consider an episodic modal direction $k$ with
\[
a_k^{(E)}
=
1+\eta\lambda_{E,k}
\in(0,1).
\]
Suppose an event is active on the finite window $t_0\leq s\leq t_1$:
\[
u_{s,k}^{(E)}
=
\begin{cases}
\beta_k, & t_0\leq s\leq t_1,\\
0, & \text{otherwise},
\end{cases}
\qquad
\beta_k>0.
\]
Consider the state component due solely to these event inputs, with zero state before the event. Let $W=t_1-t_0+1$. For $t>t_1$, its remaining post-event imprint is
\[
\tilde m_{t,k}^{(E)}
=
\beta_k
\left(a_k^{(E)}\right)^{t-t_1}
\frac{1-\left(a_k^{(E)}\right)^W}{1-a_k^{(E)}}.
\]
After the active window closes, the event leaves an imprint that decays at rate $a_k^{(E)}$.
\end{proposition}

\begin{proof}
For $t>t_1$, only the event-window inputs contribute:
\[
\tilde m_{t,k}^{(E)}
=
\sum_{s=t_0}^{t_1}
\left(a_k^{(E)}\right)^{t-s}
\beta_k.
\]
Factor out the decay from the end of the event window:
\[
\tilde m_{t,k}^{(E)}
=
\beta_k
\left(a_k^{(E)}\right)^{t-t_1}
\sum_{q=0}^{W-1}
\left(a_k^{(E)}\right)^q.
\]
Using the finite geometric sum gives
\[
\sum_{q=0}^{W-1}
\left(a_k^{(E)}\right)^q
=
\frac{1-\left(a_k^{(E)}\right)^W}{1-a_k^{(E)}}.
\]
This proves the expression.
\end{proof}

\begin{corollary}[Event time constant under exact kernel matching]
\label{cor:event_time_constant}
If the event-step size is chosen as in Lemma~\ref{lem:choose_eta}, then
\[
a_k^{(E)}
=
1+\eta\lambda_{E,k}
=
e^{-1/\tau_k}.
\]
Therefore,
\[
k_k^{(E)}(\ell)
=
e^{-\ell/\tau_k},
\]
and the post-event imprint becomes
\[
\tilde m_{t,k}^{(E)}
=
\beta_k
e^{-(t-t_1)/\tau_k}
\frac{1-e^{-W/\tau_k}}{1-e^{-1/\tau_k}}.
\]
Thus $\tau_k$ controls the decay of this mode after the finite input window closes. A shared scalar $\eta$ cannot independently match arbitrary target time constants for every fixed $\lambda_{E,k}$.
\end{corollary}

\begin{proof}
The result follows by substituting
\[
a_k^{(E)}=e^{-1/\tau_k}
\]
into Proposition~\ref{prop:event_trace_imprint}.
\end{proof}

This formalizes the role of the episodic trace. A rare debate-night interaction, tournament-final event, or command-specific burst does not vanish immediately. Its imprint has an infinite geometrically decaying tail, rather than a hard finite horizon.

\subsection{Capacity Beyond The Branch Names}
\begin{proposition}[Finite-horizon interpolation by geometric kernels]\label{prop:kernel_mixture_approximation}
Fix an integer horizon $H\geq0$ and $H+1$ distinct bases $a_1,\ldots,a_{H+1}\in(0,1)$. For any real filter values $v_0,\ldots,v_H$, there are real coefficients $c_1,\ldots,c_{H+1}$ such that
\[
K(\ell)=\sum_{m=1}^{H+1}c_m a_m^\ell=v_\ell,\qquad \ell=0,\ldots,H.
\]
\end{proposition}
\begin{proof}
The matrix $V_{\ell m}=a_m^\ell$, with rows $\ell=0,\ldots,H$, is a Vandermonde matrix. Its determinant is $\prod_{j<k}(a_k-a_j)\ne0$, so $Vc=v$ has a unique real solution.
\end{proof}
This is a finite-horizon representation result with potentially signed coefficients and one mode per constrained lag. It gives no conditioning bound, infinite-horizon approximation guarantee, or guarantee that training discovers these coefficients. A fixed-width encoder cannot invoke the result for arbitrary horizons. If complex-conjugate modes were allowed, their real combinations would additionally give damped sine and cosine kernels; the implemented shared real orthogonal parameterization does not establish such complex modes.

\begin{corollary}[Branch interpretation is not exclusive]\label{cor:inductive_not_exclusive}
The long-term, short-term, and episodic names describe intended retention behavior, rather than three disjoint classes of temporal filters. Sufficiently rich real geometric mixtures can match any specified filter on a finite set of lags under the preceding proposition's conditions.
\end{corollary}
The empirical question is therefore whether the learned branches supply complementary information. The temporal probes and trace ablations address that question without assigning each coordinate a human-interpretable preference.

\paragraph{Implementation refinements.}
The main-text equations expose the functional MTSH decomposition rather than every
parameter-sharing or residual detail. In our implementation, the three temporal traces share
an orthogonal eigenbasis ($Q_L=Q_S=Q_E$) while retaining trace-specific spectra and spectral
maps. After trace fusion, we additionally include a linear residual from the current
action-conditioned evidence; this provides a direct current-event path but introduces no
additional recurrent state or retention timescale. Command routing is also implemented
statefully: the previous task-ready state carries the routing context, and the resulting command
distribution modulates the current fused history state. These choices are implementation-level
parameterizations of the entry--persistence--fusion--consumption operators and do not introduce
additional temporal-memory components.

\subsection{Decoder Contextualization Variants}
\label{app:decoder-variants}

\paragraph{Setup.}
For each generative instance, the encoder produces (i) the task-adapted user state
$\mathbf{z}^{(t_l)}_{u}$ and (ii) the predicted query-conditioned representation
$\hat{\mathbf{z}}^{(t_{l+1})}_{q,u}$. Let $\mathbf{e}_{d_q}$ denote the embedding
of the query document to be generated from. We condition a pretrained DistilGPT2 decoder~\citep{distilgpt2-hf}
through a lightweight contextualization module that maps these representations
to a learned prefix.

\paragraph{UFI contextualization.}
User Fused Injection (UFI) directly combines the predicted query-conditioned
representation with the query-document embedding:
\begin{equation}
\label{eq:ufi_context}
\mathbf{c}^{\mathrm{user}}_{\mathrm{UFI}}
=
\mathrm{MLP}_{\mathrm{UFI}}
\left(
\left[
\hat{\mathbf{z}}^{(t_{l+1})}_{q,u};
\mathbf{e}_{d_q}
\right]
\right).
\end{equation}
UFI therefore supplies a fused user--query control without explicitly modulating
the query-conditioned representation by the user state.

\paragraph{UGI contextualization.}
User Gated Injection (UGI) first applies user-conditioned feature-wise affine
modulation to the query-conditioned representation, following the FiLM
form~\citep{perez2018film}:
\begin{equation}
\label{eq:ugi_film}
\begin{aligned}
\boldsymbol{\gamma}
&=
\mathrm{MLP}_{\gamma}
\!\left(\mathbf{z}^{(t_l)}_{u}\right),
\qquad
\boldsymbol{\beta}
=
\mathrm{MLP}_{\beta}
\!\left(\mathbf{z}^{(t_l)}_{u}\right),\\
\tilde{\mathbf{z}}^{(t_{l+1})}_{q,u}
&=
\boldsymbol{\gamma}
\odot
\hat{\mathbf{z}}^{(t_{l+1})}_{q,u}
+
\boldsymbol{\beta}.
\end{aligned}
\end{equation}
The modulated representation is then fused with the query-document embedding:
\begin{equation}
\label{eq:ugi_context}
\mathbf{c}^{\mathrm{user}}_{\mathrm{UGI}}
=
\mathrm{MLP}_{\mathrm{UGI}}
\left(
\left[
\tilde{\mathbf{z}}^{(t_{l+1})}_{q,u};
\mathbf{e}_{d_q}
\right]
\right).
\end{equation}
Thus, unlike UFI, UGI allows the user state to emphasize or suppress individual
components of the query-conditioned representation before decoder conditioning.

\paragraph{Prefix construction.}
For either variant, let $\mathbf{c}^{\mathrm{user}}$ denote the resulting
user-conditioned control. We further combine it with query-document evidence:
\begin{equation}
\label{eq:decoder_doc_context}
\mathbf{c}^{\mathrm{doc}}
=
\mathrm{MLP}_{\mathrm{doc}}
\left(
\left[
\mathbf{c}^{\mathrm{user}};
\mathbf{e}_{d_q}
\right]
\right),
\qquad
\mathbf{c}
=
\mathrm{LN}
\left(
\mathbf{c}^{\mathrm{user}}
+
\mathbf{c}^{\mathrm{doc}}
\right).
\end{equation}
The resulting vector is projected into a prefix of $K=8$ decoder embeddings:
\begin{equation}
\label{eq:decoder_injection}
\mathbf{P}
=
\operatorname{reshape}_{K\times d_D}
\left(
\mathbf{W}_{p}\,
\mathrm{Dropout}(\mathbf{c})
+
\mathbf{b}_{p}
\right),
\end{equation}
where $d_D$ is the DistilGPT2 embedding dimension. If
$x_{1:S}$ denotes the tokenized query document and $y_{1:T}$ the target
headline, teacher-forced training uses
\begin{equation}
\label{eq:decoder_sequence}
\mathbf{H}^{(0)}
=
\left[
\mathbf{P};
\mathrm{Emb}(x_1);
\ldots;
\mathrm{Emb}(x_S);
\mathrm{Emb}(y_1);
\ldots;
\mathrm{Emb}(y_{T-1})
\right],
\end{equation}
with the standard autoregressive objective
$p(y_t\mid y_{<t},x_{1:S},\mathbf{P})$.

\paragraph{Decoder tuning.}
DistilGPT2 is initialized from pretrained weights. Its parameters are initially
frozen; during joint and generation-specific training, the contextualization
modules, the final two transformer blocks, and the final decoder layer
normalization are optimized, while all remaining DistilGPT2 parameters stay
fixed. The generation-specific stage freezes the \pertide{} encoder, so
personalization continues to enter the decoder through the learned preference-history
state and its UFI/UGI contextualization rather than through user-specific decoder
parameters.
\section{Training Details}\label{app:training-details}
\begin{table}[t]
\centering
\caption{\textbf{Notation and symbol table.}}
\label{tab:notation}
\setlength{\tabcolsep}{2.5pt}
\renewcommand{\arraystretch}{1.04}
\small
\scalebox{0.95}{
\begin{tabular}{p{0.15\linewidth}p{0.33\linewidth}p{0.15\linewidth}p{0.33\linewidth}}
\toprule
\textbf{Symbol} & \textbf{Meaning} & \textbf{Symbol} & \textbf{Meaning} \\
\midrule

\multicolumn{4}{l}{\textbf{UIG and Action-On-Item Schema}} \\
\midrule
$G_u=\langle\mathcal{N}_u,\mathcal{E}_u\rangle$ & User-specific interaction graph. &
$\mathcal{N}_u,\mathcal{E}_u$ & User-specific nodes and edges. \\
$u^{(t_0)}$ & Initial user anchor node. &
$d^{(t_i)}$ & Item/content node at timestep $t_i$. \\
$r_u^{(t_i)}$ & Response node at timestep $t_i$. &
$\alpha^{(t_i)}$ & Domain-native action label. \\
$\mathcal{A}$ & Shared action-role vocabulary. &
$a^{(t_i)}$ & Schema-level action role. \\
$\rho_D(\cdot)$ & Domain-specific role mapping. &
$\psi_D(\cdot)$ & Domain-specific target-locus embedding map. \\
$S_D(e)$ & Schema-normalized native event. &
$\tau_u$ & Time-ordered user trajectory. \\
$\mathcal{T}_{\mathrm{train}},\mathcal{T}_{\mathrm{test}}$ & Train and test trajectory pools. &
$d_q$ & Query item/content node. \\
$\mathbf{g}_q$ & Query or task seed embedding. &
$\mathcal{Q}$ & Set of task/command types. \\

\midrule
\multicolumn{4}{l}{\textbf{Interaction Blocks and Target-Locus Embeddings}} \\
\midrule
$b_u^{(t_i)}=\langle a^{(t_i)},tl^{(t_i)}\rangle$ & Action role applied to target locus. &
$tl^{(t_i)}$ & Target locus of current action. \\
$\mathbf{e}_{tl}^{(t_i)}$ & Target-locus embedding. &
$\boldsymbol{\mu}^{(t_i)}$ & Action-conditioned gating vector. \\
$\mathbf{b}_u^{(t_i)}$ & Action-gated event cell. &
$\phi(\cdot)$ & Optional generic block embedder. \\
$\mathbf{x}_j$ & Generic baseline block embedding. &
$d$ & Hidden and embedding dimensionality. \\

\midrule
\multicolumn{4}{l}{\textbf{MTSH and PerTIDE State Variables}} \\
\midrule
$\mathbf{m}_u^{(t_i,L)}$ & Long-term preference trace. &
$\mathbf{m}_u^{(t_i,S)}$ & Short-term preference trace. \\
$\mathbf{m}_u^{(t_i,E)}$ & Episodic preference trace. &
$A_r$ & Trace transition matrix. \\
$Q_r,\Lambda_r$ & Shared eigenbasis ($Q_r=Q$) and trace spectrum. &
$B$ & Shared trace input operator. \\
$\Phi(\cdot)$ & Multi-trace fusion map. &
$\Psi(\cdot)$ & Action-conditioned evidence injection map. \\
$\mathbf{h}_{b_u}^{(t_i)}$ & Fused preference-history state. &
$\Pi(\cdot)$ & Command/task inference map. \\
$G(\cdot)$ & Command-conditioning gate. &
$\boldsymbol{\pi}_u^{(t_i)}$ & Predicted command/task distribution. \\
$\mathbf{z}_{b_u}^{(t_i)}$ & Task-ready user state. &
 & \\

\midrule
\multicolumn{4}{l}{\textbf{Task Heads and Encoder Objectives}} \\
\midrule
$T$ & Downstream task or command type. &
$f_T(\cdot)$ & Task-specific prediction head. \\
$\hat{\mathbf{z}}_{q,u}^{(t_l+1)}$ & Query-conditioned next-step state. &
$\hat{y}_T$ & Predicted task output. \\
$\hat{p}(\cdot)$ & Task-head predictive distribution. &
$y_T^\ast$ & Ground-truth task target. \\
$\mathcal{B}$ & Vocabulary of observed block nodes. &
$\mathcal{L}_{\mathrm{pos}}$ & Temporal-position alignment loss. \\
$\mathcal{L}_{\mathrm{next}}$ & Next-block prediction loss. &
$\mathcal{L}_{\mathrm{task}}$ & Command/task classification loss. \\
$\mathcal{L}_{\mathrm{enc}}$ & Combined encoder training objective. &
$\mathcal{L}_{\mathrm{obj}}$ & Downstream task objective. \\

\midrule
\multicolumn{4}{l}{\textbf{Decoder Contextualization: UFI and UGI}} \\
\midrule
$\mathbf{e}_{d_q}$ & Query-document/item embedding. &
$\mathbf{c}$ & Final decoder control vector. \\

$\mathbf{c}^{\mathrm{user}}_{\mathrm{UFI}}$ & User-fused control before document grounding. &
$\boldsymbol{\gamma}$ & FiLM scale vector. \\

$\boldsymbol{\beta}$ & FiLM shift vector. &
$\odot$ & Elementwise product. \\

$\tilde{\mathbf{z}}_{q,u}^{(t_l+1)}$ & User-gated query representation. &
$\mathbf{c}^{\mathrm{user}}_{\mathrm{UGI}}$ & User-gated control before document grounding. \\

$\mathbf{c}^{\mathrm{doc}}$ & Document-grounding control vector. &
$K$ & Learned prefix length ($K=8$). \\

$\mathbf{P}$ & Learned decoder prefix embeddings. &
$W_p,\mathbf{b}_p$ & Control-to-prefix projection parameters. \\

$\mathbf{H}^{(0)}$ & Decoder input states after prefix and source-token concatenation. &
$\mathbf{H}^{(L)}$ & Final decoder hidden states. \\

$\theta_D$ & Pretrained DistilGPT2 parameters, partially fine-tuned. &
 & \\

\bottomrule
\end{tabular}
}
\end{table}

\begin{table}[p]
\centering\small
\setlength{\tabcolsep}{3pt}
\renewcommand{\arraystretch}{1.1}
\caption{\textbf{Event input, traces, and fusion.} Architecture dimensions for $d=768$. MLPs have two layers with hidden width 768 unless stated otherwise.}
\label{tab:pertide_encoder_hparam_dims}
\scalebox{0.95}{
\begin{tabular}{p{0.24\linewidth}p{0.47\linewidth}p{0.25\linewidth}}
\toprule
\textbf{Symbol} & \textbf{Meaning} & \textbf{Dimensions}\\
\midrule

\multicolumn{3}{c}{\fbox{\strut \textbf{Action-Gated Event Cells}}} \\

$d$ &
Model / latent width (seed embeddings, memory traces, user state, task embeddings) &
$768$ \\

$d_a$ &
Action-role input width &
$4$ \\

$\mathbf{e}^{(t_i)}_{tl}$ &
Target-locus embedding at timestep $t_i$; induces item-side abstraction &
$\mathbb{R}^{768}$ \\

$\mathbf{a}^{(t_i)}$ &
One-hot action-role input to the gate MLP &
$\mathbb{R}^{4}$ \\

$W^{(1)}_{\text{gate}},\,\mathbf{b}^{(1)}_{\text{gate}}$ &
Gate MLP layer-1 parameters ($4\!\rightarrow\!768$) &
$\mathbb{R}^{768\times 4},\;\mathbb{R}^{768}$ \\

$W^{(2)}_{\text{gate}},\,\mathbf{b}^{(2)}_{\text{gate}}$ &
Gate MLP layer-2 parameters ($768\!\rightarrow\!768$) &
$\mathbb{R}^{768\times 768},\;\mathbb{R}^{768}$ \\

$\boldsymbol{\mu}^{(t_i)}$ &
Action-specific gate,
$\boldsymbol{\mu}^{(t_i)}
=\sigma(\mathrm{MLP}_{\text{gate}}(\mathbf{a}^{(t_i)}))$ &
$\mathbb{R}^{768}$ \\

$\mathbf{b}^{(t_i)}_{u}$ &
Action-gated event cell,
$\mathbf{b}^{(t_i)}_{u}
=\boldsymbol{\mu}^{(t_i)}\odot\mathbf{e}^{(t_i)}_{tl}$ &
$\mathbb{R}^{768}$ \\

\midrule
\multicolumn{3}{c}{\fbox{\strut \textbf{Tri-Trace SSM (Long / Short / Episodic)}}} \\

$\mathbf{m}^{(t_i,r)}_{u}$ &
Memory trace for timescale $r\in\{\mathrm{L,S,E}\}$ &
$\mathbb{R}^{768}$ \\

$Q$ &
Learned orthogonal eigenbasis shared across traces
($Q_{\mathrm{L}}=Q_{\mathrm{S}}=Q_{\mathrm{E}}=Q$) &
$\mathbb{R}^{768\times 768}$ \\

$\Lambda_{r}$ &
Trace-specific diagonal spectrum for $r\in\{\mathrm{L,S,E}\}$
(stored as diagonal entries) &
$\mathbb{R}^{768}$ \\

$B$ &
Input operator shared across traces &
$\mathbb{R}^{768\times 768}$ \\

$\delta$ &
Long-term spectral-map scalar in $\exp(\delta\Lambda_{\mathrm{L}})$ &
scalar \\

$\kappa$ &
Short-term spectral-map scalar in $\sigma(-\kappa\Lambda_{\mathrm{S}})$ &
scalar \\

$\eta$ &
Episodic spectral-map scalar in $(I+\eta\Lambda_{\mathrm{E}})$ &
scalar \\

\midrule
\multicolumn{3}{c}{\fbox{\strut \textbf{Adaptive Memory Mixer}}} \\

$\mathbf{h}^{(t_i)}_{b_u}$ &
Fused preference-history state; implementation additionally includes
a direct current-event residual &
$\mathbb{R}^{768}$ \\

$W^{(1)}_{\Phi},\,\mathbf{b}^{(1)}_{\Phi}$ &
Mixer MLP layer-1 parameters
($[m^{(L)};m^{(S)};m^{(E)}]
\in\mathbb{R}^{2304}\!\rightarrow\!768$) &
$\mathbb{R}^{768\times 2304},\;\mathbb{R}^{768}$ \\

$W^{(2)}_{\Phi},\,\mathbf{b}^{(2)}_{\Phi}$ &
Mixer MLP layer-2 parameters ($768\!\rightarrow\!768$) &
$\mathbb{R}^{768\times 768},\;\mathbb{R}^{768}$ \\

$W_{\psi},\,\mathbf{b}_{\psi}$ &
Linear current-event residual projection applied after trace fusion &
$\mathbb{R}^{768\times768},\;\mathbb{R}^{768}$ \\

\bottomrule
\end{tabular}}
\end{table}

\begin{table}[p]
\centering\small
\setlength{\tabcolsep}{3pt}
\renewcommand{\arraystretch}{1.1}
\caption{\textbf{Command readout, supervision, and decoder settings.} Architecture dimensions for $d=768$. MLPs have two layers with hidden width 768 unless stated otherwise.}
\label{tab:pertide_dims_2}
\scalebox{0.95}{
\begin{tabular}{p{0.24\linewidth}p{0.47\linewidth}p{0.25\linewidth}}
\toprule
\textbf{Symbol} & \textbf{Meaning} & \textbf{Dimensions}\\
\midrule
\multicolumn{3}{c}{\fbox{\strut
\textbf{Command-Conditioned Readout (Realization of $\Pi$ and $G$)}}} \\

$\mathbf{g}_{q}$ &
Task / command seed embedding &
$\mathbb{R}^{768}$ \\

$\boldsymbol{\pi}^{(t_i)}_{u}$ &
Command/task distribution
$\boldsymbol{\pi}^{(t_i)}_{u}
=\mathrm{SoftMax}
(W_{\text{task}}[\mathbf{z}^{(t_{i-1})}_{b_u};\mathbf{g}_q])$;
routing uses the preceding task-ready state (Section~\ref{sec:PerTIDE}) &
$\mathbb{R}^{|\mathcal{Q}|}$ \\

$W_{\text{task}}$ &
Task-distribution projection &
$\mathbb{R}^{|\mathcal{Q}|\times1536}$ \\

$W_{\pi}$ &
Command-to-state gating projection in $G(\cdot)$ &
$\mathbb{R}^{768\times|\mathcal{Q}|}$ \\

$\mathbf{b}_{G}$ &
Command-gate bias &
$\mathbb{R}^{768}$ \\

$\mathbf{z}^{(t_i)}_{b_u}$ &
Command-conditioned task-ready preference state after gating &
$\mathbb{R}^{768}$ \\

\midrule
\multicolumn{3}{c}{\fbox{\strut \textbf{Encoder Supervision Heads}}} \\

$l$ &
Trajectory slice length used for encoder supervision &
scalar \\

$W_{\text{pos}}$ &
Position classifier over $\{1,\dots,l-1\}$ &
$\mathbb{R}^{(l-1)\times768}$ \\

$W_{\text{lin}},\,\mathbf{b}_{\text{lin}}$ &
Linear head mapping
$\mathbf{z}^{(t_{l-1})}_{b_u}
\rightarrow
\hat{\mathbf{z}}^{(t_l)}_{b_u}$ &
$\mathbb{R}^{768\times768},\;\mathbb{R}^{768}$ \\

$|\mathcal{B}|$ &
b-node vocabulary size for next-step prediction &
$63{,}000$ \\

$W_{\text{next}},\,\mathbf{b}_{\text{next}}$ &
Query b-node predictor logits over $\mathcal{B}$ &
$\mathbb{R}^{|\mathcal{B}|\times768},\;
\mathbb{R}^{|\mathcal{B}|}$ \\

$\lambda_{\text{pos}},
\lambda_{\text{next}},
\lambda_{\text{task}}$ &
Encoder-loss weights in
$\mathcal{L}_{\text{enc}}
=
\lambda_{\text{pos}}\mathcal{L}_{\text{pos}}
+
\lambda_{\text{next}}\mathcal{L}_{\text{next}}
+
\lambda_{\text{task}}\mathcal{L}_{\text{task}}$ &
$0.2,\;0.4,\;0.4$ \\

\midrule
\multicolumn{3}{c}{\fbox{\strut \textbf{Shared Decoder-Side Quantities}}} \\

$d_D$ &
Pretrained DistilGPT2 decoder embedding width &
$768$ \\

$d_c$ &
Decoder control-vector width &
$768$ \\

$K$ &
Learned decoder-prefix length &
$8$ \\

$n_{\text{FT}}$ &
Number of final DistilGPT2 transformer blocks fine-tuned
during generation training; final decoder LayerNorm is also optimized &
$2$ \\

$\hat{\mathbf{z}}^{(t_{l+1})}_{q,u}$ &
Predicted query-conditioned representation from the encoder &
$\mathbb{R}^{768}$ \\

$\mathbf{z}^{(t_l)}_{u}$ &
Task-adapted user state from the encoder &
$\mathbb{R}^{768}$ \\

$\mathbf{e}_{d_q}$ &
Query-document embedding &
$\mathbb{R}^{768}$ \\

\bottomrule
\end{tabular}}
\end{table}

\begin{table}[p]
\centering\small
\setlength{\tabcolsep}{3pt}
\renewcommand{\arraystretch}{1.1}
\caption{\textbf{User fused and user gated injection.} Architecture dimensions for $d=768$. MLPs have two layers with hidden width 768 unless stated otherwise.}
\label{tab:pertide_dims_3}
\scalebox{0.95}{
\begin{tabular}{p{0.24\linewidth}p{0.47\linewidth}p{0.25\linewidth}}
\toprule
\textbf{Symbol} & \textbf{Meaning} & \textbf{Dimensions}\\
\midrule
\multicolumn{3}{c}{\fbox{\strut \textbf{Decoder-UFI (User Fused Injection)}}} \\

$W^{(1)}_{\text{UFI}},\,\mathbf{b}^{(1)}_{\text{UFI}}$ &
UFI MLP layer-1 parameters
($[\hat{\mathbf{z}}_{q,u};\mathbf{e}_{d_q}]
\in\mathbb{R}^{1536}\!\rightarrow\!768$) &
$\mathbb{R}^{768\times1536},\;\mathbb{R}^{768}$ \\

$W^{(2)}_{\text{UFI}},\,\mathbf{b}^{(2)}_{\text{UFI}}$ &
UFI MLP layer-2 parameters ($768\!\rightarrow\!768$) &
$\mathbb{R}^{768\times768},\;\mathbb{R}^{768}$ \\

$\mathbf{c}^{\mathrm{user}}_{\mathrm{UFI}}$ &
User-fused control,
$\mathbf{c}^{\mathrm{user}}_{\mathrm{UFI}}
=
\mathrm{MLP}_{\mathrm{UFI}}
([\hat{\mathbf{z}}_{q,u};\mathbf{e}_{d_q}])$ &
$\mathbb{R}^{768}$ \\

\midrule
\multicolumn{3}{c}{\fbox{\strut \textbf{Decoder-UGI (User Gated Injection)}}} \\

$W^{(1)}_{\gamma},\,\mathbf{b}^{(1)}_{\gamma}$ &
FiLM-scale MLP layer-1 parameters
($\mathbf{z}^{(t_l)}_{u}\in\mathbb{R}^{768}\!\rightarrow\!768$) &
$\mathbb{R}^{768\times768},\;\mathbb{R}^{768}$ \\

$W^{(2)}_{\gamma},\,\mathbf{b}^{(2)}_{\gamma}$ &
FiLM-scale MLP layer-2 parameters ($768\!\rightarrow\!768$) &
$\mathbb{R}^{768\times768},\;\mathbb{R}^{768}$ \\

$W^{(1)}_{\beta},\,\mathbf{b}^{(1)}_{\beta}$ &
FiLM-shift MLP layer-1 parameters
($\mathbf{z}^{(t_l)}_{u}\in\mathbb{R}^{768}\!\rightarrow\!768$) &
$\mathbb{R}^{768\times768},\;\mathbb{R}^{768}$ \\

$W^{(2)}_{\beta},\,\mathbf{b}^{(2)}_{\beta}$ &
FiLM-shift MLP layer-2 parameters ($768\!\rightarrow\!768$) &
$\mathbb{R}^{768\times768},\;\mathbb{R}^{768}$ \\

$\boldsymbol{\gamma},\boldsymbol{\beta}$ &
FiLM scale/shift vectors from the user state &
$\mathbb{R}^{768},\;\mathbb{R}^{768}$ \\

$\tilde{\mathbf{z}}^{(t_{l+1})}_{q,u}$ &
Gated query representation,
$\tilde{\mathbf{z}}_{q,u}
=
\boldsymbol{\gamma}\odot\hat{\mathbf{z}}_{q,u}
+
\boldsymbol{\beta}$ &
$\mathbb{R}^{768}$ \\

$W^{(1)}_{\text{UGI}},\,\mathbf{b}^{(1)}_{\text{UGI}}$ &
UGI MLP layer-1 parameters
($[\tilde{\mathbf{z}}_{q,u};\mathbf{e}_{d_q}]
\in\mathbb{R}^{1536}\!\rightarrow\!768$) &
$\mathbb{R}^{768\times1536},\;\mathbb{R}^{768}$ \\

$W^{(2)}_{\text{UGI}},\,\mathbf{b}^{(2)}_{\text{UGI}}$ &
UGI MLP layer-2 parameters ($768\!\rightarrow\!768$) &
$\mathbb{R}^{768\times768},\;\mathbb{R}^{768}$ \\

$\mathbf{c}^{\mathrm{user}}_{\mathrm{UGI}}$ &
User-gated control,
$\mathbf{c}^{\mathrm{user}}_{\mathrm{UGI}}
=
\mathrm{MLP}_{\mathrm{UGI}}
([\tilde{\mathbf{z}}_{q,u};\mathbf{e}_{d_q}])$ &
$\mathbb{R}^{768}$ \\

\bottomrule
\end{tabular}}
\end{table}

\begin{table}[p]
\centering\small
\setlength{\tabcolsep}{3pt}
\renewcommand{\arraystretch}{1.1}
\caption{\textbf{Document grounding, prefix, and generation loss.} Architecture dimensions for $d=768$. MLPs have two layers with hidden width 768 unless stated otherwise.}
\label{tab:pertide_dims_4}
\scalebox{0.95}{
\begin{tabular}{p{0.24\linewidth}p{0.47\linewidth}p{0.25\linewidth}}
\toprule
\textbf{Symbol} & \textbf{Meaning} & \textbf{Dimensions}\\
\midrule
\multicolumn{3}{c}{\fbox{\strut \textbf{Decoder Document Grounding}}} \\

$W^{(1)}_{\text{doc}},\,\mathbf{b}^{(1)}_{\text{doc}}$ &
Document-grounding MLP layer-1 parameters
($[\mathbf{c}^{\mathrm{user}};\mathbf{e}_{d_q}]
\in\mathbb{R}^{1536}\!\rightarrow\!768$) &
$\mathbb{R}^{768\times1536},\;\mathbb{R}^{768}$ \\

$W^{(2)}_{\text{doc}},\,\mathbf{b}^{(2)}_{\text{doc}}$ &
Document-grounding MLP layer-2 parameters ($768\!\rightarrow\!768$) &
$\mathbb{R}^{768\times768},\;\mathbb{R}^{768}$ \\

$\mathbf{c}^{\mathrm{doc}}$ &
Document-grounding control,
$\mathbf{c}^{\mathrm{doc}}
=
\mathrm{MLP}_{\mathrm{doc}}
([\mathbf{c}^{\mathrm{user}};\mathbf{e}_{d_q}])$ &
$\mathbb{R}^{768}$ \\

$\mathbf{c}$ &
Final normalized decoder control,
$\mathbf{c}
=
\mathrm{LN}
(\mathbf{c}^{\mathrm{user}}+\mathbf{c}^{\mathrm{doc}})$ &
$\mathbb{R}^{768}$ \\

\midrule
\multicolumn{3}{c}{\fbox{\strut \textbf{Learned Prefix Injection}}} \\

$W_p,\,\mathbf{b}_p$ &
Control-to-prefix projection,
$\mathbf{P}
=
\operatorname{reshape}_{K\times d_D}
(W_p\mathbf{c}+\mathbf{b}_p)$ &
$\mathbb{R}^{6144\times768},\;
\mathbb{R}^{6144}$ \\

$\mathbf{P}$ &
Learned prefix embeddings prepended to the decoder input &
$\mathbb{R}^{8\times768}$ \\

$\mathbf{H}^{(0)}$ &
Decoder input states formed from the learned prefix, source tokens,
and teacher-forced target-prefix tokens &
$\mathbb{R}^{(K+S+T-1)\times768}$ \\

\midrule
\multicolumn{3}{c}{\fbox{\strut \textbf{Generation Supervision}}} \\

$\lambda_{\mathrm{sem}}$ &
Semantic-alignment weight in
$\mathcal{L}_{\mathrm{gen}}
=
\mathcal{L}_{\mathrm{CE}}
+
\lambda_{\mathrm{sem}}\mathcal{L}_{\mathrm{sem}}$ &
$0.15$ \\

\bottomrule
\end{tabular}}
\end{table}

\subsection{Compute Resources}\label{app:resources}
All dataset ingestion and feature lookup were performed on CPU, including CSV loading and construction of the
canonical ID maps (\texttt{id2text}/\texttt{id2phrases}) and vocabulary indices (\texttt{actions2idx}, \texttt{node2idx}, \texttt{user2idx}).
Dense node embeddings for dialogue/document/response/summary nodes were loaded from precomputed pickled files and resolved
on-the-fly during training via canonical IDs, with missing embeddings defaulted to a zero vector of dimension $d_{\text{model}}{=}768$.

Experiments were run on Lightning.ai using A100 and RTX 6000-family GPUs. Pretrained LLMs used automatic device placement (\texttt{device\_map="auto"}). Exact device counts and the experiment-to-GPU mapping were not preserved in the run records. Appendix~\ref{app:detailed_results_efficiency} reports estimated resource costs.
\subsection{Training}
\label{app:training}

We train the multi-task model using a single-row update scheme: each training example is converted
into a batch of size 1, and one optimizer step is applied per valid row. We train for $E=45$ epochs
over $N_{\mathrm{train}}=78{,}274$ training rows, with rows that fail parsing or target construction
skipped. Per epoch, we deterministically shuffle the active training table with
\texttt{random\_state=SEED+epoch}. We use training seeds 42, 32, and 22, and evaluation seeds 1111, 2222, and 3333. We define a processed prefix length
$T_{\mathrm{proc}}=T-2$ so that the final two positions serve as query and target.

Optimization uses AdamW with learning rate $3\times10^{-4}$ and gradient clipping at norm $1.0$.
The shared encoder is supervised with
\[
\mathcal L_{\mathrm{enc}}
=
0.2\mathcal L_{\mathrm{pos}}
+
0.4\mathcal L_{\mathrm{next}}
+
0.4\mathcal L_{\mathrm{task}},
\]
The weights are fixed globally. The main encoder uses pooled MovieLens, MIND, PENS, and OpenAI training trajectories; the source-only transfer encoder instead uses PENS alone. Training adds the corresponding downstream task objective. Predictive training uses the
task-specific ranking or classification loss.

For generation, we condition a pretrained DistilGPT2 decoder through an eight-token learned
UFI/UGI prefix. The generation objective combines autoregressive target-token cross-entropy with
a semantic alignment term,
\[
\mathcal L_{\mathrm{gen}}
=
\mathcal L_{\mathrm{CE}}
+
0.15\,\mathcal L_{\mathrm{sem}}.
\]
The contextualization modules, final two DistilGPT2 transformer blocks, and final decoder layer
normalization are optimized for generation, while the remaining decoder parameters stay fixed.
The generation-specific stage freezes the \pertide{} encoder. Decoder contextualization and prefix
construction are detailed in Appendix~\ref{app:decoder-variants}.

\paragraph{Crash-safe checkpointing and validation.}
We save row checkpoints every 15{,}000 processed rows and an additional end-of-epoch checkpoint.
Row checkpoints store the current epoch and row index, and resumption continues from the next row.
After each epoch, we evaluate on a deterministic validation slice of up to 500 examples per task, reporting overall and per-task
mean losses (total loss and its encoder/user/item/generation components).

\subsection{Statistical Significance Testing}

We assess evaluation uncertainty using paired bootstrap resampling over the held-out
evaluation pool. The purpose of this analysis is to test whether the observed improvement
of \pertide{} over a comparator is stable to resampling of the same evaluation instances,
rather than being driven by a particular subset of users, histories, rollout positions, or
candidates. For every dataset--task--metric tuple, the paired systems are evaluated on
identical held-out examples and identical history prefixes. For predictive recommendation,
they additionally use the same candidate sets containing one ground-truth target and 699
sampled negatives.

Let $\mathcal{T}_{\mathrm{eval}}=\{x_n\}_{n=1}^{N}$ denote the held-out evaluation
pool for a task. For a model $A$ and metric $M$, let $M(A;S)$ denote the metric
computed on resampled evaluation set $S$. For each paired comparison between \pertide{}
and a comparator $A$, including the strongest task-native baseline, PAH/MDH/LDH
controls, or an ablated \pertide{} variant, we draw
$B=10{,}000$ paired bootstrap samples
$S_1,\ldots,S_B$ from $\mathcal{T}_{\mathrm{eval}}$.
The same sampled indices are used for both systems within every bootstrap replicate.

For bootstrap replicate $b$, we compute the paired metric difference
\[
\Delta_b =
M(\mathrm{\pertide{}};S_b)-M(A;S_b).
\]

The reported point estimate for each model is the mean over the bootstrap evaluations,
\[
\bar{M}_B(A)
=
\frac{1}{B}\sum_{b=1}^{B} M(A;S_b),
\]
and the reported paired improvement is
\[
\bar{\Delta}_B
=
\frac{1}{B}\sum_{b=1}^{B}\Delta_b.
\]

Statistical support for a positive improvement is determined directly from the empirical
paired-bootstrap distribution $\{\Delta_b\}_{b=1}^{B}$. We construct a one-sided
95\% bootstrap confidence interval for the paired difference and regard an improvement
as statistically supported when its lower confidence bound is greater than zero.
Because the resampling is paired, all variation within a replicate arises from changing
the evaluated instances while holding the two compared systems and their evaluation
conditions fixed.

We apply this procedure to predictive metrics including MRR, nDCG@5, nDCG@10,
and HR@10, and to the generation and personalization metrics reported for PENS and
OpenAI-Reddit, including PSE-JSD, PSE-SU4, PSE-METEOR, ROUGE, BLEU-4,
BERTScore, and human-judgment scores where applicable.

This analysis measures test-set sampling variability for fixed trained systems. It does
not estimate variability arising from independent retraining, alternative train/test splits,
or different negative-sampling pools. Accordingly, statistical-support statements in the
paper refer specifically to paired stability of the reported gains under the fixed evaluation
protocol.
\section{Baseline Details}
\label{app:baselines with details}

\subsection{Encoder Baselines}
\label{app:encoder_baselines}

\paragraph{Notation and grouping principle.}
Let $\tau^{u}_{1:i-1}=(b_u^{(t_1)},\ldots,b_u^{(t_{i-1})})$ denote the ordered prefix of a user's I-tier trajectory before timestep $t_i$. Each interaction block is embedded as
$\mathbf{x}_j=\phi\!\big(b_u^{(t_j)}\big)\in\mathbb{R}^d$, and the history encoder exposes a task-facing state
$\mathbf{z}_u^{(t_i)}\in\mathbb{R}^d$ after processing this prefix.
Following Section~\ref{subsec:history_states_to_flow}, we organize baselines into Profile Aggregation (PAH), Markovian Drift (MDH), and Long-Horizon Dynamics (LDH) families according to the form of the state exposed downstream. These are operational state-design families, not mutually exclusive assumptions about user dynamics. Architectures can admit more than one mathematical description; assignments below follow the dominant mechanism by which each model constructs its task-facing history state.

\paragraph{Profile aggregation family (PAH).}
PAH groups encoders that expose an order-invariant profile obtained by aggregating past interaction representations. Its generic form is a permutation-invariant operator $f_{\mathrm{Agg}}(\cdot)$:
\begin{equation}
\label{eq:PAH_core_app}
\mathbf{z}_u^{(t_i)}
=
f_{\mathrm{Agg}}\!\left(\{\mathbf{x}_j\}_{j=1}^{i-1}\right).
\end{equation}
We use the following PAH-family baselines.

\textbf{MeanPool-B.}
\begin{equation}
\label{eq:meanpool_app}
\mathbf{z}_u^{(t_i)}
=
\frac{1}{i-1}\sum_{j=1}^{i-1}\mathbf{x}_j.
\end{equation}

\textbf{MaxPool-B.}
\begin{equation}
\label{eq:maxpool_app}
\big[\mathbf{z}_u^{(t_i)}\big]_k
=
\max_{1\le j\le i-1}\big[\mathbf{x}_j\big]_k,
\qquad
\forall k\in\{1,\dots,d\}.
\end{equation}

\textbf{NAML (PAH-aligned).}
NAML-style encoders first produce per-item representations $\mathbf{h}_j$ from multi-view features and then aggregate them through attention pooling:
\begin{equation}
\label{eq:naml_attn_app}
\mathbf{z}_u^{(t_i)}
=
\sum_{j=1}^{i-1}\alpha_j\,\mathbf{h}_j,
\qquad
\alpha_j
=
\frac{\exp(s(\mathbf{h}_j))}
{\sum_{q=1}^{i-1}\exp(s(\mathbf{h}_q))}.
\end{equation}
Under this history-level formulation, the task-facing state is an order-invariant weighted aggregation of per-item representations, placing NAML in the PAH family.

\textbf{PLM-NR (PAH-aligned).}
PLM-NR-style encoders replace the item representation module with a pretrained language model to obtain $\mathbf{h}_j$, while retaining an order-invariant history aggregation operator, typically attention pooling as in Eq.~\eqref{eq:naml_attn_app}. We therefore place PLM-NR in PAH according to its exposed history state, while $\phi(\cdot)$ provides the stronger content representation.

\paragraph{Recurrent state and Markovian drift family (MDH).}
MDH groups encoders that expose a single sequentially updated or recency-bounded state rather than an order-invariant profile or an explicitly long-horizon task-facing representation. For recurrent members, the generic first-order update is
\begin{equation}
\label{eq:MDH_core_app}
\mathbf{z}_u^{(t_i)}
=
f_{\mathrm{Rec}}\!\left(
\mathbf{z}_u^{(t_{i-1})},
\mathbf{x}_{i-1};
\rho
\right)
+
\boldsymbol{\xi}^{(t_i)},
\end{equation}
where $\rho$ parameterizes the effective update or recency behavior and $\boldsymbol{\xi}^{(t_i)}$ denotes an optional perturbation term. The family label refers to the form of the exposed state, not to a claim that the underlying user process itself is Markovian.

\textbf{GRU-B (MDH instantiation).}
We instantiate the recurrent form with a gated recurrent unit over the ordered interaction-block embeddings $\{\mathbf{x}_i\}_{i=1}^{L}$. Given $\mathbf{h}_{i-1}\in\mathbb{R}^d$ and $\mathbf{x}_i\in\mathbb{R}^d$,
\begin{equation}
\label{eq:gru_gates_app}
\mathbf{r}_i
=
\sigma(\mathbf{W}_r\mathbf{x}_i+\mathbf{U}_r\mathbf{h}_{i-1}),
\qquad
\mathbf{u}_i
=
\sigma(\mathbf{W}_u\mathbf{x}_i+\mathbf{U}_u\mathbf{h}_{i-1}),
\end{equation}
\begin{equation}
\label{eq:gru_candidate_app}
\tilde{\mathbf{h}}_i
=
\tanh\!\left(
\mathbf{W}_h\mathbf{x}_i
+
\mathbf{U}_h(\mathbf{r}_i\odot\mathbf{h}_{i-1})
\right),
\end{equation}
\begin{equation}
\label{eq:gru_state_app}
\mathbf{h}_i
=
(1-\mathbf{u}_i)\odot\mathbf{h}_{i-1}
+
\mathbf{u}_i\odot\tilde{\mathbf{h}}_i,
\qquad
\mathbf{z}_u^{(t_{L+1})}\equiv\mathbf{h}_L.
\end{equation}
Here, $\sigma(\cdot)$ is the logistic sigmoid, $\odot$ denotes element-wise multiplication, and $\{\mathbf{W}_\cdot,\mathbf{U}_\cdot\}$ are learned parameters. GRU-B therefore realizes the recurrent MDH form by compressing the processed prefix into one evolving hidden state.

\textbf{EBNR (MDH-aligned).}
EBNR encodes click-order information through a recurrent user encoder and ranks candidates by matching the resulting user embedding to candidate-news representations. In its recurrent user-modeling form, each clicked news item $d_i$ is represented by $\mathbf{n}_i\in\mathbb{R}^d$ and the history state evolves as
\begin{equation}
\label{eq:ebnr_user_gru_app}
\mathbf{h}_i
=
\mathrm{GRU}(\mathbf{h}_{i-1},\mathbf{n}_i),
\qquad
\mathbf{u}\equiv\mathbf{h}_L.
\end{equation}
For candidate news item $d$ with representation $\mathbf{n}(d)$,
\begin{equation}
\label{eq:ebnr_score_app}
s(u,d)
=
\mathbf{u}^{\top}\mathbf{n}(d).
\end{equation}
A probabilistic implementation may use $\hat{y}(u,d)=\sigma(s(u,d))$ with observed clicks and sampled or impression-based non-clicks. For our operational grouping, EBNR belongs to MDH because its downstream user representation is produced by a single recurrently updated history state.

\paragraph{Long-horizon state family (LDH).}
LDH groups encoders whose task-facing state can retain or directly aggregate information across the observed prefix through attention, sequence kernels, or long-range state-space dynamics. In the notation of Section~\ref{subsec:history_states_to_flow}, the generic form is
\begin{equation}
\label{eq:LDH_core_app}
\mathbf{z}_u^{(t_i)}
=
f_{\mathrm{LDH}}
\!\left(
\{\mathbf{x}_j\}_{j=1}^{i-1};
K_{\Theta}
\right),
\end{equation}
where $K_{\Theta}$ denotes the learned long-horizon sequence operator. A common kernelized realization is
\begin{equation}
\label{eq:LDH_kernel_app}
\mathbf{z}_u^{(t_i)}
=
\sum_{j=1}^{i-1}
K_{\Theta}(t_i,t_j,\mathbf{x}_{1:i-1})
\,\psi_{\Theta}(\mathbf{x}_j).
\end{equation}
The operator may be content-dependent, as in attention, or induced implicitly by state-space dynamics. Unrolling a recurrent state-space model yields an explicit lag kernel over the full prefix. We group these models here by that long-horizon sequence property. The grouping is used to organize comparisons rather than to define disjoint architectural classes.

\textbf{D-FM Attention (LDH-aligned control).}
We use a decay-modulated causal-attention control inspired by forgetting-based attention~\citep{foxformer-2025}. The control below uses a lag penalty; FoX itself uses data-dependent forgetting, so this equation describes our adapted control. Given
$\mathbf{q}_i=\mathbf{W}_Q\mathbf{x}_i$,
$\mathbf{k}_j=\mathbf{W}_K\mathbf{x}_j$, and
$\mathbf{v}_j=\mathbf{W}_V\mathbf{x}_j$, the causal logits are
\begin{equation}
\label{eq:dfm_logits_app}
\ell_{ij}
=
\frac{\mathbf{q}_i^\top\mathbf{k}_j}{\sqrt{d}}
-
\gamma(t_i-t_j),
\qquad j<i,
\end{equation}
with
\begin{equation}
\label{eq:dfm_softmax_app}
\alpha_{ij}
=
\frac{\exp(\ell_{ij})}
{\sum_{q=1}^{i-1}\exp(\ell_{iq})},
\qquad
\mathbf{z}_u^{(t_i)}
=
\sum_{j=1}^{i-1}\alpha_{ij}\mathbf{v}_j.
\end{equation}
This is LDH-aligned because the task-facing state aggregates the observed prefix through a learned attention kernel with explicit temporal decay.

\textbf{Discretized SSM-B (single-timescale LDH control).}
Our ordinary linear state-space control supplies one recurrent state. SS4Rec~\citep{ss4rec_ldh-arxic-2025} uses a richer time- and relation-aware state-space design; the following equations specify the control evaluated here. Let
$\mathbf{m}_i\in\mathbb{R}^{d_m}$ denote the latent state and
$\mathbf{x}_i\in\mathbb{R}^{d}$ the current input:
\begin{equation}
\label{eq:ssm_update_app}
\mathbf{m}_i
=
\mathbf{A}\mathbf{m}_{i-1}
+
\mathbf{B}\mathbf{x}_i,
\qquad
\mathbf{z}_u^{(t_{i+1})}
=
\mathbf{C}\mathbf{m}_i.
\end{equation}
Unrolling the recurrence gives
\begin{equation}
\label{eq:ssm_unroll_app}
\mathbf{m}_i
=
\mathbf{A}^{i}\mathbf{m}_0
+
\sum_{j=1}^{i}
\mathbf{A}^{\,i-j}\mathbf{B}\mathbf{x}_j.
\end{equation}
Thus, although the implementation is recurrent, the exposed state contains a full-prefix lag expansion through powers of $\mathbf{A}$. We place this baseline in LDH because the comparison of interest is its induced long-horizon kernel, while it still exposes only one temporal state and one retention mechanism.

\textbf{Tri-SSM-B (three-branch LDH control).}
Tri-SSM-B extends the ordinary SSM recurrence to three independently parameterized parallel branches while retaining the same schema-normalized input and corresponding downstream fusion/readout organization used for the \pertide{} comparison. For $r\in\{1,2,3\}$,
\begin{equation}
\label{eq:trissm_app}
\mathbf{m}_i^{(r)}
=
\mathbf{A}_r\mathbf{m}_{i-1}^{(r)}
+
\mathbf{B}_r\mathbf{x}_i.
\end{equation}
Unlike \pertide{}, the branches use the same generic SSM transition form and are not assigned explicit long-, short-, or episodic spectral maps. Tri-SSM-B therefore controls for schema input and parallel three-branch state organization while leaving trace-specific retention design absent.

\textbf{Transformer-style LDH baselines (NRMS / FastFormer-PLM / MINER).}
These models construct contextualized history representations through attention-style interactions over the observed prefix. Schematically,
\begin{equation}
\label{eq:mhsa_app}
\mathrm{Attn}(\mathbf{Q},\mathbf{K},\mathbf{V})
=
\mathrm{softmax}\!\left(
\frac{\mathbf{Q}\mathbf{K}^{\top}}{\sqrt{d}}
\right)\mathbf{V},
\qquad
\mathrm{MHSA}(\mathbf{X})
=
\mathrm{Concat}(\mathrm{head}_1,\ldots,\mathrm{head}_H)\mathbf{W}_O,
\end{equation}
where
$\mathbf{X}=[\mathbf{x}_1;\ldots;\mathbf{x}_{i-1}]$ and
\[
\mathrm{head}_h
=
\mathrm{Attn}
\!\left(
\mathbf{X}\mathbf{W}_Q^{(h)},
\mathbf{X}\mathbf{W}_K^{(h)},
\mathbf{X}\mathbf{W}_V^{(h)}
\right).
\]
A single task-facing user state can then be obtained through additive attention pooling:
\begin{equation}
\label{eq:add_pool_app}
\alpha_j
=
\frac{
\exp(\mathbf{v}^{\top}\tanh(\mathbf{W}\mathbf{x}'_j))
}{
\sum_{q=1}^{i-1}
\exp(\mathbf{v}^{\top}\tanh(\mathbf{W}\mathbf{x}'_q))
},
\qquad
\mathbf{z}_u^{(t_i)}
=
\sum_{j=1}^{i-1}\alpha_j\mathbf{x}'_j,
\end{equation}
where $\mathbf{x}'_j$ denotes the contextualized sequence representation. NRMS uses self-attentive history encoding followed by pooling, FastFormer-PLM uses additive/global-context attention to construct contextualized representations, and MINER extracts multiple interest representations from long histories before candidate matching. We group these models under LDH because their task-facing representations are constructed from learned interactions over extended history rather than an order-invariant profile or a single recency-bounded state.

\subsection{Baseline Personalized Models}
\label{app:baseline-personalized-models}

\paragraph{PENS-NRMS Injection-Type 1.}
The PENS framework \citep{pens-acl-2021} generates personalized headlines by conditioning on a user representation together with the input news article. In this variant, the user representation is produced by NRMS \citep{nrms}, whose news encoder represents article content and whose user encoder applies self-attention over interacted news. Injection-Type 1 initializes the decoder hidden state with the resulting user representation, allowing personalization to influence generation from the beginning of decoding.

\paragraph{PENS-NRMS Injection-Type 2.}
This variant uses the same NRMS user representation but introduces it through the decoder attention mechanism rather than the initial hidden state. The user representation therefore modulates attention over the source article during generation.

\paragraph{PENS-NAML Injection-Type 1.}
NAML \citep{naml} constructs news representations from multiple content views, including title, body, and category information, and aggregates interacted news into a user representation. In PENS Injection-Type 1, this representation initializes the decoder hidden state.

\paragraph{PENS-EBNR Injection-Type 1.}
EBNR \citep{ebnr} produces a user representation with a recurrent encoder over browsing history. Under PENS Injection-Type 1, this user representation initializes the decoder and thereby conditions generation from its initial state.

\paragraph{PENS-EBNR Injection-Type 2.}
This configuration uses the same recurrent EBNR user representation but introduces it through the decoder attention mechanism, allowing the user state to modulate source-side attention during generation.

\paragraph{General Then Personal (GTP).}
General Then Personal (GTP) \citep{gtp-tacl-2023} follows a two-stage personalized headline-generation pipeline. A general encoder--decoder first learns content-focused headline generation, after which a headline customizer incorporates a user-specific control representation. GTP further uses Information Self-Boosting (ISB) to recover relevant article information during customization and Masked User Modeling (MUM) to regularize the use of user representations.

\paragraph{Signature Phrase.}
Signature Phrase personalization \citep{signature-phrase-acl-2023} represents a user's reading history through a compact set of learned phrases. These phrases act as a dynamic profile of higher-level interests and are used to condition personalized headline generation while retaining grounding in the source article.

\paragraph{MovieLens-1M baselines.}
For MovieLens-1M, we organize predictive recommendation baselines using the same operational state-family grouping introduced in Section~\ref{subsec:history_states_to_flow} and Appendix~\ref{app:encoder_baselines}. Each user history is an ordered prefix
$\tau^{u}_{1:i-1}=(v^{u}_{1},\ldots,v^{u}_{i-1})$, where $v^{u}_{j}$ denotes the movie interacted with by user $u$ at timestep $j$. Each model produces a candidate score $s(u,c)$ for movie candidate $c$, and all candidates are ranked under the same one-positive plus sampled-negative evaluation protocol. The PAH/MDH/LDH labels below describe the exposed history-state form used for comparison; they do not imply mutually exclusive assumptions about the underlying user process.

\paragraph{BPR-MF (PAH-aligned).}
BPR-MF is used as a static profile-based matrix-factorization baseline. It assigns each user $u$ a persistent embedding $\mathbf{p}_u\in\mathbb{R}^{d}$ and each candidate $c$ an item embedding $\mathbf{q}_c\in\mathbb{R}^{d}$:
\[
s(u,c)=\mathbf{p}_u^\top\mathbf{q}_c.
\]
The pairwise ranking objective is
\[
\mathcal{L}_{\mathrm{BPR}}
=
-\sum_{(u,i^+,i^-)}
\log\sigma\!\left(
s(u,i^+)-s(u,i^-)
\right).
\]
Its exposed user representation is persistent and order-insensitive rather than computed from the ordered prefix, so we place BPR-MF in the PAH family.

\paragraph{GRU4Rec (MDH instantiation).}
GRU4Rec is used as a recurrent sequential recommendation baseline. Given item embeddings $\{\mathbf{q}_{v_j}\}_{j<i}$,
\[
\mathbf{h}_j
=
\mathrm{GRU}(\mathbf{h}_{j-1},\mathbf{q}_{v_j}),
\qquad
\mathbf{z}^{(t_i)}_u
=
\mathbf{h}_{i-1},
\]
and candidates are scored by
\[
s(u,c)
=
\mathbf{z}^{(t_i)\top}_u\mathbf{q}_c.
\]
GRU4Rec is an MDH instantiation because its task-facing state is a single recurrently updated representation obtained from the previous hidden state and the current interaction.

\paragraph{Caser (MDH-aligned local-window baseline).}
Caser encodes a bounded recent interaction window as
\[
\mathbf{X}_i
=
[\mathbf{q}_{v_{i-L}};\ldots;\mathbf{q}_{v_{i-1}}],
\]
then applies horizontal and vertical convolutional filters:
\[
\mathbf{z}^{(t_i)}_u
=
\mathrm{MLP}\!\left(
[
\mathrm{Conv}_{h}(\mathbf{X}_i);
\mathrm{Conv}_{v}(\mathbf{X}_i);
\mathbf{p}_u
]
\right),
\qquad
s(u,c)
=
\mathbf{z}^{(t_i)\top}_u\mathbf{q}_c.
\]
Caser is not a recurrent realization of Eq.~\eqref{eq:MDH_core_app}; we place it in the MDH family operationally because its task-facing sequential representation is dominated by a bounded recent-history window rather than an order-invariant profile or an explicit long-horizon sequence operator.

\paragraph{S$^3$Rec (LDH-aligned).}
S$^3$Rec uses a self-supervised Transformer-style sequence encoder:
\[
\mathbf{H}_i
=
\mathrm{SeqEnc}_{\theta}
\left(
[\mathbf{q}_{v_1}+\mathbf{p}_1;\ldots;
\mathbf{q}_{v_{i-1}}+\mathbf{p}_{i-1}]
\right),
\qquad
\mathbf{z}^{(t_i)}_u
=
\mathrm{Readout}(\mathbf{H}_i),
\]
with
\[
s(u,c)
=
\mathbf{z}^{(t_i)\top}_u\mathbf{q}_c.
\]
We place S$^3$Rec in LDH because the exposed recommendation state is constructed from a self-attentive representation of the observed prefix.

\paragraph{Diff4Rec (LDH-aligned).}
Diff4Rec uses curriculum-scheduled diffusion augmentation for sequential recommendation. Its recommendation stage scores candidates from a sequence representation:
\[
\mathbf{z}^{(t_i)}_u
=
\mathrm{SeqEnc}_{\theta}(\tau^u_{1:i-1}),
\qquad
s(u,c)
=
f_{\theta}(\mathbf{z}^{(t_i)}_u,\mathbf{q}_c).
\]
For our operational grouping, Diff4Rec is LDH-aligned because candidate prediction is conditioned on a representation of the extended interaction prefix rather than a static user profile or bounded local state.

\paragraph{SASRec (LDH-aligned).}
SASRec~\citep{sasrec-icdm-2018} applies causal self-attention over the ordered interaction prefix:
\[
\mathbf{H}_i
=
\mathrm{CausalSelfAttn}_{\theta}
\left(
[\mathbf{q}_{v_1}+\mathbf{p}_1;\ldots;
\mathbf{q}_{v_{i-1}}+\mathbf{p}_{i-1}]
\right),
\qquad
\mathbf{z}^{(t_i)}_u
=
\mathrm{Readout}(\mathbf{H}_i),
\]
\[
s(u,c)
=
\mathbf{z}^{(t_i)\top}_u\mathbf{q}_c.
\]
We place SASRec in LDH because its task-facing state is constructed through a learned attention operator over the observed prefix.

\paragraph{BERT4Rec (LDH-aligned).}
BERT4Rec learns contextualized item representations through masked-item prediction with a bidirectional Transformer:
\[
\mathbf{H}_i
=
\mathrm{BiTransformer}_{\theta}
\left(
[\mathbf{q}_{v_1}+\mathbf{p}_1;\ldots;
\mathbf{q}_{v_{i-1}}+\mathbf{p}_{i-1}]
\right),
\qquad
\mathbf{z}^{(t_i)}_u
=
\mathrm{Readout}(\mathbf{H}_i).
\]
Candidate movies are scored as
\[
s(u,c)
=
\mathbf{z}^{(t_i)\top}_u\mathbf{q}_c.
\]
We place BERT4Rec in LDH because the downstream recommendation state is obtained from a contextualized representation of the observed interaction sequence.

\paragraph{Mamba4Rec (LDH-aligned state-space baseline).}
Mamba4Rec uses selective state-space sequence modeling:
\[
\mathbf{m}_j
=
\mathbf{A}_j\mathbf{m}_{j-1}
+
\mathbf{B}_j\mathbf{q}_{v_j},
\qquad
\mathbf{z}^{(t_i)}_u
=
\mathbf{C}_{i-1}\mathbf{m}_{i-1},
\qquad
s(u,c)
=
\mathbf{z}^{(t_i)\top}_u\mathbf{q}_c.
\]
The recurrent implementation uses learned state-space dynamics to retain long-range dependence over the processed prefix. We therefore group Mamba4Rec under LDH according to its long-horizon sequence operator, while noting that the PAH/MDH/LDH families are operational rather than disjoint architectural classes.

\paragraph{SSD4Rec (LDH-aligned state-space baseline).}
SSD4Rec~\citep{ssd4rec-tois-2026} uses bidirectional structured state-space duality blocks for variable-length recommendation sequences. Its sequence representation incorporates forward and backward processing of the observed history; it should not be identified with a single forward linear recurrence. We group it under LDH because of its sequence-level history modeling. The evaluated comparison concerns this architecture versus \pertide{}'s action-gated, trace-specific updates, rather than a claim that all state-space recommenders expose the same computation.

\section{Measuring Personalization}\label{app:perseval}

\subsection{Motivation}\label{app:perseval-motivation}

EGISES \citep{egises} measures how insensitive a personalized summarizer is to user subjectivity, where lower values indicate less insensitivity and therefore better personalization. Following this idea, PerSEval \citep{perseval-tmlr-2024} evaluates whether a model changes its generated summary when the target user's expected summary changes, while also penalizing poor reference agreement. We use the notation $d_i$ for a document, $u_{ij}$ for the expected summary of user $j$ for document $d_i$, and $\hat{s}_{ij}$ for the model-generated summary for the same user-document pair. Let $U_i$ be the set of users who provide references for document $d_i$, and let $\sigma(\cdot,\cdot)$ be a distance function between two summaries.

The summary-level degree of responsiveness, denoted as DEGRESS, compares two quantities: the divergence between user references and the divergence between corresponding model outputs. For user $j$ and document $d_i$, it is defined as
\begin{equation}
\begin{split}
D_{ij}
=
\frac{1}{|U_i|}
\sum_{k=1}^{|U_i|}
\frac{\min(X_{ijk},Y_{ijk})+\epsilon}
{\max(X_{ijk},Y_{ijk})+\epsilon},
\end{split}
\end{equation}
where
\begin{equation}
\begin{split}
X_{ijk}=w^{u}_{ijk}\sigma(u_{ij},u_{ik}),
\qquad
Y_{ijk}=w^{s}_{ijk}\sigma(\hat{s}_{ij},\hat{s}_{ik}).
\end{split}
\end{equation}
Here, $X_{ijk}$ measures the distance between the expected summaries of users $j$ and $k$, and $Y_{ijk}$ measures the corresponding distance between the model-generated summaries. The weights $w^{u}_{ijk}$ and $w^{s}_{ijk}$ are softmax-normalized document-relative weights. They control how much each user-pair comparison contributes to the final responsiveness score.

The system-level DEGRESS score is
\begin{equation}
D(M)
=
\frac{1}{|\mathcal{D}|}
\sum_{i=1}^{|\mathcal{D}|}
\frac{1}{|U_i|}
\sum_{j=1}^{|U_i|}
D_{ij}.
\end{equation}
A low value indicates that the model does not respond sufficiently to user-specific reference variation. A high value indicates stronger responsiveness to different user preferences.

\subsection{PerSEval Formulation}\label{app:perseval-formulation}

DEGRESS only measures responsiveness. It does not penalize a model that produces personalized but inaccurate summaries. PerSEval addresses this limitation by multiplying DEGRESS with an accuracy-sensitive penalty. For user $j$ and document $d_i$, the summary-level PerSEval score is
\begin{equation}
P_{ij}
=
D_{ij}\cdot E_{ij},
\end{equation}
where $E_{ij}$ is the effective DEGRESS penalty factor:
\begin{equation}
E_{ij}
=
1-
\frac{1}
{1+10^{\alpha}\exp(-10^{\beta}Q_{ij})},
\qquad
\alpha\geq 3,\ \beta\geq 1.
\end{equation}
The penalty input $Q_{ij}$ combines document-level accuracy loss and user-level accuracy inconsistency:
\begin{equation}
Q_{ij}
=
A_i+C_{ij}.
\end{equation}

The document-level accuracy penalty is
\begin{equation}
A_i
=
\frac{1}
{1+10^{\gamma}\exp\left(
-10\cdot
\frac{a_i^{*}}
{1-a_i^{*}+\epsilon}
\right)},
\qquad
\gamma\geq 4,
\end{equation}
where
\begin{equation}
a_i^{*}
=
\min_{j\in U_i} a_{ij}.
\end{equation}
Here, $a_{ij}$ is the distance-based accuracy score for $\hat{s}_{ij}$ against the corresponding user reference $u_{ij}$, and $a_i^{*}$ is the best document-level accuracy term under the chosen distance convention.

The accuracy-inconsistency penalty is
\begin{equation}
C_{ij}
=
\frac{1}
{1+10^{\gamma}\exp\left(
-10\cdot
\frac{a_{ij}-a_i^{*}}
{\bar{a}_i-a_i^{*}+\epsilon}
\right)},
\qquad
\gamma\geq 4,
\end{equation}
where
\begin{equation}
\bar{a}_i
=
\frac{1}{|U_i|}
\sum_{j=1}^{|U_i|}
a_{ij}.
\end{equation}
Thus, $A_i$ penalizes poor best-case accuracy, while $C_{ij}$ penalizes inconsistency across user-specific generations for the same document.

The system-level PerSEval score is
\begin{equation}
P(M)
=
\frac{1}{|\mathcal{D}|}
\sum_{i=1}^{|\mathcal{D}|}
\frac{1}{|U_i|}
\sum_{j=1}^{|U_i|}
P_{ij}.
\end{equation}
The score lies in $[0,1]$, with higher values indicating better personalization under the chosen reference-distance convention, and remains bounded by the system-level DEGRESS score.

\subsection{PSE Metrics}\label{app:pse-metrics}

We use three PerSEval variants by changing the distance or similarity function used inside the formulation.

\textbf{PSE-SU4.} This variant uses ROUGE-SU4 as the base comparison function. ROUGE-SU4 measures overlap through unigrams and skip-bigrams, where a skip-bigram is an ordered word pair with up to four skipped words between the pair. We convert its F1 similarity to the required distance, $\sigma(G,R)=1-\mathrm{ROUGE\mbox{-}SU4}_{F1}(G,R)$, as in PerSEval~\citep{perseval-tmlr-2024}.

\textbf{PSE-JSD.} This variant uses Jensen-Shannon divergence between the normalized unigram or n-gram distributions of the generated and reference summaries. For distributions $P$ and $Q$,
\begin{equation}
\mathrm{JSD}(P\|Q)
=
\frac{1}{2}\mathrm{KL}(P\|M)
+
\frac{1}{2}\mathrm{KL}(Q\|M),
\qquad
M=\frac{1}{2}(P+Q).
\end{equation}
We set $\sigma(G,R)=\mathrm{JSD}(P_G\|P_R)$ to obtain PSE-JSD.

\textbf{PSE-METEOR.} This variant uses METEOR as the base similarity score. METEOR aligns generated and reference unigrams using exact, stem, synonym, and paraphrase matches, then combines precision, recall, and a fragmentation penalty. Since PerSEval requires a distance, we use
\begin{equation}
\sigma(G,R)=1-\mathrm{METEOR}(G,R).
\end{equation}
This gives PSE-METEOR.
\section{Prediction Accuracy Metrics}\label{app:prediction accuracy metrics}

(i) \textbf{nDCG@5 and HR@10.} We report normalized Discounted Cumulative Gain (nDCG) at cutoffs $K\in\{5,10\}$ to evaluate ranking quality while emphasizing higher-ranked items. For each query instance with a ranked list of candidates, let $\mathrm{rel}_i$ denote the relevance of the item at rank $i$ (binary relevance in our setup). We compute
\begin{equation}
\mathrm{DCG@K}=\sum_{i=1}^{K}\frac{2^{\mathrm{rel}_i}-1}{\log_2(i+1)},\qquad
\mathrm{nDCG@K}=\frac{\mathrm{DCG@K}}{\mathrm{IDCG@K}},
\end{equation}
where $\mathrm{IDCG@K}$ is the maximum possible $\mathrm{DCG@K}$ achievable by an ideal ordering for that instance. nDCG rewards higher placement of relevant items. HR@10 is the fraction of instances whose positive target occurs in the top ten and does not distinguish positions within that set. Under the one-positive protocol, HR@K equals Recall@K. All ranking scores in the result tables are multiplied by 100.

(ii) \textbf{MRR.} We report Mean Reciprocal Rank (MRR) to measure how early the first relevant item appears in the ranked list. For each instance, let $\mathrm{rank}^\star$ be the rank position of the first relevant item; the reciprocal rank is $1/\mathrm{rank}^\star$. The aggregate metric is
\begin{equation}
\mathrm{MRR}=\frac{1}{N}\sum_{n=1}^{N}\frac{1}{\mathrm{rank}^\star_n},
\end{equation}
where $N$ is the number of evaluation instances. MRR heavily rewards placing a relevant item at the very top and is therefore sensitive to early precision.

(iii) \textbf{Recall@K.} We report Recall at cutoff $K$ (denoted Recall@K) to quantify whether at least one relevant item is retrieved among the top-$K$ predictions. With binary relevance and (typically) one target item per instance, Recall@K reduces to an indicator of whether the target appears in the top-$K$ list:
\begin{equation}
\mathrm{Recall@K}=\frac{1}{N}\sum_{n=1}^{N}\mathbb{I}\big[\mathrm{rank}^\star_n\le K\big],
\end{equation}
where $\mathbb{I}[\cdot]$ is the indicator function. Larger Recall@K indicates higher top-$K$ hit rate, complementing MRR and nDCG by focusing on retrieval coverage rather than graded ordering.
\section{Reference-Based Content Metrics}\label{app:content accuracy metrics}

(i) \textbf{ROUGE-2.} ROUGE-2 measures bigram overlap between a generated text $y$ and a reference text $r$ using recall-oriented matching of contiguous 2-grams. Let $G_2(\cdot)$ denote the set of distinct bigram types, and let $\mathrm{count}_x(g)$ be the count of bigram $g$ in text $x$. The ROUGE-2 recall is
\begin{equation}
\mathrm{ROUGE\mbox{-}2}=\frac{\sum_{g\in G_2(r)} \min\big(\mathrm{count}_y(g),\,\mathrm{count}_r(g)\big)}
{\sum_{g\in G_2(r)} \mathrm{count}_r(g)}.
\end{equation}
Higher values indicate that the generation covers more of the reference bigram content.

(ii) \textbf{ROUGE-L.} ROUGE-L evaluates sequence-level overlap via the Longest Common Subsequence (LCS), capturing in-order matching without requiring contiguity. Let $\mathrm{LCS}(y,r)$ be the length of the LCS between $y$ and $r$, and let $|y|,|r|$ be token lengths. Define LCS-based precision and recall as
\begin{equation}
P_{\mathrm{LCS}}=\frac{\mathrm{LCS}(y,r)}{|y|},\qquad
R_{\mathrm{LCS}}=\frac{\mathrm{LCS}(y,r)}{|r|}.
\end{equation}
ROUGE-L is commonly reported as an F-measure:
\begin{equation}
\mathrm{ROUGE\mbox{-}L}=\frac{(1+\beta^2) P_{\mathrm{LCS}} R_{\mathrm{LCS}}}{R_{\mathrm{LCS}}+\beta^2 P_{\mathrm{LCS}}},
\end{equation}
where $\beta$ controls the relative weight of recall versus precision (often set to favor recall in summarization).

(iii) \textbf{ROUGE-SU4.} ROUGE-SU4 measures overlap of skip-bigrams with a maximum skip distance of 4, together with unigrams (the ``U'' component). A skip-bigram is any ordered pair of tokens $(w_i,w_j)$ such that $1 \le j-i \le 5$ (i.e., up to four skipped tokens between them). Let $G_{\mathrm{SU4}}(x)$ be the set of distinct unigram and SU4 skip-bigram types from text $x$. The ROUGE-SU4 recall is
\begin{equation}
\mathrm{ROUGE\mbox{-}SU4}=\frac{\sum_{g\in G_{\mathrm{SU4}}(r)} \min\big(\mathrm{count}_y(g),\,\mathrm{count}_r(g)\big)}
{\sum_{g\in G_{\mathrm{SU4}}(r)} \mathrm{count}_r(g)}.
\end{equation}
This metric is more permissive than ROUGE-2 because it credits in-order matches even when tokens are separated by a small gap.

(iv) \textbf{METEOR.} METEOR aligns generated and reference tokens using exact matches and (optionally) linguistic matches (e.g., stem/synonym), then computes a harmonic mean of unigram precision and recall with a fragmentation penalty for disordered alignments. Let $m$ be the number of matched unigrams under the alignment, and let $|y|,|r|$ be token lengths. Define
\begin{equation}
P=\frac{m}{|y|},\qquad R=\frac{m}{|r|},\qquad
F_{\alpha}=\frac{PR}{\alpha P+(1-\alpha)R},
\end{equation}
where $\alpha$ weights precision versus recall (commonly set to emphasize recall). Let $c$ be the number of matched ``chunks'' (contiguous sequences of matches in the aligned order). The penalty is
\begin{equation}
\mathrm{Pen}=\gamma\left(\frac{c}{m}\right)^{\theta},
\end{equation}
and the final METEOR score is
\begin{equation}
\mathrm{METEOR}=(1-\mathrm{Pen})\,F_{\alpha},
\end{equation}
with $\gamma,\theta$ controlling the strength of the fragmentation penalty.

(v) \textbf{BLEU-4.} BLEU-4 evaluates $n$-gram precision up to $n=4$ with a brevity penalty to discourage overly short outputs. Let $p_n$ denote the modified precision for $n$-grams, computed using clipped counts against the reference. BLEU-4 is
\begin{equation}
\mathrm{BLEU\mbox{-}4}=\mathrm{BP}\cdot \exp\left(\sum_{n=1}^{4} w_n \log p_n\right),
\end{equation}
where typically $w_n=\frac{1}{4}$. The brevity penalty is
\begin{equation}
\mathrm{BP}=
\begin{cases}
1, & |y|>|r|,\\
\exp\left(1-\frac{|r|}{|y|}\right), & |y|\le |r|.
\end{cases}
\end{equation}
Higher BLEU-4 indicates greater $n$-gram precision while accounting for length adequacy via $\mathrm{BP}$.

(vi) \textbf{BERTScore.} BERTScore measures semantic similarity between a generated text $y=\{y_i\}_{i=1}^{m}$ and a reference text $r=\{r_j\}_{j=1}^{n}$ by aligning contextual token embeddings from a pretrained encoder. Let $\mathbf{h}(y_i)$ and $\mathbf{h}(r_j)$ denote the contextualized embeddings, and let $s_{ij}=\cos(\mathbf{h}(y_i),\mathbf{h}(r_j))$ be their cosine similarity. BERTScore precision and recall are computed via soft token matching:
\begin{equation}
P_{\mathrm{BS}}=\frac{1}{m}\sum_{i=1}^{m}\max_{1\le j\le n} s_{ij},\qquad
R_{\mathrm{BS}}=\frac{1}{n}\sum_{j=1}^{n}\max_{1\le i\le m} s_{ij},
\end{equation}
and combined into an F-score
\begin{equation}
F_{\mathrm{BS}}=\frac{2P_{\mathrm{BS}}R_{\mathrm{BS}}}{P_{\mathrm{BS}}+R_{\mathrm{BS}}}.
\end{equation}
Higher BERTScore indicates stronger semantic alignment between the generation and the reference, even when surface $n$-gram overlap is low.
\section{Detailed Results And Additional Analysis}
\label{app:detailed_results}

\FloatBarrier
\subsection{Efficiency And Deployment Cost}
\label{app:detailed_results_efficiency}
\begin{table*}[htbp]
\caption{\textbf{Runtime and deployment costs for \pertide{}.}
We report approximate resource figures for
\pertide{} and comparison systems. The action-on-item encoder supports online updates
from cached user state without replaying the full history, whereas prompted LLM
personalization conditions on history through the inference context. Numbers correspond
to $d{=}768$, $|\mathcal{Q}|{=}4$, and, for generation, a pretrained DistilGPT2
decoder conditioned by a learned $K{=}8$ prefix. The final two decoder blocks and
final layer normalization are fine-tuned during generation training; the remaining
decoder parameters stay fixed. Latency, throughput, VRAM, FLOPs, and training time are source estimates, not a matched hardware benchmark; the per-user cache is computed directly from the state dimensions.}
\label{tab:runtime_deployment_costs_pertide}
\centering
\footnotesize
\setlength{\tabcolsep}{3pt}
\scalebox{1}{
\begin{tabular}{p{0.28\linewidth}p{0.24\linewidth}p{0.24\linewidth}p{0.16\linewidth}}
\toprule
\textbf{Metric} &
\textbf{\pertide{} (Ours)} &
\textbf{Single-scale encoder} &
\textbf{Prompted LLM (ICL)} \\
\midrule

\multicolumn{4}{l}{\textbf{A. Parameter and Memory Footprint}} \\
\midrule

\textbf{Trainable parameters (encoder only)} &
$\sim$154.36M &
$\sim$130--145M &
0 \\

\textbf{Decoder contextualization parameters} &
$\sim$8.27M (UFI) / $\sim$10.63M (UGI) &
task-dependent &
0 \\

\textbf{Encoder + contextualization parameters} &
$\sim$162.63M (UFI) / $\sim$164.99M (UGI) &
$\sim$130--145M + context &
0 \\

\textbf{Decoder parameters fine-tuned for generation} &
$\sim$14.18M (last 2 blocks + final LN) &
same protocol if DistilGPT2 is used &
0 (ICL) \\

\textbf{Pretrained decoder backbone deployed} &
$\sim$82M (DistilGPT2) &
$\sim$82M (if used) &
7B--32B \\

\textbf{Deployable parameters (generation)} &
$\sim$245--247M &
$\sim$212--227M + context &
7B--32B \\

\textbf{Per-user cached state} &
4 $\times$ 768 floats ($\approx$12 KiB) &
history replay or single state &
0 \\

\textbf{Peak VRAM (predictive inference)} &
$\sim$1.1--1.3 GB &
$\sim$1.0--1.2 GB &
$\ge$18 GB \\

\textbf{Peak VRAM (generative inference)} &
$\sim$1.6--1.9 GB &
$\sim$1.5--1.8 GB &
$\ge$28 GB \\

\midrule
\multicolumn{4}{l}{\textbf{B. History Dependence and Online Updates}} \\
\midrule

\textbf{History dependence at inference} &
cached state only (no replay) &
windowed / cached state / replay &
full prompt history \\

\textbf{Online update cost per interaction} &
$\mathcal{O}(d^2)$ (three trace updates + fusion) &
$\mathcal{O}(d^2)$ &
-- \\

\textbf{State refresh frequency} &
every timestep &
every timestep &
-- \\

\midrule
\bottomrule
\end{tabular}}
\end{table*}
\begin{table*}[p]
\centering\footnotesize
\setlength{\tabcolsep}{3pt}
\caption{\textbf{Approximate serving and training resources.} Continuation of the resource estimates in Appendix~\ref{app:detailed_results_efficiency}, Table~\ref{tab:runtime_deployment_costs_pertide}. These estimates are not a matched hardware benchmark.}
\label{tab:runtime_estimates_continued}
\scalebox{1}{
\begin{tabular}{p{0.28\linewidth}p{0.24\linewidth}p{0.24\linewidth}p{0.16\linewidth}}
\toprule
\textbf{Metric}&\textbf{\pertide{}}&\textbf{Single-scale encoder}&\textbf{Prompted LLM}\\
\midrule
\multicolumn{4}{l}{\textbf{C. Predictive Serving Cost}} \\
\midrule

\textbf{Latency per sample (end-to-end)} &
$\sim$2--6 ms &
$\sim$2--5 ms &
$\sim$20--60 ms \\

\textbf{Throughput (samples/sec)} &
$\sim$160--420 &
$\sim$180--450 &
$\sim$20--60 \\

\textbf{FLOPs per sample} &
$\sim$3.8$\times10^{9}$ &
$\sim$2.5--3.2$\times10^{9}$ &
$\ge$1$\times10^{11}$ \\

\midrule
\multicolumn{4}{l}{\textbf{D. Generative Serving Cost}} \\
\midrule

\textbf{Contextualization} &
UFI / UGI learned prefix &
concat / none &
ICL prompt \\

\textbf{Personalization context injected} &
8 learned prefix tokens &
0--1 learned tokens &
130--250 history tokens \\

\textbf{Generation length} &
task-dependent &
task-dependent &
task-dependent \\

\textbf{End-to-end latency per sample} &
0.4--3 s &
0.4--3 s &
15--160 s \\

\textbf{FLOPs per output (est.)} &
$\sim$6.1$\times10^{10}$ &
$\sim$6.0$\times10^{10}$ &
$\sim$1.6$\times10^{12}$ \\

\midrule
\multicolumn{4}{l}{\textbf{E. Training Cost}} \\
\midrule

\textbf{Training GPU-hours} &
$\sim$22 &
$\sim$18--30 &
-- \\

\textbf{Decoder backbone fine-tuning} &
Partial (last 2 blocks + final LN) &
task-dependent &
None (ICL) \\

\bottomrule
\end{tabular}
}
\end{table*}
\pertide{} supports incremental history updates with a fixed cache: three 768-dimensional traces and the previous 768-dimensional task-ready state require 12 KiB per user in FP32. The cache does not grow with history length. Each new event updates it directly, and subsequent predictive or generative requests reuse the retained state. This is the deployment benefit of the recurrence even though the encoder has more parameters than the single-scale comparison.

Generation uses DistilGPT2 with an eight-token contextual prefix and partial decoder tuning, whereas the evaluated 7B--32B LLMs receive histories through prompting. The analytical or approximate resource estimates in Table~\ref{tab:runtime_deployment_costs_pertide} describe expected costs; they are not measurements from a matched hardware benchmark. The fixed cache is determined by the architecture; realized throughput and latency additionally depend on implementation, hardware, and request lengths.

\FloatBarrier
\subsection{Full Predictive Personalization Results}
\label{app:detailed_results_predictive}
\FloatBarrier
\subsubsection{MovieLens-1M Recommendation}
\label{app:detailed_results_movielens}
\begin{table}[htbp]
\centering
\caption{\textbf{MovieLens movie recommendation under action-on-item preference flow.}
We compare \pertide{} against task-native recommendation baselines on next-item ranking.
Baselines expose profile, recurrent, or long-horizon history states, while \pertide{} uses
schema-normalized action evidence with multi-timescale persistence. Point estimates are means
under the paired-bootstrap evaluation in Appendix~\ref{app:training-details}; $\pm$ denotes
standard deviation over three uniformly sampled 25\% evaluation subsets for \pertide{} variants.}
\label{tab:ml1m_recommendation}
\footnotesize
\setlength{\tabcolsep}{5pt}
\renewcommand{\arraystretch}{1.03}
\scalebox{1}{
\begin{tabular}{llccc}
\toprule
\textbf{Regime} & \textbf{Model} & \textbf{MRR} & \textbf{nDCG@10} & \textbf{HR@10} \\
\midrule
\multirow{1}{*}{PAH}
& \texttt{BPR-MF}      & 6.50  & 6.12  & 12.81 \\
\midrule
\multirow{3}{*}{MDH}
& \texttt{GRU4Rec}     & 12.76 & 16.42 & 29.01 \\
& \texttt{Caser}       & 13.54 & 19.21 & 28.92 \\
& \texttt{Diff4Rec}    & 12.02 & 15.72 & 22.38 \\
\midrule
\multirow{5}{*}{LDH}
& \texttt{S$^3$Rec}    & 10.62 & 13.67 & 19.21 \\
& \texttt{SASRec}      & 13.20 & 19.96 & 30.23 \\
& \texttt{BERT4Rec}    & 14.41 & 20.22 & 31.71 \\
& \texttt{Mamba4Rec}   & 14.56 & 19.83 & 31.24 \\
& \texttt{SSD4Rec}     & 14.95 & 20.31 & 31.52 \\
\midrule
\multirow{3}{*}{MTSH}
& \pertide\ w/o AC     & 15.21 $\pm$ 0.93 & 19.53 $\pm$ 1.12 & 28.16 $\pm$ 1.26 \\
& \pertide\ w/o CR     & 17.05 $\pm$ 0.95 & 19.08 $\pm$ 1.43 & 27.85 $\pm$ 0.56 \\
& \pertide{}           & \textbf{17.31 $\pm$ 0.65} & \textbf{23.78 $\pm$ 0.84} & \textbf{33.20 $\pm$ 1.17} \\
\bottomrule
\end{tabular}
}
\vspace{-1mm}
\end{table}
\pertide{} reaches 17.31/23.78/33.20 MRR/nDCG@10/HR@10. Compared with SSD4Rec, the gains are 2.36/3.47/1.68 points (Table~\ref{tab:ml1m_recommendation}). BERT4Rec has the strongest baseline HR@10, 31.71, while SSD4Rec leads baseline MRR and nDCG. The full model exceeds both. Removing action conditioning gives 15.21 MRR and 28.16 HR@10; removing command readout gives 17.05/19.08/27.85 across the three metrics. These interventions connect the ranking advantage to how evidence enters and is consumed, beyond changing the history-model family.

\FloatBarrier
\subsubsection{MIND News Recommendation}
\label{app:detailed_results_mind}
\begin{table}[htbp]
\centering
\caption{\textbf{Action-on-item encoder performance on MIND.}
We compare task-native history encoders and schema-matched controls with \pertide{} on sequential ranking.
AC denotes action-conditioned evidence entry, and CR denotes command-conditioned readout.
Point estimates are means under the paired-bootstrap evaluation in Appendix~\ref{app:training-details};
$\pm$ denotes standard deviation over three uniformly sampled 25\% evaluation subsets for \pertide{} variants.}
\label{tab:history_encoder_results}
\scalebox{.83}{
\begin{tabular}{l l c c c}
\toprule
\textbf{Regime} & \textbf{Model} & \textbf{MRR} & \textbf{nDCG@5} & \textbf{HR@10} \\
\midrule

\multirow{4}{*}{\textbf{PAH}}
& NAML & 32.75 & 35.66 & 41.40 \\
& PLM-NR & 35.39 & 38.71 & 44.38 \\
& Mean Pooling-B (ours) & 29.43 & 31.22 & 33.17 \\
& Max Pooling-B (ours) & 28.32 & 31.16 & 32.41 \\
\midrule

\multirow{2}{*}{\textbf{MDH}}
& EBNR & 31.26 & 32.18 & 39.04 \\
& GRU-B (ours) & 32.64 & 35.17 & 37.92 \\
\midrule

\multirow{6}{*}{\textbf{LDH}}
& NRMS & 33.25 & 36.28 & 41.98 \\
& MINER & 36.60 & 40.20 & 46.30 \\
& Fastformer+PLM-NR & 35.91 & 39.16 & 45.03 \\
& D-FM Attention (ours) & 32.17 & 34.56 & 36.21 \\
& Discretized SSM (ours) & 35.14 & 38.53 & 37.27 \\
& Tri-SSM-B (ours) & 42.53 & 43.84 & 46.33 \\
\midrule

\multirow{3}{*}{\textbf{MTSH}}
& \pertide\ w/o AC (ours) & 40.13 $\pm$ 2.43 & 43.75 $\pm$ 1.37 & 46.35 $\pm$ 1.16 \\
& \pertide\ w/o CR (ours) & 42.51 $\pm$ 2.16 & 46.44 $\pm$ 1.75 & 51.06 $\pm$ 1.56 \\
& \pertide\ (ours) & \textbf{46.65 $\pm$ 2.24} & \textbf{49.21 $\pm$ 1.64} & \textbf{55.48 $\pm$ 1.14} \\
\bottomrule
\end{tabular}
}
\end{table}
The task-native comparison spans profile, recurrent, attention, and state-space encoders (Table~\ref{tab:history_encoder_results}). \pertide{} exceeds MINER by 10.05/9.01/9.18 MRR/nDCG@5/HR@10 points. The more controlled same-input comparison in main-text Table~\ref{tab:iclr_schema_matched} gives gains of 4.12/5.37/9.15 over Tri-SSM-B. Holding the trajectory representation and three-branch organization fixed therefore leaves a measurable advantage for the proposed retention design. The action and readout ablations additionally reduce MRR from 46.65 to 40.13 and 42.51, respectively.

\FloatBarrier
\subsubsection{PENS Next-News Prediction}
\label{app:detailed_results_pens_structured}
\begin{table}[htbp]
\centering
\caption{\textbf{Structured news prediction on PENS-Structured.} We report same-user next-news ranking under the action-on-item schema. Baselines are grouped by the exposed history-state assumption in Section~\ref{sec:action_item_preference_trajectories}. All values are on a 0--100 scale.}
\label{tab:rq1_pens_structured_recommendation}
\setlength{\tabcolsep}{4pt}
\renewcommand{\arraystretch}{1.02}
\scalebox{0.9}{
\begin{tabular}{llcccc}
\toprule
\textbf{Regime} & \textbf{Model} & \textbf{MRR} & \textbf{nDCG@5} & \textbf{nDCG@10} & \textbf{HR@10} \\
\midrule
\textbf{MDH} & \texttt{EBNR} & 2.65 & 1.82 & 2.45 & 4.87 \\
\textbf{PAH} & \texttt{NAML} & 1.29 & 0.43 & 0.81 & 2.01 \\
\textbf{LDH} & \texttt{NRMS} & 1.18 & 0.39 & 0.74 & 1.92 \\
\textbf{LDH} & \texttt{TrRMIo} & 8.02 & 7.85 & 8.53 & 12.77 \\
\textbf{LDH} & \texttt{LSTUR} & 9.04 & 8.21 & 10.16 & 12.93 \\
\textbf{LDH} & \texttt{MINER} & 5.17 & 4.71 & 5.61 & 9.77 \\
\textbf{LLM} & \texttt{DeepSeek-32B} & 1.31 & 0.28 & 0.64 & 1.53 \\
\textbf{LLM} & \texttt{Qwen2.5-32B} & 1.24 & 0.26 & 0.59 & 1.47 \\
\midrule
\textbf{MTSH} & \texttt{\pertide} & \textbf{9.94} & \textbf{9.88} & \textbf{11.58} & \textbf{18.32} \\
\bottomrule
\end{tabular}
}
\end{table}
\pertide{} reaches 9.94 MRR, 9.88 nDCG@5, 11.58 nDCG@10, and 18.32 HR@10 (Table~\ref{tab:rq1_pens_structured_recommendation}). Relative to LSTUR, gains are 0.90 MRR, 1.42 nDCG@10, and 5.39 HR@10 points. With one positive target per candidate set, HR@10 measures how often that target appears anywhere in the top ten; it does not assess its order within that set. These histories also support headline generation below, providing a common domain in which to test both consumers.

\FloatBarrier
\subsection{Retention Probes And Trace Complementarity}
\label{app:detailed_results_trace_subset}
\begin{table}[htbp]
\centering
\caption{\textbf{Temporal regime-specific trace specialization.} Controlled L-Tr, S-Tr, and E-Tr probes emphasize persistent, recent, and event-local evidence. Scores are MRR/nDCG@10 on the 0--100 scale. Bold denotes the overall best result; underlining marks the strongest isolated \pertide{} trace in each regime. The full model leads every reported probe; the long trace narrowly leads the short trace on MovieLens S-Tr.}
\label{tab:controlled_temporal_split_pertide_variants}
\small
\setlength{\tabcolsep}{4pt}
\renewcommand{\arraystretch}{1.05}
\begin{tabular}{lllccc}
\toprule
\textbf{Dataset} &
\textbf{Family} &
\textbf{Model / Variant} &
\textbf{L-Tr} &
\textbf{S-Tr} &
\textbf{E-Tr} \\
\midrule

\multirow{7}{*}{MovieLens}
& \multirow{3}{*}{Baseline}
& Mamba4Rec
& 26.43/29.38
& 26.12/28.87
& 25.72/28.31 \\
&
& GRU4Rec
& 24.17/27.32
& 23.65/26.58
& 22.13/24.71 \\
&
& SASRec
& 25.82/28.16
& 25.85/28.45
& 25.17/27.91 \\
\cmidrule(lr){2-6}

& \multirow{4}{*}{\pertide{}}
& PerTIDE-L
& \underline{27.31/29.96}
& \underline{27.18/29.35}
& 25.83/28.77 \\
&
& PerTIDE-S
& 25.21/28.64
& 27.13/29.22
& 22.16/24.84 \\
&
& PerTIDE-E
& 22.19/24.08
& 22.75/25.04
& \underline{27.14/30.36} \\
&
& \textbf{PerTIDE-L+S+E}
& \textbf{29.33/32.17}
& \textbf{29.04/31.85}
& \textbf{29.12/32.09} \\

\midrule

\multirow{7}{*}{MIND}
& \multirow{3}{*}{Baseline}
& NAML
& 46.21/51.23
& 46.25/51.35
& 45.32/49.18 \\
&
& EBNR
& 47.44/52.25
& 48.32/53.17
& 44.19/48.79 \\
&
& LSTUR
& 49.31/54.41
& 48.93/53.86
& 45.64/49.84 \\
\cmidrule(lr){2-6}

& \multirow{4}{*}{\pertide{}}
& PerTIDE-L
& \underline{49.54/54.78}
& 48.65/53.22
& 47.32/51.66 \\
&
& PerTIDE-S
& 49.21/53.94
& \underline{50.07/55.35}
& 48.65/52.81 \\
&
& PerTIDE-E
& 47.31/51.26
& 47.13/51.06
& \underline{50.35/55.84} \\
&
& \textbf{PerTIDE-L+S+E}
& \textbf{52.75/59.32}
& \textbf{51.76/57.92}
& \textbf{51.34/57.56} \\
\bottomrule
\end{tabular}
\end{table}
Target-fixed probes place preference-similar evidence in persistent (L-Tr), recent (S-Tr), or short-burst (E-Tr) configurations. Table~\ref{tab:controlled_temporal_split_pertide_variants} gives the complete scores. On MIND, the intended isolated trace leads each regime. On MovieLens S-Tr, the long trace narrowly exceeds the short trace, 27.18/29.35 versus 27.13/29.22 MRR/nDCG@10. The full model leads all six dataset--regime combinations, supporting the use of the traces together even where isolated specialization overlaps.
\begin{table*}[htbp]
\centering

\caption{\textbf{Trace-subset ablation for schema-normalized evidence persistence.} We ablate long-term (L), short-term (S), and episodic (E) persistence channels individually and in combination. Recommendation results are reported on MIND and MovieLens-1M, and personalized summarization is reported on PENS. Best results are bolded.}
\label{tab:trace_subset_ablation}
\setlength{\tabcolsep}{3pt}
\scalebox{0.82}{
\begin{tabular}{lccc ccc ccc}
\toprule
\multirow{2}{*}{\textbf{Variant}}
& \multicolumn{3}{c}{\textbf{MIND Recommendation}}
& \multicolumn{3}{c}{\textbf{ML-1M Recommendation}}
& \multicolumn{3}{c}{\textbf{PENS Summarization}} \\
\cmidrule(lr){2-4}
\cmidrule(lr){5-7}
\cmidrule(lr){8-10}
& \textbf{MRR} & \textbf{nDCG@5} & \textbf{HR@10}
& \textbf{MRR} & \textbf{nDCG@10} & \textbf{HR@10}
& \textbf{PSE-JSD} & \textbf{PSE-SU4} & \textbf{PSE-METEOR} \\
\midrule
L only
& 40.41 & 43.75 & 46.13
& 12.31 & 16.15 & 20.34
& 0.231 & 0.095 & 0.107 \\

S only
& 37.65 & 41.32 & 44.24
& 12.47 & 16.62 & 20.81
& 0.238 & 0.106 & 0.111 \\

E only
& 24.18 & 27.73 & 31.23
& 10.18 & 13.17 & 18.54
& 0.076 & 0.042 & 0.045 \\

L + S
& 41.26 & 43.95 & 46.72
& 14.17 & 18.44 & 22.11
& 0.264 & 0.143 & 0.182 \\

L + E
& 40.45 & 43.82 & 46.44
& 14.03 & 18.07 & 22.74
& 0.246 & 0.103 & 0.108 \\

S + E
& 38.26 & 42.34 & 45.07
& 13.43 & 17.18 & 21.67
& 0.093 & 0.062 & 0.077 \\

\textbf{L + S + E}
& \textbf{46.65} & \textbf{49.21} & \textbf{55.48}
& \textbf{17.31} & \textbf{23.78} & \textbf{33.20}
& \textbf{0.366} & \textbf{0.341} & \textbf{0.355} \\
\bottomrule
\end{tabular}
}
\end{table*}
The full L+S+E model is strongest across all reported metrics on MIND, MovieLens, and PENS generation (Table~\ref{tab:trace_subset_ablation}). A trace can make a larger contribution in combination than its isolated performance suggests. E alone reaches 24.18 MRR on MIND, while adding it to L+S raises MRR from 41.26 to 46.65. On PENS, L+S reaches 0.264 PSE-JSD and 0.182 PSE-METEOR; the full model reaches 0.366 and 0.355. The weakest isolated branch can therefore supply information useful to the fused state.

This contribution depends on the combination. PENS S+E gives 0.093/0.062/0.077 in the three PSE metrics, below S alone at 0.238/0.106/0.111. The full model's advantage demonstrates measured complementarity under the evaluated training setup, rather than proving that every extra trace helps or that the learned coordinates recover distinct latent interests.

\FloatBarrier
\subsection{Generation And Reference Agreement}
\label{app:detailed_results_generation}
\FloatBarrier
\subsubsection{PENS Personalized Headlines}
\label{app:detailed_results_pens_generation}
\begin{table}[htbp]
\centering
\caption{\textbf{Command-conditioned personalized headline generation on PENS.}
We compare prompted LLMs, task-native personalized generators, schema-matched history controls, and
\pertide{}. AC denotes action-conditioned evidence entry, CR denotes command-conditioned readout,
and UFI/UGI denote User Fused/Gated Injection. Point estimates are means under the paired-bootstrap
evaluation in Appendix~\ref{app:training-details}; $\pm$ denotes standard deviation over three
uniformly sampled 25\% evaluation subsets for \pertide{} variants.}
\label{tab:pse_icml}
\scalebox{.82}{
\begin{tabular}{l l c c c}
\toprule
\textbf{Category} & \textbf{Model} & \textbf{PSE-JSD} & \textbf{PSE-SU4} & \textbf{PSE-METEOR} \\
\midrule

\multirow{5}{*}{\textbf{LLMs (2-shot)}} 
& LLaMA-13B        & 0.227 & 0.078 & 0.081 \\
& DeepSeek-14B     & 0.248 & 0.094 & 0.097 \\
& Gemini-2.5-Flash & 0.222 & 0.104 & 0.124 \\
& DeepSeek-R1-32B  & 0.263 & 0.125 & 0.174 \\
& Qwen-2.5-32B     & 0.162 & 0.103 & 0.115 \\
\midrule

\multirow{2}{*}{\textbf{Prompt-Chaining}} 
& Mistral-7B       & 0.072 & 0.026 & 0.023 \\
& DeepSeek-14B     & 0.078 & 0.028 & 0.024 \\
\midrule

\multirow{3}{*}{\textbf{PAH}} 
& Signature-Phrase & 0.017 & 0.015 & 0.016 \\
& PENS-NAML-T1     & 0.021 & 0.014 & 0.016 \\
& MaxPool-B        & 0.008 & 0.005 & 0.005 \\
\midrule

\multirow{3}{*}{\textbf{MDH}} 
& PENS-EBNR-T1     & 0.015 & 0.010 & 0.011 \\
& PENS-EBNR-T2     & 0.011 & 0.008 & 0.009 \\
& GRU-B            & 0.021 & 0.018 & 0.019 \\
\midrule

\multirow{5}{*}{\textbf{LDH}} 
& GTP              & 0.024 & 0.017 & 0.019 \\
& Discretized SSM  & 0.037 & 0.032 & 0.031 \\
& D-FM Attention-B & 0.028 & 0.019 & 0.025 \\
& PENS-NRMS-T1     & 0.015 & 0.011 & 0.011 \\
& PENS-NRMS-T2     & 0.008 & 0.007 & 0.007 \\
\midrule

\multirow{4}{*}{\textbf{MTSH}} 
& \pertide\ w/o AC      & 0.314 $\pm$ 0.023 & 0.271 $\pm$ 0.007 & 0.292 $\pm$ 0.006 \\
& \pertide\ w/o CR      & 0.317 $\pm$ 0.024 & 0.276 $\pm$ 0.005 & 0.284 $\pm$ 0.008 \\
& \pertide\ (UFI)       & 0.332 $\pm$ 0.019 & 0.289 $\pm$ 0.009 & 0.303 $\pm$ 0.003 \\
& \pertide\ (Full; UGI) & \textbf{0.366 $\pm$ 0.011} & \textbf{0.341 $\pm$ 0.002} & \textbf{0.355 $\pm$ 0.005} \\
\bottomrule
\end{tabular}
}
\end{table}
UGI reaches 0.366/0.341/0.355 PSE-JSD/PSE-SU4/PSE-METEOR, exceeding UFI's 0.332/0.289/0.303 and the listed two-shot LLMs (Table~\ref{tab:pse_icml}). Its gains over UFI are 0.034/0.052/0.052, supporting user-conditioned modulation before decoder injection. The compact decoder can thus use a learned history state effectively across the reported personalization metrics. Training differs from the two-shot LLM protocol, so these scores compare the complete evaluated systems.

\FloatBarrier
\subsubsection{Reference Overlap And Human Similarity Ratings}
\label{app:detailed_results_content_fidelity}
\begin{table}[htbp]
\centering

\caption{\textbf{Reference agreement for command-driven headline generation.} Automatic metrics measure reference overlap or semantic similarity; HJ is the six-point reference-similarity rating described in Appendix~\ref{app:hj}.}
\label{tab:accuracy_results}
\scalebox{.85}{
\begin{tabular}{llcccc}
\hline
\textbf{Category} & \textbf{Model} & \textbf{RG-SU4} & \textbf{BLEU-4} & \textbf{BScore} & \textbf{HJ} \\
\hline
\multirow{3}{*}{Specialized (Personalized)} 
  & PENS-NRMS-T2   & 13.64 & 4.48 & 86.13 & 2.95 \\
  & GTP-TrRMIo     & 21.91 & 10.31 & 88.53 & 2.44 \\
  & SP-Individual  & 19.54 & 8.90 & 86.61 & 2.86 \\
\hline
\multirow{4}{*}{LLMs (2-shot history)} 
  & LLaMA-13B      & 18.31 & 11.85 & 88.76 & 3.05 \\
  & DeepSeek-14B   & 19.57 & 12.68 & 89.43 & 3.03 \\
  & DeepSeek-32B   & \textit{29.42} & \textit{19.31} & \textbf{90.87} & \textit{3.12} \\
  & Qwen-2.5-32B   & 19.22 & 14.31 & 89.17 & 3.04 \\
\hline
\multirow{1}{*}{\textbf{MTSH}}
  & \pertide-Full  & \textbf{41.32} & \textbf{27.64} & \textit{90.84} & \textbf{3.21} \\
\hline
\end{tabular}
}
\end{table}
The improvement extends beyond PerSEval to other forms of reference agreement (Table~\ref{tab:accuracy_results}). \pertide{} reaches 41.32 RG-SU4 and 27.64 BLEU-4, compared with DeepSeek-32B's 29.42 and 19.31. BERTScore is slightly lower, 90.84 versus 90.87. Human reference-similarity ratings average 3.21 versus 3.12 on the six-point scale. These measures provide complementary evidence of agreement with personalized reference headlines; they do not directly establish factual consistency with the article or a user's experienced utility. Appendix~\ref{app:hj} gives the actual judgment task.

\FloatBarrier
\subsection{Transfer And Sparse-Evidence Stress Tests}
\label{app:detailed_results_transfer_sparse}
\FloatBarrier
\subsubsection{Frozen-Core Reuse And Schema Distance}
\label{app:reuse_controls}
\begin{table}[htbp]
\centering
\caption{\textbf{Schema-distance diagnostic for conditional reuse} on 500 MovieLens trajectories from
99 users. Larger distance is associated with higher risk and smaller reciprocal-rank gain.}
\label{tab:iclr_schema_distance}
\setlength{\tabcolsep}{4pt}
\small
\scalebox{1}{
\begin{tabular}{lcc}
\toprule
\textbf{Relation} & \textbf{Spearman $\rho$} & \textbf{User-clustered 95\% CI} \\
\midrule
$d_{\mathrm{sch}}$ vs.\ frozen risk & +0.109 & [+0.002,+0.207] \\
$d_{\mathrm{sch}}$ vs.\ adapted risk & +0.160 & [+0.060,+0.258] \\
$d_{\mathrm{sch}}$ vs.\ frozen RR gain & -0.089 & [-0.169,-0.004] \\
$d_{\mathrm{sch}}$ vs.\ adapted RR gain & -0.114 & [-0.195,-0.029] \\
\bottomrule
\end{tabular}}
\end{table}
The frozen-core protocol holds target rows, candidates, head architecture, and training fixed. The PENS-trained core with a trained MovieLens head achieves 16.344 MRR, versus 1.115 for the random core with a trained head and 1.371 for the trained core with an untrained head (Table~\ref{tab:iclr_frozen_transfer}). The 15.229-point difference, rounded to 15.23, attributes useful structure to source learning while revealing the consumer's role in accessing it. Table~\ref{tab:iclr_transfer} separately compares UniSRec and RecGURU at 500 and 3,000 target trajectories. At 3,000, \pertide{} improves on the strongest baseline by 1.43/2.98/5.26 MRR/nDCG@10/HR@10 points. Target alignment improves all three systems; the two protocols use different slices.

Across 99 users and 500 trajectories, schema-distance correlations with target risk are 0.109 and 0.160, and correlations with reciprocal-rank gain are $-0.089$ and $-0.114$ (Table~\ref{tab:iclr_schema_distance}). Their reported confidence intervals support small associations in the expected directions. The geometric quantities and conditional-loss assumptions in the theoretical bound are not empirically certified by these associations.

\FloatBarrier
\keepcompacttable
\subsubsection{Role Corruption}
\label{app:role_history_controls}
\begin{wraptable}{r}{0.59\textwidth}
\centering
\caption{\textbf{Action-role corruption on fixed PENS trajectories.}
Only positive/negative roles are flipped; trajectories, targets, and candidate sets are held fixed.}
\label{tab:iclr_role_corruption}
\setlength{\tabcolsep}{3pt}
\scalebox{1}{
\begin{tabular}{rccc}
\toprule
\textbf{Role flips} & \textbf{MRR} & \textbf{nDCG@10} & \textbf{HR@10} \\
\midrule
0\% & 12.14 & 16.02 & 32.0 \\
10\% & 10.13 & 13.02 & 25.0 \\
25\% & 9.17 & 11.10 & 22.0 \\
50\% & 7.71 & 8.93 & 15.0 \\
100\% & 2.81 & 4.90 & 10.0 \\
\bottomrule
\end{tabular}}
\end{wraptable}
Correct action labels help the encoder use otherwise identical histories. With trajectories, targets, and candidates fixed, positive/negative role corruption reduces MRR from 12.14 to 2.81 (Table~\ref{tab:iclr_role_corruption}). A 10\% corruption rate already reduces MRR by 2.01 points, and 50\% corruption lowers HR@10 from 32.0 to 15.0. Both the ordering of the target and its inclusion in the top ten deteriorate as supplied roles become less reliable. The intervention therefore supports retaining the action component of each event; it measures predictive dependence on those labels.

\FloatBarrier
\keepcompacttable
\subsubsection{History Length}
\label{app:history_length_control}
\begin{wraptable}{r}{0.59\textwidth}
\centering
\caption{\textbf{History-length sensitivity on the same 500 PENS trajectories.}
The target and 1+699 candidate sets are fixed while the visible history grows from 10 to 40 events.}
\label{tab:iclr_history_length}
\setlength{\tabcolsep}{3pt}
\scalebox{1}{
\begin{tabular}{rccc}
\toprule
\textbf{Events} & \textbf{MRR} & \textbf{nDCG@10} & \textbf{HR@10} \\
\midrule
10 & 11.874 & 14.904 & 27.8 \\
20 & 11.881 & 15.011 & 28.2 \\
30 & 11.909 & 15.091 & 28.4 \\
40 & \textbf{11.926} & \textbf{15.157} & \textbf{28.6} \\
\bottomrule
\end{tabular}}
\end{wraptable}
Additional history offers a small benefit in the separate 500-trajectory PENS study. Extending the visible prefix from 10 to 40 events changes MRR from 11.874 to 11.926 and HR@10 from 27.8 to 28.6 (Table~\ref{tab:iclr_history_length}). The target and candidate set remain fixed while the encoder receives more past events. The modest changes across all three metrics quantify the return from extra history under this protocol. They complement the role-corruption study: changing which action is attached to an event and increasing the number of observed events test different aspects of the history representation.

\FloatBarrier
\subsubsection{Eventized OpenAI--Reddit Summaries}
\label{app:detailed_results_openai_reddit}
\begin{table}[htbp]
\centering

\caption{\textbf{Eventized TL;DR personalization on OpenAI-Reddit.} We treat OpenAI-Reddit as an eventization stress test under surrogate trajectory construction, not as evidence of unconstrained cross-domain user-history transfer.}
\label{tab:rq2_openai_only}
\scalebox{0.95}{
\begin{tabular}{llccc}
\toprule
\textbf{Category} & \textbf{Model} & \textbf{PSE-JSD} & \textbf{PSE-SU4} & \textbf{PSE-METEOR} \\
\midrule

\multirow{4}{*}{LLMs (2-shot)}
& LLaMA-13B     & 0.232 & 0.093 & 0.107 \\
& Zephyr-7B     & 0.214 & 0.087 & 0.104 \\
& Mistral-7B    & 0.226 & 0.088 & 0.103 \\
& DeepSeek-14B  & \textit{\underline{0.243}} & \textit{\underline{0.095}} & \textit{\underline{0.109}} \\
\midrule

\multirow{1}{*}{\textbf{MTSH}}
& \pertide-Full & \textbf{0.285} & \textbf{0.261} & \textbf{0.294} \\
\bottomrule
\end{tabular}
}

\end{table}
The action-on-item interface also accommodates feedback-derived summary trajectories. \pertide{} reaches 0.285/0.261/0.294 PSE-JSD/PSE-SU4/PSE-METEOR, versus DeepSeek-14B's 0.243/0.095/0.109 (Table~\ref{tab:rq2_openai_only}). This extends the evaluated generation setting beyond PENS. The chronology and positive/negative roles are constructed from feedback, and augmentation occurs separately within the native training and evaluation pools; the result concerns that eventized protocol.

\FloatBarrier
\subsubsection{Sparse Positive-Action Histories}
\label{app:detailed_results_sparse}
\begin{table}[htbp]
\centering

\caption{\textbf{Robustness under sparse positive-action evidence.} We evaluate each model on a click-only sparse subset and report the original score (left) and sparse-action score (right), with the relative drop in parentheses. Smaller drops indicate higher robustness when positive evidence is limited in the action-on-item trajectory.}
\label{tab:sparse-trajectory-results}
\setlength{\tabcolsep}{3pt}
\scalebox{0.85}{
\begin{tabular}{lccc}
\toprule
\textbf{Model} & \textbf{PSE-JSD} & \textbf{PSE-SU4} & \textbf{PSE-METEOR} \\
\midrule
DeepSeek-R1-32B (2-shot)
& $0.263 \rightarrow 0.152$ \, (\textminus42.21\%)
& $0.125 \rightarrow 0.087$ \, (\textminus30.4\%)
& $0.174 \rightarrow 0.092$ \, (\textminus47.2\%) \\
Gemini-2.5-Flash (2-shot)
& $0.222 \rightarrow 0.122$ \, (\textminus45.0\%)
& $0.104 \rightarrow 0.061$ \, (\textminus41.3\%)
& $0.124 \rightarrow 0.070$ \, (\textminus43.5\%) \\
Best Baseline (GTP)
& $0.024 \rightarrow 0.016$ \, (\textminus33.3\%)
& $0.017 \rightarrow 0.009$ \, (\textminus47.1\%)
& $0.019 \rightarrow 0.011$ \, (\textminus42.1\%) \\
\midrule
\pertide
& $0.366 \rightarrow 0.211$ \, (\textminus42.35\%)
& $0.341 \rightarrow 0.244$ \, (\textminus28.45\%)
& $0.355 \rightarrow 0.262$ \, (\textminus26.2\%) \\
\bottomrule
\end{tabular}
}

\end{table}
\pertide{} retains 0.211 PSE-JSD under sparse positive evidence, compared with 0.366 originally, a 42.35\% drop (Table~\ref{tab:sparse-trajectory-results}). Its relative retention is stronger on PSE-SU4 and PSE-METEOR, with reported drops of 28.45\% and 26.2\%. For DeepSeek-R1-32B, the JSD change from 0.263 to 0.152 corresponds to 42.21\%, and the reported SU4/METEOR drops are 30.4\%/47.2\%. GTP changes from 0.017 to 0.009 on PSE-SU4. The sparse scores establish robustness under the particular intervention, while causal attribution to a specific trace would require the corresponding controlled ablation.

\FloatBarrier
\subsection{Dialogue As An Additional Application And Boundary}
\label{app:mpchat boundary}
MPChat tests the encoder on dialogue histories with inferred stance roles. The generation result demonstrates a useful extension: \pertide{} exceeds Discretized SSM by 8.48/7.08/9.09/3.52 points in RG-2/RG-L/BLEU-4/BERTScore (Table~\ref{tab:dialogue_accuracy_results}). The event interface and decoder-conditioning route can therefore operate with dialogue-derived events as well as explicit news interactions.
\begin{table}[htbp]
\centering

\caption{\textbf{MPChat broader-applicability stress test: dialogue generation.}
AC denotes action-conditioned evidence injection, CR denotes command-conditioned readout, UFI denotes user fused injection, and Full denotes \pertide{} encoder with UGI decoder. These results test dialogue eventization under imperfect preference logs rather than serving as primary evidence for schema-level generalizability; $\mu\pm\sigma$ denotes mean and standard deviation over three evaluations on uniformly random test-set subsamples.}
\label{tab:dialogue_accuracy_results}
\scalebox{.8}{
\begin{tabular}{llcccc}
\hline
\textbf{Category} & \textbf{Model} & \textbf{RG-2} & \textbf{RG-L} & \textbf{BLEU-4} & \textbf{BScore}\\
\hline

\multirow{3}{*}{LLMs (2-shot)} 
  & DeepSeek-32B        & 9.71 & 8.46 & 5.61 & 71.42 \\
  & Qwen-2.5-32B        & 9.43 & 8.21 & 4.75 & 68.83 \\
  & MS/Phi-Instruct      & 9.49 & 8.05 & 3.21 & 69.52 \\
\hline

\multirow{2}{*}{PAH}
  & Meanpool-B           & $8.54\pm1.33$ & $6.31\pm1.21$ & $3.43\pm2.15$ & $62.18\pm1.45$ \\
  & Maxpool-B            & $7.21\pm1.12$ & $4.12\pm1.43$ & $2.55\pm2.11$ & $61.42\pm1.65$ \\
\hline

\multirow{1}{*}{MDH}
  & GRU-B                & $11.11\pm1.31$ & $10.82\pm1.43$ & $4.17\pm0.78$ & $63.18\pm2.43$ \\
\hline

\multirow{2}{*}{LDH}
  & D-FM Attention        & $8.25\pm0.89$ & $6.07\pm1.32$ & $3.18\pm1.44$ & $62.05\pm2.78$ \\
  & Discretized SSM       & $12.17\pm1.32$ & $11.44\pm1.56$ & $6.48\pm1.73$ & $74.83\pm1.66$ \\
\hline

\multirow{4}{*}{\textbf{MTSH}}
  & \pertide\ w/o AC      & $14.38\pm1.32$ & $12.73\pm1.42$ & $8.33\pm0.93$ & $73.89\pm1.65$ \\
  & \pertide\ w/o CR      & $13.15\pm1.77$ & $10.91\pm1.14$ & $7.34\pm1.55$ & $71.64\pm2.42$ \\
  & \pertide\ (UFI)       & $18.22\pm1.55$ & $15.73\pm1.21$ & $11.58\pm1.36$ & $75.23\pm2.51$ \\
  & \pertide-Full         & $\textbf{20.65}\pm\textbf{1.23}$ & $\textbf{18.52}\pm\textbf{2.31}$ & $\textbf{15.57}\pm\textbf{2.06}$ & $\textbf{78.35}\pm1.55$ \\
\hline
\end{tabular}
}
\end{table}

\textbf{Response prediction.}
\pertide{} reaches 45.64/54.25 R@1/MRR, while the text-only SBERT comparator reaches 51.32/64.76 (Table~\ref{tab:response_prediction_results}). A multimodal CLIP configuration with full persona reaches 72.65/81.12. Because a text-only baseline already exceeds \pertide{}, missing images or persona cannot fully explain its shortfall. Comparisons across different inputs measure performance differences, not a causal upper bound on the value of those inputs.
\begin{table}[htbp]
\centering

\caption{\textbf{MPChat broader-applicability stress test: response prediction.}
Dialogue eventization is compared with text-only and multimodal consumers. Scores use the 0--100 scale; $\mu\pm\sigma$ denotes the reported mean and standard deviation over three uniformly sampled 25\% test subsets, shared across methods.}
\label{tab:response_prediction_results}
\scalebox{0.9}{
\begin{tabular}{lcc}
\hline
\textbf{Model} & \textbf{R@1} & \textbf{MRR} \\
\hline

SBERT (zero-shot)                    & 35.67 & 45.75 \\
SBERT                                & $51.32\pm1.32$ & $64.76\pm0.92$ \\

SBERT+ViT ($c$)                        & $57.70\pm0.71$ & $69.39\pm0.40$ \\
SBERT+ViT ($c,p^i$)                    & $58.55\pm0.70$ & $70.17\pm0.45$ \\
SBERT+ViT ($c,p^t$)                    & $64.32\pm0.64$ & $74.31\pm0.45$ \\
SBERT+ViT ($c,p^{i+t}$, Full)          & ${65.29}\pm{0.66}$ & ${75.08}\pm{0.43}$ \\
\hline

SBERT+CLIP ($c$)                       & $59.68\pm0.70$ & $70.99\pm0.49$ \\
SBERT+CLIP ($c,p^i$)                   & $60.30\pm0.50$ & $71.47\pm0.27$ \\
SBERT+CLIP ($c,p^t$)                   & $64.32\pm0.75$ & $74.33\pm0.57$ \\
SBERT+CLIP ($c,p^{i+t}$, Full)         & ${65.43}\pm{0.42}$ & ${75.19}\pm{0.32}$ \\
\hline

CLIP+CLIP ($c^i$)                      & $40.85\pm0.64$ & $54.32\pm0.30$ \\
CLIP+CLIP ($c$)                        & $69.11\pm0.74$ & $78.22\pm0.49$ \\
CLIP+CLIP ($c,p^i$)                    & $69.87\pm0.40$ & $78.85\pm0.27$ \\
CLIP+CLIP ($c,p^t$)                    & $72.13\pm0.61$ & $80.72\pm0.38$ \\
CLIP+CLIP ($c,p^{i+t}$, Full)          & ${72.65}\pm{0.38}$ & ${81.12}\pm{0.26}$ \\
\hline

DeepSeek-R1-32B ($c,p^t$)                       & 41.45 & 54.42 \\
Qwen2.5-32B ($c,p^t$)                           & 41.32 & 52.17 \\
\hline

\pertide-Full ($c^t$) (ours)           & $45.64\pm0.45$ & $54.25\pm0.67$ \\
\hline
\end{tabular}
}

\end{table}

\textbf{Speaker identification.}
\pertide{} reaches 41.77/47.19 R@1/MRR (Table~\ref{tab:speaker_identification_results}). Text persona improves several comparator configurations, and CLIP with multimodal persona reaches 62.17/73.08. The results motivate richer representation of discourse and speaker information while documenting where the current encoder trails specialized alternatives.
\begin{table}[htbp]
\centering

\caption{\textbf{MPChat broader-applicability stress test: speaker identification.}
We report speaker identification under dialogue eventization and compare against systems using persona and multimodal context; $c$: dialogue context, $c^t$: text-only context, $r$: response, $\mathcal{P}_c$: candidate set with profile information}
\label{tab:speaker_identification_results}
\setlength{\tabcolsep}{3.5pt}
\scalebox{0.85}{
\begin{tabular}{llcc}
\hline
\textbf{Category} & \textbf{Model} & \textbf{R@1} & \textbf{MRR} \\
\hline

\multirow{1}{*}{Text Only}
& SBERT ($c^t,r,\mathcal{P}^t_c$)
& $56.47\pm0.58$ & $67.92\pm0.52$ \\
\hline

\multirow{3}{*}{SBERT+ViT}
& $c,r,\mathcal{P}^i_c$
& $19.56\pm0.64$ & $35.84\pm0.45$ \\
& $c,r,\mathcal{P}^t_c$
& $56.87\pm0.60$ & $68.33\pm0.37$ \\
& $c,r,\mathcal{P}^{i+t}_c$ (Full)
& ${57.28\pm0.44}$ & ${68.86\pm0.33}$ \\
\hline

\multirow{3}{*}{SBERT+CLIP}
& $c,r,\mathcal{P}^i_c$
& $25.71\pm0.49$ & $42.47\pm0.34$ \\
& $c,r,\mathcal{P}^t_c$
& $56.63\pm0.66$ & $68.15\pm0.42$ \\
& $c,r,\mathcal{P}^{i+t}_c$ (Full)
& ${57.24\pm0.63}$ & ${68.69\pm0.39}$ \\
\hline

\multirow{3}{*}{CLIP+CLIP}
& $c,r,\mathcal{P}^i_c$
& $44.27\pm0.66$ & $59.04\pm0.35$ \\
& $c,r,\mathcal{P}^t_c$
& $59.89\pm0.71$ & $70.87\pm0.53$ \\
& $c,r,\mathcal{P}^{i+t}_c$ (Full)
& ${62.17\pm0.56}$ & ${73.08\pm0.35}$ \\
\hline

\multirow{4}{*}{LLMs (2-shot)}
& DeepSeek-32B ($c,r,\mathcal{P}^{t}_c$) & 48.13 & 54.37 \\
& Qwen-2.5-32B ($c,r,\mathcal{P}^{t}_c$) & 45.31 & 51.62 \\
& DeepSeek-32B ($c,r$) & 28.56 & 31.43 \\
& Qwen-2.5-32B ($c,r$) & 25.61 & 29.67 \\
\hline
\multirow{1}{*}{\textbf{MTSH (ours)}}
& \pertide-Full ($c^t,r$)
& $41.77\pm0.82$ & $47.19\pm1.30$ \\
\hline

\end{tabular}
}
\end{table}

The dialogue results separate applicability from uniform superiority: inferred events support improved generation in the reported comparison, while response prediction and speaker identification remain more difficult. The transfer theory in Appendix~\ref{app:schema_generalizability} makes compatibility explicit; it does not measure the effect of missing persona or multimodal evidence in these experiments.

\section{Human Reference-Similarity Evaluation (PENS)}\label{app:hj}
To assess agreement with personalized reference headlines, we conduct a survey-based evaluation in which each participant\footnote{Demography: 65 male, 22 female graduate students from Computer Science, Humanities, Mathematics, \& Natural Sciences.} is shown (i) the gold-reference summary and (ii) one model-generated summary for the same PENS news article. Model identities are hidden to mitigate brand and expectation bias, and participants rate reference–generation similarity on a 1--6 scale (1: low, 6: high). We evaluate \pertide\ against representative strong baselines spanning specialized personalized models (PENS-NRMS-T2, GTP-TrRMIo, SP-Individual) and 2-shot history-conditioned LLMs (DeepSeek-14B, DeepSeek-32B, Qwen-2.5-32B).

We report average human-judgment (HJ) scores alongside automatic metrics (RG-SU4, BLEU-4, BScore) in Table~\ref{tab:accuracy_results}. \pertide\ attains the highest HJ score (3.21), exceeding the best 2-shot LLM (DeepSeek-32B: 3.12) and all specialized baselines (e.g., PENS-NRMS-T2: 2.95; GTP-TRRMIO: 2.44; SP-Individual: 2.86). 

Participation was voluntary, responses were de-identified before analysis, and the study involved only minimal-risk similarity judgments over news headlines. We collected only aggregate demographic information and did not collect identifiable participant information or private user histories. The study followed the applicable institutional approval, exemption, or equivalent review process under anonymized reporting.

\begin{figure*}[t]
\centering
\includegraphics[width=0.45\linewidth]{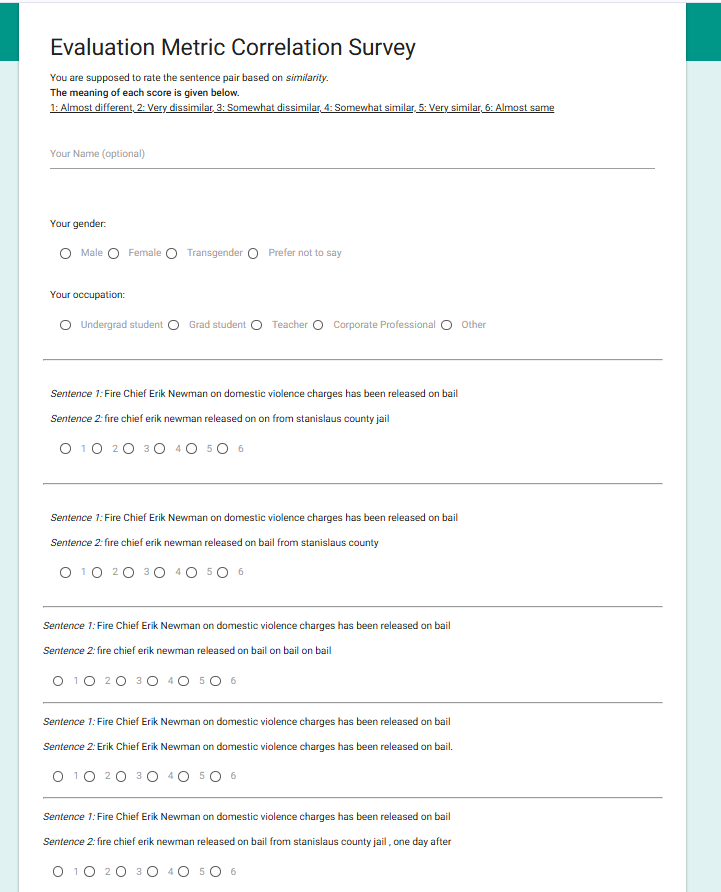}
\caption{HJ-annotation survey format}.
\label{fig:survey}
\end{figure*}

\begin{figure*}[t]
\centering
\includegraphics[width=0.75\linewidth]{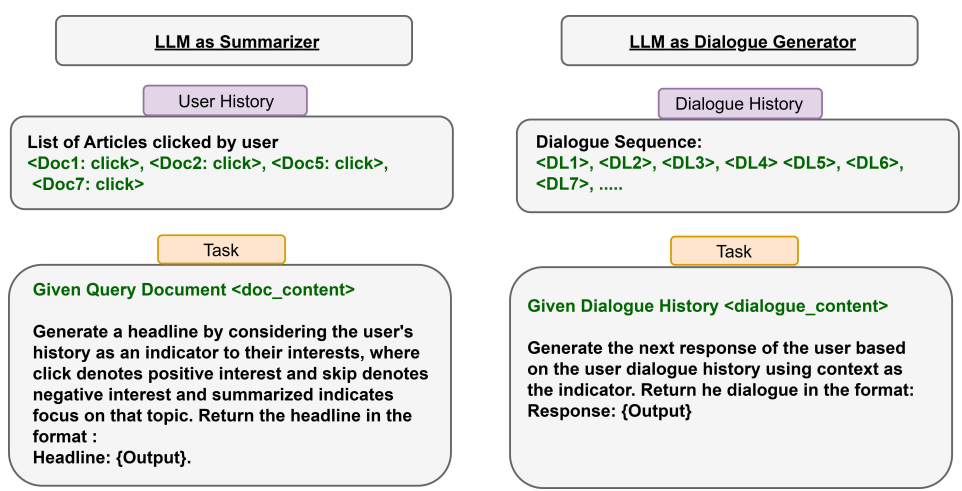}
\caption{Prompt Templates for LLMs for (a)Personalized Summarization, (b)Dialogue Generation}.
\label{fig:prompts}
\end{figure*}

\section{LLM Prompting Protocol}\label{app:app_llm_prompting}

\textbf{Overview.}
We use prompted LLMs for generative tasks using temperature $0.4$ and nucleus sampling with $\texttt{top\_p}=0.9$ for all LLM runs.

\subsection{LLMs As Summarizers (Personalized Headline Generation)}\label{app:app_llm_summarizers_writeup}
Each instance is a conditional generation problem, where we prompt the LLMs using user-history as the context and a single query document body and instruct the LLM to generate one personalized headline.
The instruction ties personalization to the history (clicked topics should be emphasized; skipped topics should be deprioritized) while requiring factual consistency with the query document.
We constrain the output to a single field (\texttt{Headline: ...}) and ignore any additional text.

\subsection{LLMs As Dialogue Generators}\label{app:app_llm_dialogue}
We prompt the LLMs using the entire dialogue discourse as the context and instruct the LLM to generate one next response.
The instruction ties personalization to factual consistency with the persona information and the discourse context.
% We constrain the output to a single field (\texttt{Response: ...}) and ignore any additional text.

% Each instance is a listwise ranking problem: the prompt provides the user-history prefix, the query document, and a candidate set of summaries with stable IDs.
% The LLM outputs an ordered top-$K$ list of candidate IDs ($K=10$) ranked by user-specific relevance.
% We enforce a single-line schema (\texttt{Ranked List: [...]}) to enable deterministic rank extraction.

\section{Assets And Usage Terms}\label{app:licence-agreement}
\paragraph{Roles of pretrained models.}
The prompted LLM baselines use fixed pretrained weights. \pertide{} instead partially tunes DistilGPT2 and its contextualization modules as specified in Appendix~\ref{app:decoder-variants}; its item embeddings are precomputed. The DistilGPT2 model card lists Apache 2.0 (\url{https://huggingface.co/distilbert/distilgpt2}). DeepSeek-R1 models use the provider's MIT terms (\url{https://github.com/deepseek-ai/DeepSeek-R1/blob/main/LICENSE}); Mistral-7B-Instruct-v0.2 lists Apache 2.0 (\url{https://huggingface.co/mistralai/Mistral-7B-Instruct-v0.2}); Llama-2 uses its Community License Agreement (\url{https://huggingface.co/meta-llama/Llama-2-13b-chat-hf}); and Zephyr-7B-beta lists MIT (\url{https://huggingface.co/HuggingFaceH4/zephyr-7b-beta}). Hosted-model use is additionally governed by the relevant provider's service terms. Model citations and baseline descriptions identify their roles; a provider license does not by itself certify compliance of every downstream use.

\paragraph{Datasets.}
MovieLens usage conditions are supplied by GroupLens (\url{https://grouplens.org/datasets/movielens/1m/}). MIND and PENS are distributed for research under the Microsoft Research License Terms linked from their official pages (\url{https://msnews.github.io/}; \url{https://msnews.github.io/pens.html}). OpenAI's summarization release includes its repository license (\url{https://github.com/openai/summarize-from-feedback/blob/master/LICENSE}); MPChat provides CC BY 4.0 terms in its release (\url{https://github.com/ahnjaewoo/MPCHAT/blob/main/LICENSE}). These assets retain their original attribution and access conditions. The paper describes derived event trajectories and does not grant redistribution rights to the underlying data.

\end{document}